\documentclass[10pt,journal,compsoc]{IEEEtran}

\usepackage{times}
\usepackage{epsfig}
\usepackage{graphicx}
\usepackage{amsmath}
\usepackage{amssymb}
\usepackage{amsthm}
\usepackage{bm}
\usepackage{subfig}
\usepackage{capt-of}
\usepackage{booktabs}
\usepackage[table]{xcolor}
\usepackage{mathtools}
\usepackage{siunitx}
\usepackage[figurename=Figure]{caption}
\usepackage[hyphens,spaces]{url}

\usepackage{adjustbox}
\usepackage{bbm}
\usepackage{xspace}
\usepackage{multirow}
\usepackage{hyperref}
\usepackage[flushleft]{threeparttable}
\usepackage{graphics}
\usepackage{subfig}
\usepackage{scrextend}
\usepackage{float}

\def\REVISION#1{\textcolor[RGB]{0,0,0}{#1}}

\newtheorem{theorem}{Theorem}[section]

\newtheorem{corollary}[theorem]{Corollary}
\newtheorem{proposition}[theorem]{Proposition}

\newtheorem{remark}[theorem]{Remark}

\newcommand{\etal}{\textit{et al}. }
\newcommand{\ie}{\textit{i}.\textit{e}., }
\newcommand{\eg}{\textit{e}.\textit{g}., }

\newcommand{\fig}{{Figure}\@\xspace}
\newcommand{\tab}{{Table}\@\xspace}
\newcommand{\eqn}{{Eq.}\@\xspace}

\newcolumntype{L}[1]{>{\raggedright\arraybackslash}p{#1}}
\newcolumntype{C}[1]{>{\centering\arraybackslash}p{#1}}
\newcolumntype{R}[1]{>{\raggedleft\arraybackslash}p{#1}}

\ifCLASSOPTIONcompsoc
  \usepackage[nocompress]{cite}
\else
  \usepackage{cite}
\fi

\begin{document}

\title{Direction Concentration Learning: \\Enhancing Congruency in Machine Learning}

\author{Yan~Luo,
        Yongkang~Wong,~\IEEEmembership{Member,~IEEE,}
        Mohan~Kankanhalli,~\IEEEmembership{Fellow,~IEEE,}
        and~Qi~Zhao,~\IEEEmembership{Member,~IEEE}
\IEEEcompsocitemizethanks
			{
			\IEEEcompsocthanksitem Y. Luo and Q. Zhao are with the Department of Computer Science and Engineering, University of Minnesota, Minneapolis, MN, 55455.\protect\\
			E-mail: luoxx648@umn.edu \& qzhao@cs.umn.edu
			\IEEEcompsocthanksitem Y. Wong and M. Kankanhalli are with the School of Computing, National University of Singapore, Singapore, 117417.
			E-mail: yongkang.wong@nus.edu.sg \& mohan@comp.nus.edu.sg
			}
}

\markboth{IEEE TRANSACTIONS ON PATTERN ANALYSIS AND MACHINE INTELLIGENCE,~VOL.~xx, NO.~x, August~20xx}%
{Shell \MakeLowercase{\textit{et al.}}: Bare Demo of IEEEtran.cls for Computer Society Journals}

\IEEEtitleabstractindextext{%
\begin{abstract}


One of the well-known challenges in computer vision tasks is the visual diversity of images, which could result in an agreement or disagreement between the learned knowledge and the visual content exhibited by the current observation. In this work, we first define such an agreement in a concepts learning process as congruency. Formally, given a particular task and sufficiently large dataset, the congruency issue occurs in the learning process whereby the task-specific semantics in the training data are highly varying. We propose a Direction Concentration Learning (DCL) method to improve congruency in the learning process, where enhancing congruency influences the convergence path to be less circuitous. The experimental results show that the proposed DCL method generalizes to state-of-the-art models and optimizers, as well as improves the performances of saliency prediction task, continual learning task, and classification task. Moreover, it helps mitigate the catastrophic forgetting problem in the continual learning task. The code is publicly available at \url{https://github.com/luoyan407/congruency}.

\end{abstract}

\begin{IEEEkeywords}
Optimization, Machine Learning, Computer Vision, Accumulated Gradient, Congruency.
\end{IEEEkeywords}}

\maketitle

\IEEEdisplaynontitleabstractindextext

\IEEEpeerreviewmaketitle

\IEEEraisesectionheading{\section{Introduction}}
\label{sec:introduction}

\IEEEPARstart{D}{eep}
learning has been receiving considerable attention due to its success in various computer vision tasks~\cite{Krizhevsky_NIPS_2012,He_CVPR_2016,Huang_CVPR_2017,Chang_PAMI_2017} and challenges~\cite{Deng_CVPR_2009,Lin_ECCV_2014}. To prevent model overfitting and enhance the generalization ability, a training process often sequentially updates the model with gradients w.r.t. a mini-batch of training samples, as opposed to using a larger batch~\cite{Goyal_arXiv_2017}. Due to the complexity and diversity in the nature of image data and task-specific semantics, the discrepancy between current and previous observed mini-batches could result in a circuitous convergence path, which possibly hinders the convergence to a local minimum.

To better understand the circuitousness/straightforwardness in a learning process, we introduce {\it congruency} to quantify the agreement between new information used for an update and the knowledge learned from previous iterations. The word ``congruency'' is borrowed from a psychology study~\cite{Underwood_QJEP_2006} that inspects the influence of an object which is inconsistent with the scene in the visual attention perception task. In this work, we define congruency $\nu$ as the cosine similarity between the gradient $g$ to be used for update and a referential gradient $\hat{g}$ that indicates a general descent direction resulting from previous updates, \ie
\begin{align}
\nu = \cos \alpha(g, \hat{g}), 
\label{eqn:cong_def}
\end{align}
(The detailed formulation is presented in Section~\ref{sec:problem}). \fig~\ref{fig:concept} presents an illustration of congruency in the saliency prediction task. Due to similar scene (\ie dining) and similar fixations on faces and foods, the update of sample {\small $S_{2}$} (\ie {\small $\Delta w_{S_2}$}) is congruent with {\small $\Delta w_{S_1}$}. In contrast, the scene and fixations in sample {\small $S_{3}$} are different from sample {\small $S_{1}$} and {\small $S_{2}$}. This leads to a large angle ($> \ang{90}$) between {\small $\Delta w_{S_3}$} and {\small $\Delta w_{S_2}$} (or {\small $\Delta w_{S_1}$)}.

\begin{figure}[!t]
	\centering
	\includegraphics[width=1.0\linewidth]{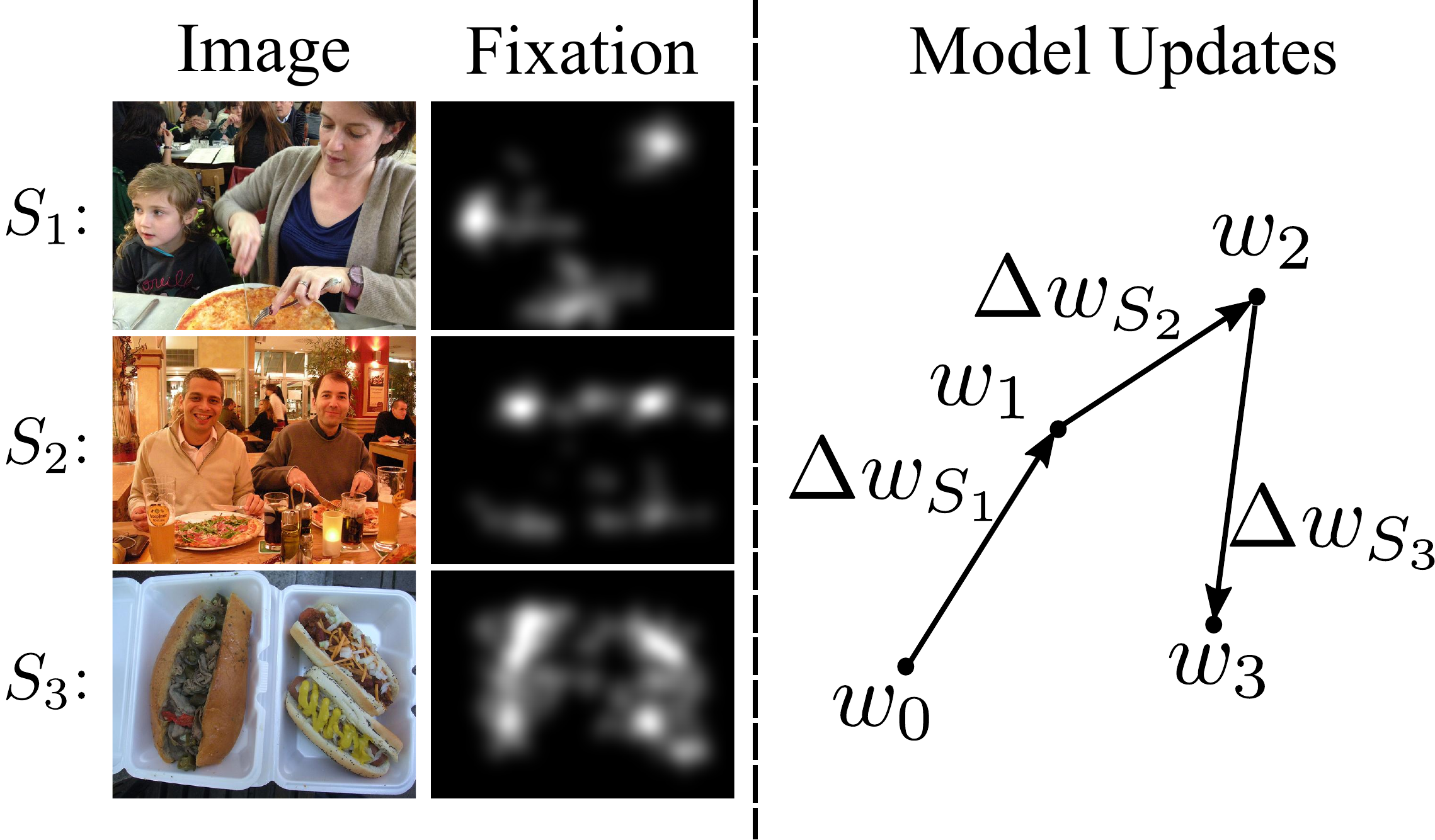}
	\caption{An illustration of congruency in the saliency prediction task. Assuming training samples are provided in a sequential manner, an incongruency occurs since the food item is related to different saliency values across these samples. Here, {\small $S_j$} stands for sample {\small $j = \{1,2,3\}$}, {\small $w_{i}$} is the weight at time step $i$, {\small $\Delta w_{S_j}$} is the weight update generated with {\small $S_j$} for {\small $w_{i}$}, and the arrows indicate updates for the model. Specifically, {\small $\Delta w_{S_j}=-\eta g_{S_{j}}$} where $\eta$ is the learning rate and {\small $g_{S_{j}}$} is the gradient w.r.t. {\small $S_{j}$}. The update of {\small $S_2$} (\ie {\small $\Delta w_{S_2}$}) is congruent with {\small $\Delta w_{S_1}$}, whereas {\small $\Delta w_{S_3}$} is incongruent with {\small $\Delta w_{S_1}$} and {\small $\Delta w_{S_2}$}.}
	\label{fig:concept}
\end{figure}

Congruency reflects the diversity of task-specific semantics in training samples (\ie images and the corresponding ground-truths). In the visual attention task, attention is explained by various hypotheses~\cite{Bruce_JOV_2007,Carpenter_CVIU_1998,Treisman_CP_1980} and can be affected by many factors, such as bottom-up feature, top-down feature guidance, scene structure, and meaning \cite{Wolfe_NHB_2017}. As a result, objects in the same category may exhibit disagreements with each other in various images in terms of attracting attention. Therefore, there is a high variability in the mapping between visual appearance and the corresponding fixations. Another task that has a considerable amount of diversity is continual learning, which is able to learn continually from a stream of data that is related to new concepts (\ie unseen labels)~\cite{Lopez_NIPS_2017}. The diversity of the data among multiple classification subtasks may be so much discrepant such that learning from new data violates previously established knowledge (\ie catastrophic forgetting) in the learning process. Moreover, congruency can also be found in the classification task. Compared to saliency prediction and continual learning, the source of diversity in classification task is relatively simple, namely, diverse visual appearances w.r.t. various labels in the real-world images. In summary, saliency prediction, continual learning, and classification are challenging scenarios susceptible to the effects of congruency.

In machine learning, congruency can be considered as a factor that influences the convergence of optimization methods, such as stochastic gradient descent (SGD) \cite{Robbins_AMS_1951}, RMSProp \cite{Hinton_CO_2012}, or Adam \cite{Kingma_arXiv_2014}. Without specific rectification, the diversity among training samples is implicitly and passively involved in a learning process and affects the descent direction in convergence. To understand the effects of congruency on convergence, we explicitly formulate a direction concentration learning (DCL) method by sensing and restricting the angle of deviation between an update gradient and a referential gradient that indicates the descent direction according to the previous updates. 
Inspired by Nesterov's accelerated gradient \cite{Nesterov_DAN_1983}, we consider the accumulated gradient as the referential gradient in the proposed DCL method.

We comprehensively evaluate the proposed DCL method with various models and optimizers in saliency prediction, continual learning, and classification tasks. The experimental results show that the constraints restricting the angle deviation between the gradient for an update and the accumulated referential gradient can help the learning process to converge efficiently, comparing to the approaches without such constraints. Furthermore, we present the congruency patterns to show how the task-specific semantics affect congruency in a learning process. Last but not least, our analysis shows that enhancing congruency in continual learning can improve backward transfer.

The main contributions in this work are as follows:
\begin{itemize}
	\item We define {\it congruency} to quantify the agreement between new information and the learned knowledge in a learning process, which is useful to understand the model convergence in terms of tractability.
	\item We propose a direction concentration learning (DCL) method to enhance congruency so that the disagreement between new information and the learned knowledge can be alleviated.
	It also generally adapts to various optimizers (\eg SGD, RMSProp and Adam) and various tasks (\eg saliency prediction, continual learning and classification).
	\item The experimental results from continual learning task demonstrate that enhancing congruency can improve backward transfer. Note that large negative backward transfer is known as catastrophic forgetting~\cite{Lopez_NIPS_2017}.
	\item A general method analyzing congruency is presented and it can be used within both conventional models and models with the proposed DCL method. Comprehensive analyses w.r.t saliency prediction and classification show that our DCL method generally enhances the congruencies of the corresponding learning processes.
\end{itemize}

The rest of the paper is organized as follows. 
We begin by highlighting related works in Section~\ref{sec:related}. Then, we formulate the problem of congruency and discuss its factors in Section~\ref{sec:problem}. The proposed DCL method is introduced in Section~\ref{sec:method}. Moreover, the experiments and analyses are provided in Section~\ref{sec:experiment} and \ref{sec:aly}, respectively. Section~\ref{sec:conclusion} concludes the paper.

\section{Related Works}
\label{sec:related}

\subsection{State-of-the-art Models for Classification}
Convolutional networks (ConvNets) \cite{Krizhevsky_NIPS_2012,Huang_CVPR_2017,He_CVPR_2016,Xie_CVPR_2017} have exhibited their powers in the classification task. AlexNet \cite{Krizhevsky_NIPS_2012} is a typical ConvNet and consists of a series of convolutional, pooling, activation, and fully-connected layers, it achieves the best performance on ILSVRC 2012 \cite{Deng_CVPR_2009}. Since then, there are more and more attempts to delve into the architecture of ConvNets. He \etal proposed residual blocks to solve the vanishing gradient problem and the resulting model, \ie ResNet \cite{He_CVPR_2016}, achieves best performance on ILSVRC 2015. Along with a similar line of ResNet, ResNeXt \cite{Xie_CVPR_2017} is proposed to extend residual blocks to multi-branch architecture and DenseNet \cite{Huang_CVPR_2017} is devised to establish the connections between each layer and later layers in a feed-forward fashion. Both models achieve desirable performance. 
\REVISION{Recently, Tan and Le \cite{Tan_ICML_2019} study how network depth, width, and resolution influence the classification performance and propose EfficientNet that achieves state-of-the-art performance on ImageNet.}
In this work, we use ResNet, ResNeXt, DenseNet, and EfficientNet in the image classification experiments.
 
Yang \etal \cite{Yang_TMM_2012} introduce a regularized feature selection framework for multi-task classification. Specifically, the trace norm of a low rank matrix is used in the objective function to share common knowledge across multiple classification tasks. Congruency generally works with gradient based optimization methods, whereas trace norm works with a specific optimization method. Moreover, congruency measures the agreement (or disagreement) between new information learned from a sample and the established knowledge, whereas trace norm is based on the weights of multiple classifiers and only measures the correlation between established knowledge w.r.t. different classification tasks.
 
\subsection{Computational Modelling of Visual Attention}
Saliency prediction is an attentional mechanism that focuses limited perceptual and cognitive resources on the most pertinent subset of the available sensory data. Itti~\etal~\cite{Itti_PAMI_1998} implement the first computational model to predict saliency maps by integrating bottom-up features. Recently, Huang~\etal~\cite{Huang_ICCV_2015} propose a data-driven DNN model, named SALICON, to model visual attention. 
Cornia~\etal \cite{Cornia_TIP_2016} propose a convolutional LSTM to iteratively refine the predictions and Kummerer~\etal~\cite{Kummerer_ICCV_2017} design a readout network that is fed with the output features of VGG \cite{Simonyan_arXiv_2014} to improve saliency prediction. 
\REVISION{Yang \etal \cite{Yang_arXiv_2019} introduce an end-to-end Dilated Inception Network (DINet) to capture multi-scale contextual features for saliency prediction and achieves state-of-the-art performance. In this work, we adopt the SALICON model and DINet in the saliency prediction experiments.}

There are several insightful works~\cite{Underwood_QJEP_2006,Gordon_JEP_2004,Underwood_EM_2007} exploring the effects of congruency/incongruency in visual attention. In particular, according to the perception experiments, Gordon finds that the object which is inconsistent with the scene, \eg a live chicken standing on a kitchen table, has significant influence on attentive allocation~\cite{Gordon_JEP_2004}. Underwood and Foulsham~\cite{Underwood_QJEP_2006} find an unexpected interaction between saliency and negative congruency in the search task, that is, the congruency of the conspicuous object does not influence the delay in its fixation, but it is fixated earlier when the other object in the scene is incongruent. Furthermore, Underwood~\etal~\cite{Underwood_EM_2007} investigate whether the effects of semantic inconsistency appear in free viewing. In their studies, inconsistent objects were fixated for significantly longer duration than consistent objects. These works inspire us to explore the congruency between the current and previous updates. In saliency prediction, negative congruency may result from the disagreement among the training samples in terms of visual appearance and ground-truth.

\subsection{Catastrophic Forgetting}
Catastrophic forgetting problem has been extensively studied in \cite{Mccloskey_PLM_1989,Ratcliff_PR_1990,French_TCS_1999,Goodfellow_arXiv_2013}. McCloskey and Cohen \cite{Mccloskey_PLM_1989} study the problem that new learning may interfere catastrophically with old learning when models are trained sequentially. New learning may alter weights that are involved in representing old learning, and this may lead to catastrophic interference. Along the same line, Ratcliff~\cite{Ratcliff_PR_1990} further investigates the causes of catastrophic forgetting, and two problems are observed: 1) sequential learning is prone to rapidly forget well-learned information as new information is learned; 2)  discrimination between observed samples and unobserved samples either decreases or is non-monotonic as a function of learning. 
To address the catastrophic forgetting problem, there are several works~\cite{Rebuffi_CVPR_2017,Kirkpatrick_PNAS_2017,Lopez_NIPS_2017} proposed to solve the problem by using episodic memory. Kirkpatrick \etal \cite{Kirkpatrick_PNAS_2017} propose an algorithm named elastic weight consolidation (EWC), which can adjust learning to minimize changes in parameters important for previously seen task. Moreover, Lopez and Ranzato \cite{Lopez_NIPS_2017} introduce the gradient episodic memory (GEM) method to alleviate catastrophic forgetting problem. However, there could exist incongruency in the training process of GEM.
\section{Congruency in Machine Learning}
\label{sec:problem}

\subsection{Problem Statement}

We first review the general goal in machine learning. Without loss of generality, given a training set {\small $D=\{(I_{i},y_{i})\}^{n}_{i=1}$}, where a pair {\small $(I_{i},y_{i})$} represents a training sample composed of an image {\small $I_{i} \in \mathbb{R}^{N_{I}}$} ($N_{I}$ is the dimension of images) and the corresponding ground-truth {\small $y_{i}\in \mathcal{Y}$}, the goal is to learn a model {\small $f: \mathbb{R}^{N_{I}} \xrightarrow[]{} \mathcal{Y}$}. Specifically, a Deep Neural Network (DNN) model has a trunk net to generate discriminative features {\small $x_{i} \in \mathcal{X}$} and a classifier {\small $f_{w}: \mathcal{X} \xrightarrow[]{w} \mathcal{Y}$} to fulfill the task, where $w$ is the weights of classifier. Note that we consider that DNN is a classifier as whole and the input is raw RGB images.

To accomplish the learning process, the conventional approach is to first specify and initialize a model. 
Next, the empirical risk minimization (ERM) principle~\cite{Vapnik_TNN_1999} is employed to find a desirable $w$ w.r.t. $f$ by minimizing a loss function {\small $\ell: \mathcal{Y}\times \mathcal{Y}\rightarrow [0,\infty)$} penalizing prediction errors, 
\ie {\small $\underset{w}{\text{minimize}}\ \frac{1}{|D|}\sum_{(x_{i},y_{i})\in D}^{}\ell(f_{w}(x_{i}),y_{i})$}. 
At time step $k$, the gradient computed by the loss is used to update the model, \ie {\small $w_{k+1} \coloneqq w_{k} + \Delta w_{k}$}, where {\small $\Delta w_{k}$} is an update as well as a function of gradient {\small $g(w_{k};x_{k},y_{k}) = \nabla_{w_{k}}\ell (f_{w_{k}}(x_{k}),y_{k})$}. Optimizers, such as SGD~\cite{Robbins_AMS_1951}, RMSProp~\cite{Hinton_CO_2012}, or Adam~\cite{Kingma_arXiv_2014}, determine {\small $\Delta w_{k}(g(w_{k};x_{k},y_{k}))$}. Without loss of generality, we assume the optimizer is SGD in the following for convenience.

There exist two challenges w.r.t. congruency for practical use. First, due to the dynamic nature of the learning process, 
how to find a stable referential direction which can quantify the agreement between current and previous updates. Second, how to guarantee the referential direction is beneficial to search for a local minimum.

As the gradient at a training step implies the direction towards a local minimum by the currently observed mini-batch, the accumulation of all previous gradients provides an overall direction towards a local minimum. Hence, it provides a good referential direction to measure the agreement between a specific update and its previous updates. We denote the accumulated gradient as
\begin{align}
\hat{g}_{k|w_{m}} = \sum_{i=m}^{k}g_{i},
\label{eqn:accm_grad}
\end{align}
where $w_{m}$ is the weights learned at time step $m$ and {\small $\hat{g}_{k|w_{m}}$} indicates that the accumulation starts from $w_{m}$ at time step k. If there is no explicit $w_{m}$ indicated, {\small $\hat{g}_{k} = \hat{g}_{k|w_{1}}$}. \fig~\ref{fig:illus} shows an example of accumulated gradient, 
where the gradient of $S_3$ deviates from the accumulated gradient of $S_1$ and $S_2$.
This also elicits our solution to measure congruency in a training process.

\begin{figure*}[t!]
	\centering
	\includegraphics[width=1.0\textwidth]{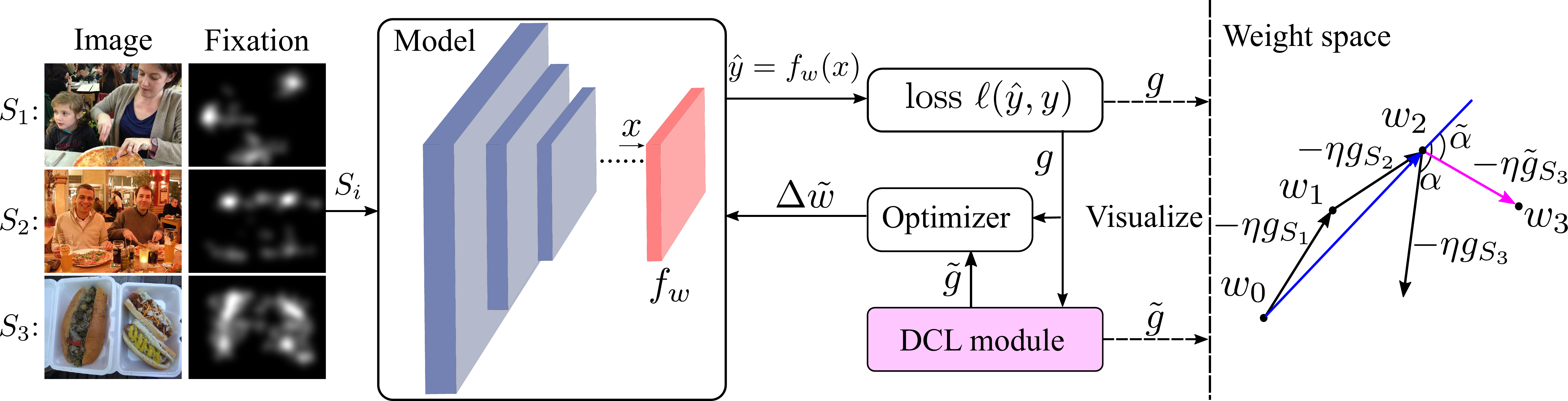} 
	\caption{An illustration of model training with the proposed DCL module. Here, 3 samples are observed in a sequential manner. The gradient generated by $S_3$ is expected to be different with the gradients generated by $S_1$ and $S_2$. Hence, to tackle the expected violation between the update $-\eta g_{S_3}$ and the accumulated update $-\eta \sum_{i=1}^{2}g_{S_i}$, the proposed DCL method finds a corrected update $-\eta \tilde{g}_{S_3}$ (the pink arrow) by solving a quadratic programming problem (\ref{eqn:obj}). In this way, the angle between $-\eta \tilde{g}_{S_3}$ and $-\eta \sum_{i=1}^{2}g_{S_i}$ (the blue arrow), \ie $\tilde{\alpha}$, is guaranteed to be equal to or less than $\alpha$. Note that the gradient descent processes with or without the proposed DCL module is identical in the test phase. }
	\label{fig:illus}
\end{figure*}

\subsection{Definition} \textit{Congruency} $\nu$ is a metric to measure the agreement of updates in a training process. In general, it is directly related to an angle between the gradient for an update and the accumulated gradient, \ie $\alpha(\hat{g}_{k-1|w_{m}},g_{k})\in [ 0,\pi ]$. Smaller angle indicates higher congruency.
Practically, we use cosine similarity to approximate the angle for computational simplicity. Mathematically, at time step $k$, $\nu_{k}$ can be defined as follows
\begin{align}
\nu_{k|w_{m}} = \cos \alpha(g_{k},\hat{g}_{k-1|w_{m}}) = \frac{\hat{g}_{k-1|w_{m}}^{\top} g_{k}}{\|\hat{g}_{k-1|w_{m}}\| \|g_{k}\|},\ m \le k
\label{eqn:cong_k}
\end{align}
where $w_{m}$ is the weight learned at time step $m$ and taken as a reference point in weight space.
$\alpha(\hat{g}_{k},g_{k})$ is the angle between $\hat{g}_{k}$ and $g_{k}$. Based on $\nu_{k-1|w_{m}}$, the congruency of a training process that starts from $w_{j}$ to learn out $w_{n}$ can be defined as
\begin{align}
\nu_{ w_{j}\rightarrow w_{n}|w_{m} } = \frac{1}{n-j+1}\sum_{i=j}^{n} \nu_{i|w_{m}},\ m \le j < n
\label{eqn:cong}
\end{align}
Since the concept of congruency is built upon cosine similarity, $\nu_{k|w_{m}}$ will range from $[-1,1]$. Another advantage of using cosine similarity is the tractability. The gradient computed from the loss is considered as a vector in weight space. Hence, cosine similarity can take any pair of gradients, such as the accumulated gradient and the gradient computed by a training sample, or two gradients computed by two respective training samples. 

\subsection{Task-specific Factors}
Congruency is semantics-aware. As congruency is based on the gradients which are computed with the images and the semantic ground-truth, such as class label in the classification task or human fixation in the saliency prediction task. Therefore, congruency reflects the task-specific semantics. We discuss congruency task-by-task in the following subsection.

\noindent\textbf{Saliency Prediction}. 
Visual attention is attracted to visually salient stimuli and is affected by many factors, such as scale, spatial bias, context and scene composition, oculomotor constraints, and so on. These factors result in high variabilities over fixations across various persons. The variabilities of visual semantics imply that same class objects in two images may have different salience levels, \ie one object is predicted as salient object while the other same class object is not. In this sense, negative congruency in learning for saliency prediction may result from both feature-level and label-label disagreement across the images.

\noindent\textbf{Continual Learning}. In the continual learning setting \cite{Rebuffi_CVPR_2017,Kirkpatrick_PNAS_2017,Lopez_NIPS_2017}, a classification model is learned with past observed classes and samples. New samples w.r.t. the unobserved classes may be distinct from previously seen samples in terms of both visual appearance and label. This leads to negative congruency in learning.

\noindent\textbf{Classification}. For classification, the class labels are usually deterministic to human. The factors that cause negative congruency in learning lie in visual appearances. 
Due to the variability of real-world images, visual appearance of samples from the same class may be very different from each other in different images. 

\section{Methodology}
\label{sec:method}

In this section, we first overview the proposed DCL method. Then, we introduce its formulation and properties in detail. Finally, we discuss the lower bound of congruency with gradient descent methods. For simplicity, we assume it is at time step $k$ and omit underscored $k$ in the following formulations unless we explicitly indicate it.

\subsection{Overview}
\fig~\ref{fig:illus} demonstrates the basic idea of the proposed DCL method. Given training sample $(I,y)$, where $I$ is an image and $y$ is the ground-truth, the corresponding feature $x$ are first generated by the sample before it is passed to the classifier for computing the predictions $\hat{y}=f_{w}(x)$. Conventionally, the derivatives $g$ of the loss $\ell(\hat{y},y)$ are computed to determine the update $\Delta w$ by an optimizer to back-propagate the error layer by layer. In the proposed DCL method, $g$ is taken to estimate a corrected gradient $\tilde{g}$ that is congruent with previous updates. For example, as shown in \fig~\ref{fig:illus}, the gradient of $S_3$ is expected to have a large deviation angle $\alpha$ to the accumulated anti-gradient {\small $-\sum_{i=1}^{2}g_{si}$} because $S_1$ and $S_2$ share similar visual appearance, but $S_3$ is different from them. The proposed DCL method aims to estimate a corrected $\tilde{g}$ which has a smaller deviation angle $\tilde{\alpha}$ to {\small $-\sum_{i=1}^{2}g_{S_i}$}.

\subsection{Direction Concentration Learning}
The core idea of the proposed DCL method is to concentrate the current update to a certain search direction. 
The accumulated gradient $\hat{g}$ is the direction voted by previous updates which provides information towards the possible local minimum. Ideally, according to the definition of congruency, \ie \eqn (\ref{eqn:cong_k}) and (\ref{eqn:cong}), cosine similarity should be considered in optimization. However, minimizing cosine similarity with constraints is complicated.
Therefore, similar to GEM \cite{Lopez_NIPS_2017}, we adopt an alternative that optimizes the inner product, instead of the cosine similarity. According to \eqn (\ref{eqn:cong_k}), {\small $\left<g_{1} ,g_{2}\right> \ge 0$} indicates that the angle between the two vectors is less than or equal to \ang{90}.

As shown in \fig~\ref{fig:illus}, the proposed DCL method uses the accumulated gradient as a referential direction to compute a corrected gradient $\tilde{g}$, \ie

\begin{align}
\begin{split}
\underset{\tilde{g}}{\text{minimize}} \ \ &\frac{1}{2} \| \tilde{g}-g \|_{2}^{2}\\
\text{s.t. } \ \ &\langle -\hat{g}_{r_{i}}, -\tilde{g} \rangle \ge 0, \ \ 1 \le i \le N_{r}
\end{split}
\label{eqn:obj}
\end{align}
where $r_{i}$ is a reference point in weight space, $\hat{g}_{r_{i}}$ is the accumulated gradient that starts the accumulation from $r_{i}$, and $N_{r}$ is the number of reference points. The accumulated gradient $\hat{g}_{r_{i}}$ indicates that the accumulation starts from the reference $r_{i}$ to the current weights $w$. The proposed DCL method can take $N_{r}$ points as the references $\{r_{i}| 1 \le i \le N_{r}\}$. Assume that the weights at time step $t$ is taken as the reference $r_{i}$, \ie $r_{i} = w_{t}$, we denote $sub(\cdot)$ as a function to find the index of a point in weight space. For example, with $t = sub(r_{i}) = sub(w_{t})$, we can compute the accumulated gradient $\hat{g}_{r_{i}} = \sum_{j=sub(r_{i})}^{} g_{j}$. On the other hand, the function $\frac{1}{2} \| \tilde{g}-g \|_{2}^{2}$ is widely used in gradient-based methods \cite{Lopez_NIPS_2017,Knight_arXiv_2018,Tao_PAMI_2009,Salimans_NIPS_2016,Hoi_ICML_2014} and forces $\tilde{g}$ to be close to $g$ in Euclidean space as much as possible. The constraints $\langle -\hat{g}_{r_{i}}, -\tilde{g} \rangle \ge 0$ are to guarantee that the gradient that is used for an update should not substantially deviate from the accumulated gradient.

In practice, instead of directly computing $\hat{g}_{r_{i}}$ by its definition (\ref{eqn:accm_grad})), 
we compute it by subtracting the current point $w$ with the reference point $r_{i}$, \ie $\hat{g}_{r_{i}} = w-r_{i} = -\eta \sum_{j=i}^{} g_{j}$. Hence, the constraints can be deformed in a matrix form
\begin{align}
\begin{split}
A(-\tilde{g}) = -1\times \begin{bmatrix}
(w-r_{1})^{\top} \\
(w-r_{2})^{\top} \\
\vdots \\
(w-r_{N_{r}})^{\top}
\end{bmatrix} \tilde{g} \ge 0
\end{split}
\label{eqn:constraints}
\end{align}
\fig \ref{fig:illu_grad} demonstrates the effect of constraints in optimization. The dashed line in the same color indicates the border of feasible region with regards to $-\hat{g}_{r_{i}}, i\in\{1,2\}$ as Constraint (\ref{eqn:constraints}) forces $\tilde{g}$ to have an angle smaller than \ang{90}. Due to two references in this example, the intersection between two feasible regions, \ie the shaded region, is the intersected feasible region for optimization. 
Note that an accumulated gradient determines half-plane (hyperplane) as feasible region, instead of the full plane (hyperplane) in conventional gradient descent case.

\begin{figure}[!t]
	\centering
	\includegraphics[width=0.76\columnwidth]{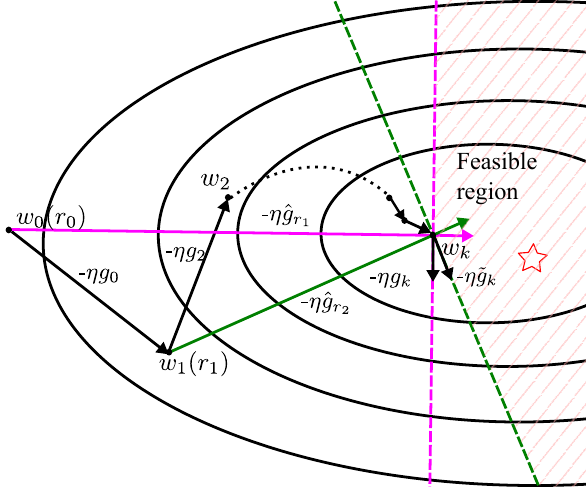}
	\caption{An illustration of DCL constraints with two reference points $r_{0} = w_{0}, r_{1} = w_{1}$. $\hat{g}_{r_{0}}$ is the pink arrow while $\hat{g}_{r_{1}}$ is the green one. The colored dashed line indicates the border of feasible region with regards to $-\hat{g}_{r_{i}}, i\in\{0,1\}$, since Constraint (\ref{eqn:constraints}) forces $-\eta\tilde{g}_{k}$ to have an angle which is smaller than or equal to \ang{90} w.r.t. $\hat{g}_{r_{0}}$ and $\hat{g}_{r_{1}}$.}
	\label{fig:illu_grad}
\end{figure}

The optimization (\ref{eqn:obj}) becomes a classic quadratic programming problem and we can easily solve it by off-the-shelf solvers like quadprog\footnote{\url{https://github.com/rmcgibbo/quadprog}} or CVXOPT\footnote{\url{https://cvxopt.org/}}. However, since the size of $\tilde{g}$ can be sufficiently large, straightforward solution may be computationally expensive in terms of both time and storage. As introduced by Dorn \cite{Dorn_QAM_1960}, we apply a primal-dual method for quadratic programs to solve it efficiently.

Given a general quadratic problem, it can be formulated as follows
\begin{equation}
\underset{z}{\text{minimize}}  \ \ \frac{1}{2} z^{\top} C z + q^{\top} z \ \ \ \text{s.t.} \ \ Bz\ge b,
\label{eqn:primal}
\end{equation}
whereas the corresponding dual problem to Problem (\ref{eqn:primal}) is 
\begin{align}
\begin{split}
\underset{u,v}{\text{minimize}} \ \ & \frac{1}{2} u^{\top} C u + b^{\top} v \\ 
\text{s.t.} \ \ & B^{\top}v-Cu = q, \ \  v \ge 0.
\end{split}
\label{eqn:dual}
\end{align}
Dorn provides the proof of the connection between Problem (\ref{eqn:primal}) and Problem (\ref{eqn:dual}).
\begin{theorem}[Duality]
	if $z=z^{*}$ is a solution to Problem (\ref{eqn:primal}) then a solution $(u, v)=(u^{*}, v^{*})$ exists to Problem (\ref{eqn:dual}). Conversely, if a solution $(u, v)=(u^{*}, v^{*})$ to Problem (\ref{eqn:dual}) exists then a solution which satisfies $Cz=Cu^{*}$ to Problem (\ref{eqn:primal}) also exists.
\end{theorem}
Due to the equality constraint $B^{\top}v-Cu = q$, assume $C$ is full rank, we can plug $u = C^{-1}(B^{\top}v-q)$ back to the objective function to further simplify Problem (\ref{eqn:dual}), \ie 
\begin{align}
\begin{split}
\underset{v}{\text{minimize}} \ \ & \frac{1}{2} v^{\top}B(C^{-1})^{\top}B^{\top}v+(b-p^{\top}B^{\top})v \\ 
\text{s.t.} \ \ & v \ge 0.
\end{split}
\label{eqn:sim_dual}
\end{align}
Now it turns out to be a quadratic problem w.r.t. $v$ only.

The DCL quadratic problem can be solved by the aforementioned primal-dual method. Specifically, $\|\tilde{g}-g\|^{2}_{2}=(\tilde{g}-g)^{\top}(\tilde{g}-g)=\tilde{g}^{\top}\tilde{g} - 2g^{\top}\tilde{g} + g^{\top}g$. By omitting the constant term $g^{\top}g$, it turns to a quadratic problem form $\tilde{g}^{\top}\tilde{g} - 2g^{\top}\tilde{g}$. Since we know the primal problem (\ref{eqn:primal}) can be converted to its dual problem (\ref{eqn:dual}), the related coefficient matrices/vectors are easily determined by
\begin{align*}
\begin{split}
C = I, \ \ \ \ B = -A, \ \ \ \ b = \bm{0}, \ \ \ \ \ p = -g,
\end{split}
\end{align*}
where $I$ is a unit matrix. With these coefficients at hand, we have the corresponding dual problem 
\begin{equation}
\underset{v}{\text{minimize}} \ \ \frac{1}{2} v^{\top}AA^{\top}v-g^{\top}A^{\top}v  \\\ \text{s.t.} \ \ v \ge 0.
\label{eqn:dcl_dual}    
\end{equation}

By solving (\ref{eqn:dcl_dual}), we have $v^{*}$. On the other hand, {\small $C\tilde{g}=Cu^{*}, C=I$} and we can have the solution $\tilde{g}^{*}$ by
\begin{equation}
\tilde{g}^{*} = Cu^{*} = B^{\top}v-q = -A^{\top}v+g 
\label{eqn:dcl_solu}
\end{equation}
Note that {\small $\tilde{g}, u \in \mathbb{R}^{p}$}, {\small $v \in \mathbb{R}^{N_{r}}$}, {\small $A\in \mathbb{R}^{N_{r}\times p}$}, and {\small $b\in \mathbb{R}^{N_{r}}$} where $p$ is the size of $w$. If taking the fully-connected layer of ResNet as $w$, $p=2048$. In contrast with $p$, $N_{r}$ is usually smaller, \ie 1,2, or 3. As $N_{r}$ becomes larger, it increases the possibility that the constraints are inconsistent.
Thus, {\small $N_{r}\ll p$}. This implies that solving Problem~(\ref{eqn:dcl_dual}) in $\mathbb{R}^{N_{r}}$ is more efficient than solving Problem~(\ref{eqn:obj}) in $\mathbb{R}^{p}$.

\subsection{Theoretical Lower Bound}
\label{subsec:lbound}

	Here, we discuss about the congruency lower bound with gradient descent methods. First, we recall the theoretical characteristics w.r.t. gradient descent methods.
	\begin{proposition}[Quadratic upper bound \cite{Nesterov_Springer_2014}]
		If the gradient of a function $f: \mathbb{R}^{n} \rightarrow \mathbb{R}$ is Lipschitz continuous with Lipschitz constant $L$ for any $x,y\in \mathbb{R}^{n}$, \ie 
		\begin{align}
		\|\nabla f(y) - \nabla f(x)\| \le L \| y-x \|
		\end{align}
		then
		\begin{align}
		f(y) \le f(x) + \nabla f(x)^{\top} (y-x) + \frac{L}{2}\| y-x \|^{2}
		\end{align}
		\label{thm::p1}
	\end{proposition}
	\noindent On the other hand, there is a proved bound w.r.t. the loss.
	\begin{corollary}[The bound on the loss at one iteration \cite{FG_Lecture_2016,Tibshirani_Lecture_2013}]
		Let $x_{k}$ be the $k$-th iteration result of gradient descent and $\eta_{k} \ge 0$ the $k$-th step size. If $\nabla f$ is $L$-Lipschitz continuous, then
		\begin{align}
		f(x_{k+1}) \le f(x_{k}) -\eta_{k}\left( 1-\frac{L\eta_{k}}{2} \right) \|\nabla f(x_{k})\|^{2}
		\end{align}
		\label{thm::c1}
	\end{corollary}
	\noindent By adding up a collection of inequalities, we can move further along this line to have the following corollary.
	\begin{corollary}[]
		Let $x_{k}$ be the $k$-th iteration result of gradient descent and $\eta_{k} \ge 0$ the $k$-th step size. If $\nabla f$ is $L$-Lipschitz continuous, then
		\begin{align}
		f(x_{k}) \le f(x_{0}) -\sum_{i=0}^{k-1} \eta_{i}\left( 1-\frac{L\eta_{i}}{2} \right) \|\nabla f(x_{i})\|^{2}
		\end{align}
		\label{thm::c2}
	\end{corollary}
	\begin{theorem}[Congruency lower bound]
		Assume the gradient descent method uses a fixed step size $\eta$ and the gradient of the loss function $f: \mathbb{R}^{n} \rightarrow \mathbb{R}$ is Lipschitz continuous with Lipschitz constant $L$, the congruency $\nu_{k|x_{0}}$ referring to the initial point $x_{0}$ at the $k$-th iteration has the following lower bound
		\begin{align}
		\begin{split}
		\nu_{k|x_{0}} \ge & \max \Big\{ (1 -L\eta) \sum_{i=0}^{k-1} \frac{\| \nabla f(x_{i})\|}{\|\nabla f(x_{k})\| } \\
		& - L\eta \frac{\sum_{i=0}^{k-1} \| \nabla f(x_{i}) \| \| \sum_{j=0}^{i-1}  \nabla f(x_{j}) \|}{\|\nabla f(x_{k})\| \|\sum_{i=0}^{k-1} \nabla f(x_{i})\|}, -1 \Big\}
		\end{split}
		\end{align}
		\label{thm::t1}
	\end{theorem}
	\begin{proof}[\textbf{Proof}:\nopunct]
		Given $x_{k}$ and $x_{0}$, according to Proposition \ref{thm::p1} we have
		{ \small
			\begin{align*}
			\nabla f(x_{k})^{\top} (x_{k}-x_{0}) \le f(x_{k}) - f(x_{0}) + \frac{L}{2}\| x_{k}-x_{0} \|^{2}
			\end{align*}
		}
		Since {\small $x_{k} = x_{0}-\eta \sum_{i=0}^{k-1} \nabla f(x_{i})$} and {\small $\nu_{k|x_{0}} = (-\nabla f(x_{k}))^{\top} (-\sum_{i=0}^{k-1} \nabla f(x_{i}))/(\|\nabla f(x_{k})\| \|\sum_{i=0}^{k-1} \nabla f(x_{i})\|)$}, we can have
		{ \small
			\begin{align*}
			\nabla f(x_{k})^{\top} (x_{k}-x_{0}) = &-\eta (-\nabla f(x_{k}))^{\top} (-\sum_{i=0}^{k-1} \nabla f(x_{i})) \\
			                                     = & -\eta \|\nabla f(x_{k})\| \|\sum_{i=0}^{k-1} \nabla f(x_{i})\| \nu_{k|x_{0}}
			\end{align*}
		}
		Plugging this in the inequality, it yields 
		{ \small
			\begin{align*}
			\nu_{k|x_{0}} \ge \frac{1}{\eta} \frac{f(x_{0}) - f(x_{k}) - \frac{L\eta^{2}}{2}\| \sum_{i=0}^{k-1} \nabla f(x_{i}) \|^{2}}{\|\nabla f(x_{k})\| \|\sum_{i=0}^{k-1} \nabla f(x_{i})\|}
			\end{align*}
		}
		According to Corollary \ref{thm::c2}, the inequality can be rewritten as
		{ \footnotesize
			\begin{align}
			\nu_{k|x_{0}} \ge \frac{(1 -\frac{L\eta}{2}) \sum_{i=0}^{k-1} \|\nabla f(x_{i})\|^{2} - \frac{L\eta}{2}\| \sum_{i=0}^{k-1} \nabla f(x_{i}) \|^{2}} {\|\nabla f(x_{k})\| \|\sum_{i=0}^{k-1} \nabla f(x_{i})\|} \label{eqn:proof_1}
			\end{align}
		}
		By using polynomial expansion and the Cauchy-Schwarz inequality, we can expand the term $\| \sum_{i=0}^{k-1} \nabla f(x_{i}) \|^{2}$ as follows
		{ \small
			\begin{align*}
			&\| \sum_{i=0}^{k-1} \nabla f(x_{i}) \|^{2} = \|\nabla f(x_{k-1}) + \sum_{i=0}^{k-2} \nabla f(x_{i}) \|^{2} \\
			\le& \| \nabla f(x_{k-1}) \|^{2} + 2 \| \nabla f(x_{k-1}) \| \| \sum_{i=0}^{k-2} \nabla f(x_{i}) \| + \| \sum_{i=0}^{k-2} \nabla f(x_{i}) \|^{2} \\
			\end{align*}
		}
		Recursively, $\| \sum_{i=0}^{k-2} \nabla f(x_{i}) \|^{2}$, $\| \sum_{i=0}^{k-3} \nabla f(x_{i}) \|^{2}$, $\ldots$, till $\| \sum_{i=0}^{1} \nabla f(x_{i}) \|^{2}$ can be expanded, \eg
		{ \small
			\begin{align*}
			&\| \sum_{i=0}^{1} \nabla f(x_{i}) \|^{2} = \|\nabla f(x_{1}) + \nabla f(x_{0}) \|^{2} \\
			\le&\| \nabla f(x_{1}) \|^{2} + 2 \| \nabla f(x_{1}) \| \| \nabla f(x_{0}) \| + \| \nabla f(x_{0}) \|^{2} \\
			\end{align*}
		}
		The above inequalities yield
		{ \small
			\begin{align*}
			\| \sum_{i=0}^{k-1} \nabla f(x_{i}) \|^{2} &\le \sum_{i=0}^{k-1} \| \nabla f(x_{i}) \|^{2} + 2 \sum_{i=0}^{k-1} \| \nabla f(x_{i}) \| \| \sum_{j=0}^{i-1}  \nabla f(x_{j}) \|
			\end{align*}
		}
		Plugging it into Inequality (\ref{eqn:proof_1}), we have 
		{ \small
			\begin{align*}
			\nu_{k|x_{0}} \ge & (1 -L\eta)\frac{ \sum_{i=0}^{k-1} \|\nabla f(x_{i})\|^{2} } {\|\nabla f(x_{k})\| \|\sum_{i=0}^{k-1} \nabla f(x_{i})\|} \\
			&-L\eta \frac{\sum_{i=0}^{k-1} \| \nabla f(x_{i}) \| \| \sum_{j=0}^{i-1}  \nabla f(x_{j})\|}{\|\nabla f(x_{k})\| \|\sum_{i=0}^{k-1} \nabla f(x_{i})\|}
			\end{align*}
		}
		Due to {\small $\frac{\sum_{i=0}^{k-1} \| \nabla f(x_{i})\|}{\|\sum_{i=0}^{k-1} \nabla f(x_{i})\|} \ge 1 $}, the congruency lower bound can be further simplified as
		{ \small
			\begin{align*}
			\nu_{k|x_{0}} \ge & (1 -L\eta) \sum_{i=0}^{k-1} \frac{\| \nabla f(x_{i})\|}{\|\nabla f(x_{k})\| } \\
			& - L\eta \frac{\sum_{i=0}^{k-1} \| \nabla f(x_{i}) \| \| \sum_{j=0}^{i-1}  \nabla f(x_{j}) \|}{\|\nabla f(x_{k})\| \|\sum_{i=0}^{k-1} \nabla f(x_{i})\|}
			\end{align*}
		}
		Combining with the fact $\nu_{k|x_{0}} \ge -1$, we complete the proof.
	\end{proof}

\begin{figure}[!t]
	\centering
	\includegraphics[width=0.63\columnwidth]{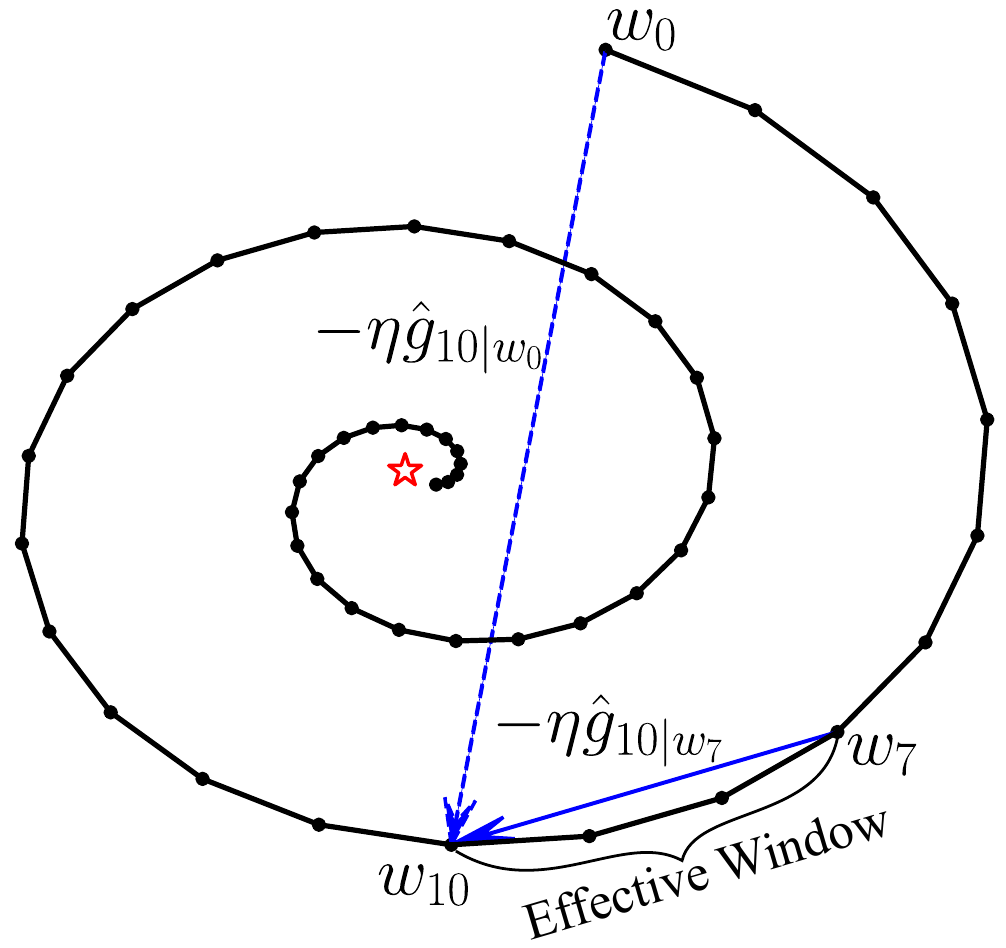}
	\caption{An illustration to demonstrate the concept of the effective window. 
		Given the spiral convergence path, $-\eta \hat{g}_{10|w_{0}}$ restricts the search direction and the minimum (\ie the red star) and $w_{11}$ are unreachable according to the search direction. In contrast, $w_{11}$ can be reached along the search direction of $-\eta \hat{g}_{10|w_{7}}$. To adaptively yield appropriate accumulated gradients that converge to the minimum, we define an effective window to periodically update the reference.}
	\label{fig:unfavour_case}
\end{figure}

	\begin{remark}
		Theorem~\ref{thm::t1} implies that when we apply gradient descent method to search a local minimum, the congruency lower bound at a certain iteration in the learning process is determined by the gradients at current iteration and previous iterations. 
	\end{remark}
	\begin{remark}
		Theorem~\ref{thm::t1} implies that the lower bound of congruency with a small step size, \ie {\small $\eta < \frac{1}{L}$}, is tighter than the one of congruency with a large step size, \ie {\small $\eta \ge \frac{1}{L}$}. This is consistent with the fact the large step size could lead to a zigzag convergence path. The negative lower bound of congruency when {\small $\eta \ge \frac{1}{L}$} indicates the huge turnaround would possibly occur in the learning process.
	\end{remark}

\begin{figure*}[!t]
	\centering
	\subfloat[Convergence with GD at iteration 50]      {\includegraphics[width=0.32\linewidth]{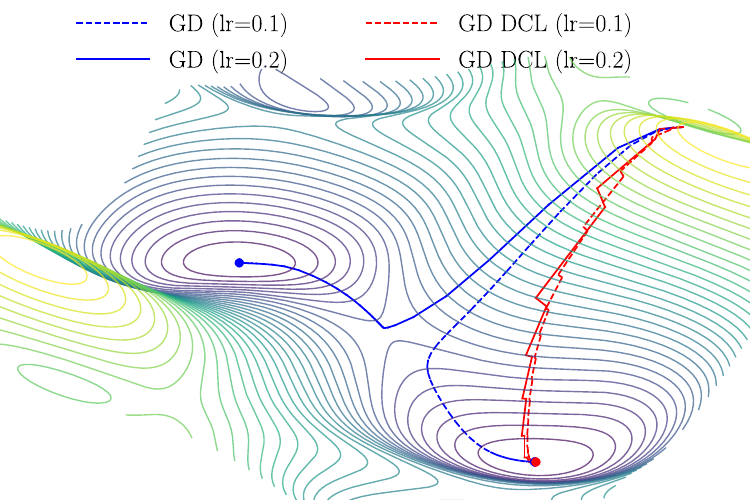}      } \hfill
	\subfloat[Convergence with RMSProp at iteration 150]{\includegraphics[width=0.32\linewidth]{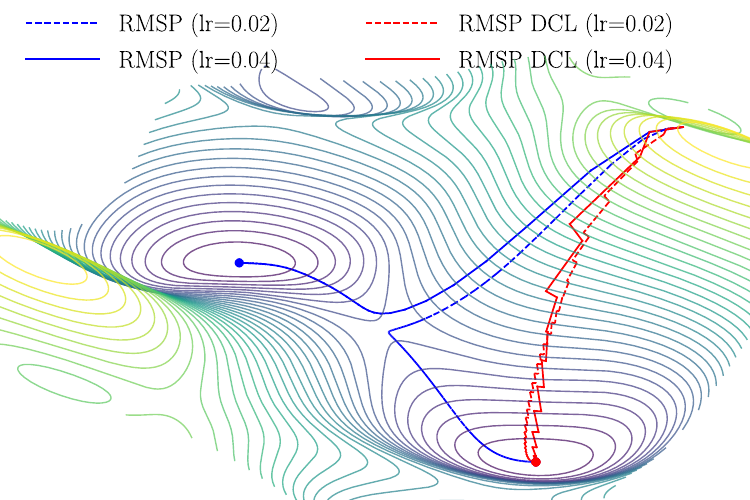}    } \hfill
	\subfloat[Convergence with Adam at iteration 200]   {\includegraphics[width=0.32\linewidth]{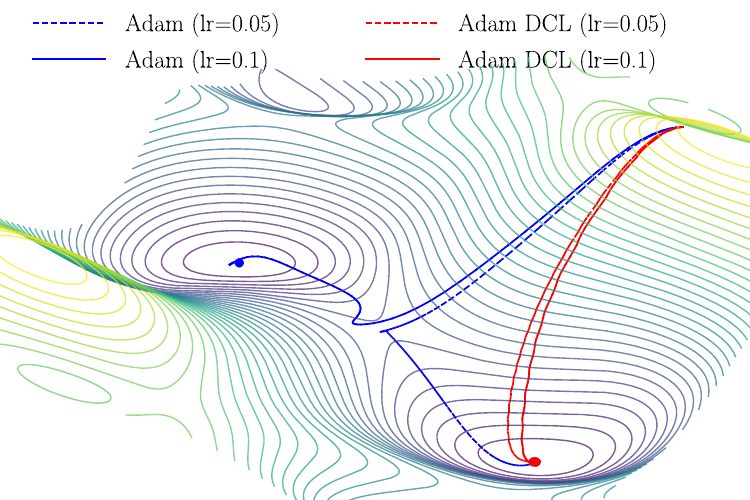}    } \\
	\subfloat[Curves of $z$ v.s. iteration with GD]     {\includegraphics[width=0.32\linewidth]{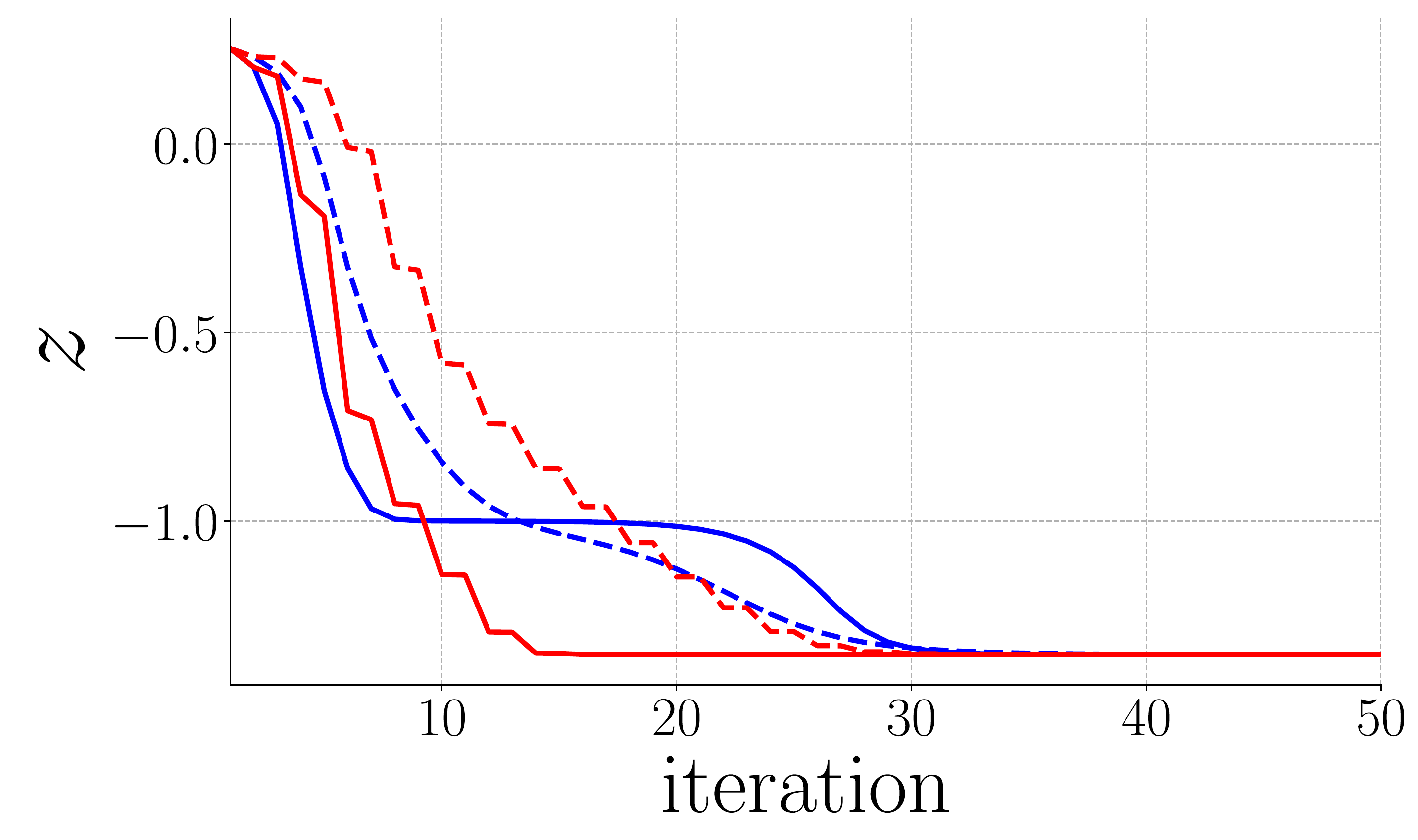}  } \hfill
	\subfloat[Curves of $z$ v.s. iteration with RMSProp]{\includegraphics[width=0.32\linewidth]{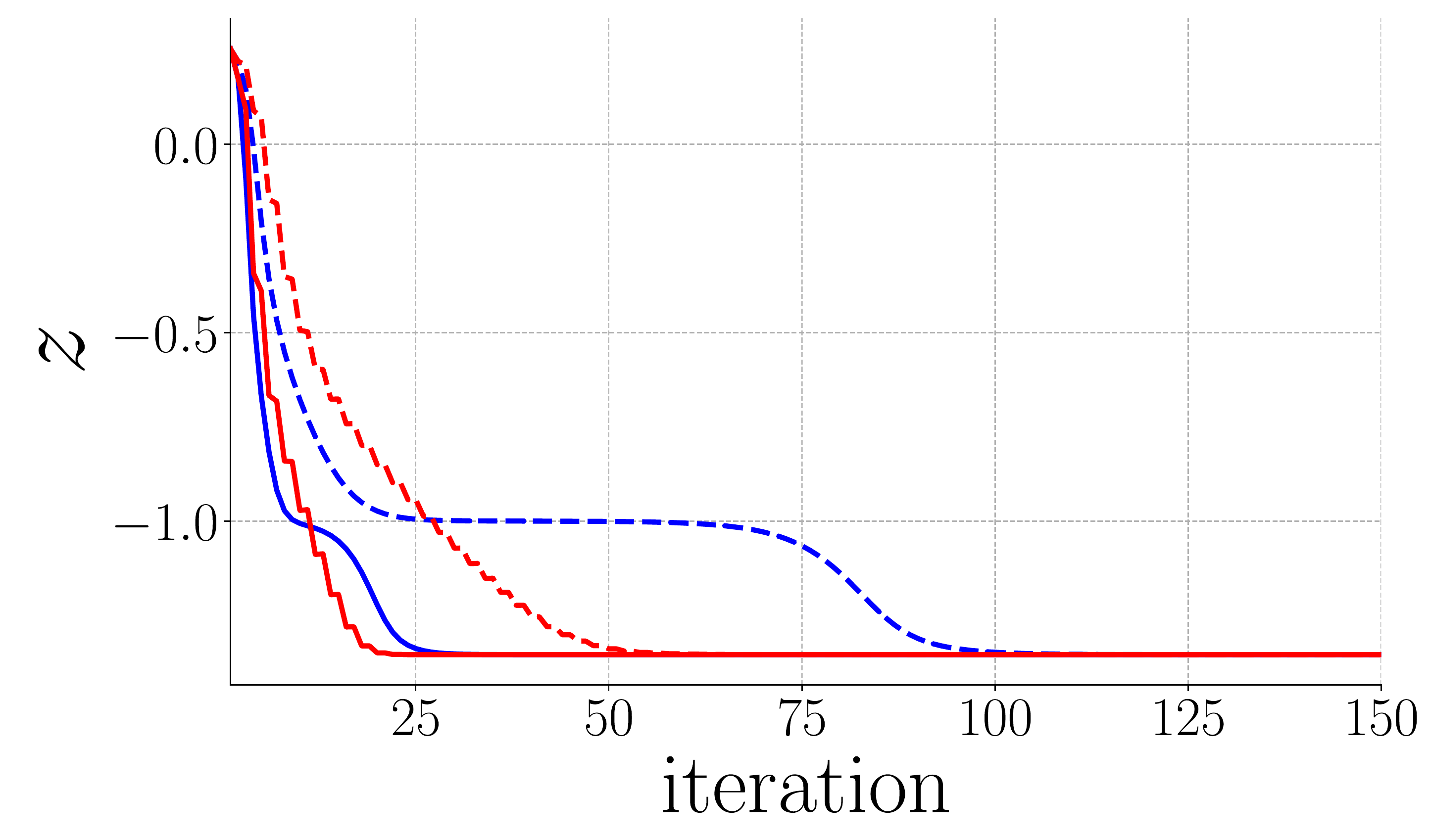}} \hfill
	\subfloat[Curves of $z$ v.s. iteration with Adam]   {\includegraphics[width=0.32\linewidth]{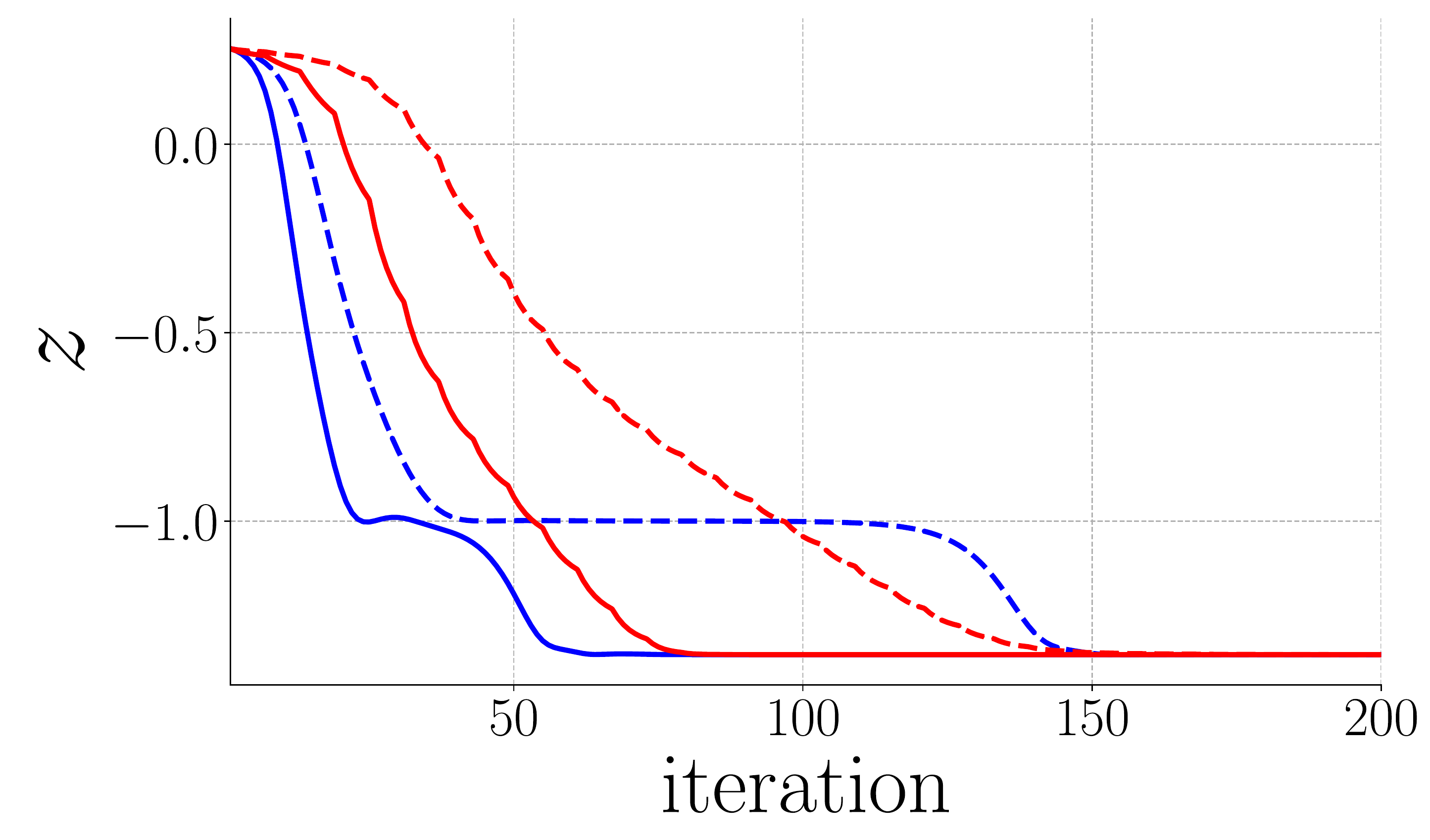}}
	\caption{An example demonstrating the effect of the proposed DCL method on three optimizers, \ie~gradient descent (GD), RMSProp, and Adam.
		Given a problem {\small $z=f(x,y)$}, we use these optimization algorithms to compute the local minima, \ie $(x^{*},y^{*})$ that yield the minimal $z^{*}$. In the experiment, except the learning rate, the setting and hyperparameters are the same for ALGO and ALGO DCL, where ALGO=\{GD, RMSProp, Adam\}.
		The proposed DCL method encourages the convergence paths to be as straight as possible.}
	\label{fig:effect}
\end{figure*}

\subsection{Adaptivity to Learning}

As the reference is used to compute the accumulated gradient for narrowing down the search direction, a desirable referential direction should orient to a local minimum. Conversely, an inappropriate referential direction could mislead the training and slow down the convergence. Therefore, it is important to update the references to adapt to the target optimization problem.

In this work, we update the references with a short temporal window so as to yield a locally stable and reliable referential direction. For instance, \fig \ref{fig:unfavour_case} shows an unfavorable case that takes $w_{0}$ as the reference, 
where the convergence path is spiral.
Due to the circuitous manifold, $w_{0}$ results in a misleading direction $-\eta \hat{g}_{10|w_{0}}$. In contrast, if taking $w_{7}$ as a reference, it can yield the appropriate search direction to reach $w_{11}$. Therefore, we introduce an ``effective window'' to allow the proposed DCL method to find an appropriate search direction. The effective window forces the proposed DCL method to only accumulate the gradients within the window. 
In \fig \ref{fig:unfavour_case}, the proposed DCL method with a small window size would converge, whereas the one with a large window size would diverge.
We denote the window size as $\beta_{w}$ and the reference offset as $\beta_{o}$. When the time step $t$ satisfies
\begin{align}
\begin{split}
t \; \texttt{mod} \; \beta_{w} = \beta_{o},
\end{split}
\label{eqn:adapt_win}
\end{align}
where $\texttt{mod}$ is the modulo operator, it would trigger the reset mechanism, \ie starting over to set references {\small $r_{i} \leftarrow w_{t}, \ 1 \le i \le N_{r}$}. $\beta_{o}$ indicates the first reference weight point. Once the reset process starts, the proposed DCL method would use $g$, instead of $\tilde{g}$, for update until all the $N_{r}$ references are reset.

\subsection{Effect of DCL}
\label{subsec:effect}
To intuitively understand the effect of the proposed DCL method, we present visual comparisons of the convergence paths with three popular optimizers, \ie SGD \cite{Robbins_AMS_1951}, RMSProp \cite{Hinton_CO_2012}, and Adam \cite{Kingma_arXiv_2014}, on a publicly available problem\footnote{\url{https://github.com/Jaewan-Yun/optimizer-visualization}}. 

In particular, given the problem $z=f(x,y)$, we apply the three optimizers to compute a local minimum $(x^{*}, y^{*})$. Unlike image classification, the problem does not need randomized data sequence as input so there is no stochastic process. For a fair comparison, except the learning rate, we keep the settings and hyperparameters the same between ALGO and ALGO DCL, where ALGO=\{GD, RMSProp, Adam\} and GD stands for gradient descent. The convergence paths w.r.t. the optimization algorithms are shown in \fig~\ref{fig:effect}(a)-(c), while the corresponding $z$ v.s. iteration curves are plotted in \fig~\ref{fig:effect}(d)-(f).

We can see that all the baseline curves are circuitous, \ie a sharp turn at the ridge region between two local minima. Moreover, different learning rates lead to different local minima. It implies that the training process in this case is influenceable and fickle in terms of the direction of the convergence. The proposed DCL method noticeably improves the convergence direction by choosing a relatively straightforward path over the three optimization algorithms. Note that as the objective function (\ref{eqn:obj}) implies, if we do not take any the accumulated gradients (\ie no constraints), or take the gradient for the coming update as the accumulated gradient (\ie $\hat{g}_{r_{i}}=g$), the proposed DCL method would become the baseline (\ie $\tilde{g}=g$).

\subsection{DCL in Continual Learning}
\label{subsec:dcl}

In previous subsections, we introduce the proposed DCL method in mini-batch learning. By its very nature, it can also work in continual learning manner. GEM \cite{Lopez_NIPS_2017} is a recent method proposed for continual learning. The objective function of GEM is the same as the proposed DCL method, whereas the constraints of GEM and the proposed DCL method are devised for respective purposes. To apply the proposed DCL method in continual learning, we can merge the constraints of the proposed method with the ones of GEM. Hence, we have a new $A$ as follows
\begin{align}
\begin{split}
A = \begin{bmatrix}
(w-r_{1})^{\top} \\
\vdots \\
(w-r_{N_{r}})^{\top} \\
-g(x_{S_1},y_{S_1})^{\top}\\
\vdots \\
-g(x_{S_{N_{m}}},y_{S_{N_{m}}})^{\top} \\
\end{bmatrix}, \ \ \  S_i \in \mathcal{M}
\end{split}
\label{eqn:constraints_conn}
\end{align}
where $\mathcal{M}$ is the memory and $N_{m}$ is the size of the memory. With the proposed DCL constraints, the corrected $\tilde{g}$ is forced to be consistent with both the accumulated gradients and the directions of gradients generated by the samples in memory.

\subsection{Comparison with Memory-based Constraints}

Now, we discuss the difference between the proposed DCL constraints and the memory-based constraints used in GEM~\cite{Lopez_NIPS_2017}.

There are two main differences between the DCL constraints and the GEM constraints.
First, as shown in \fig \ref{fig:comparison}, the descent direction in the proposed DCL method is regulated by the accumulated gradient, whereas the gradient for an update in GEM is regulated to avoid the violation with the gradients of the memory samples (\ie images and the corresponding ground-truths). Since the weights are iteratively updated and the memory samples are preserved, the gradients of the memory samples could be changed at each iteration so the direction of the adjusted gradient could be dynamically varying. Second, the proposed DCL method only needs to memorize the references, whereas GEM memorizes the images and the corresponding ground-truths. The proposed DCL constraints are efficiently computed by a subtraction in \eqn~(\ref{eqn:constraints}), other than by computing the corresponding gradients like GEM. 


Although the proposed DCL constraints are different from GEM constraints in terms of definition, they are able to work with each other in continual learning. We will dive into the details in the following experiment section. Moreover, GEM computes the gradients on all the parameters of a DNN. This works in the situations that input image resolution is relatively small, \eg 784 for MNIST~\cite{Lecun_IEEE_1998} or 3072 for CIFAR-10/100~\cite{Krizhevsky_Report_2009}. The networks used to classify these images have small number of weights like MLP and ResNet-18. However, the number of parameters in a DNN could be huge. For example, ResNeXt-29 ($16\times 64$) \cite{Xie_CVPR_2017} has 68 million parameters. Although GEM applies primal-dual method to reduce the computation in optimization, the overall computation is still considerably high. In this work, we instead compute the gradients on the highest-level layer to generalize the proposed DCL method to any general DNN.

\begin{figure}[!t]
	\centering
	\includegraphics[width=0.95\columnwidth]{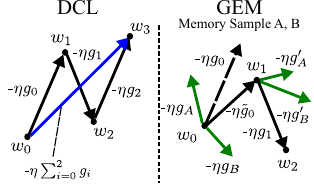}
	\caption{An illustration demonstrating the difference between DCL (left) and GEM \cite{Lopez_NIPS_2017} (right). The search direction in DCL is determined by the accumulated gradient while the adjusted gradient (solid line) of GEM is optimized by avoiding the violation between the gradient (dashed line) and memory samples' gradients (green line). Since the weights are iteratively updated and the memory samples are preserved, the direction of the adjusted gradient of the memory samples could be dynamically varying.}
	\label{fig:comparison}
\end{figure}
\section{Experiments}
\label{sec:experiment}

\subsection{Experimental Setup}
To comprehensively evaluate the proposed DCL method, we conduct experiments on three tasks, \ie saliency prediction, continual learning, and classification.

\subsubsection{Datasets} 
\noindent For saliency prediction task, we use SALICON~\cite{Jiang_CVPR_2015} (the 2017 version), MIT1003~\cite{Judd_ICCV_2009}, and OSIE~\cite{Xu_JOV_2014}. For continual learning task, we follow the same experimental settings in GEM \cite{Lopez_NIPS_2017} to use MNIST Permutations (MNIST-P), MNIST Rotations (MNIST-R), and incremental CIFAR-100 (iCIFAR-100). For classification, we use CIFAR \cite{Krizhevsky_Report_2009}, Tiny ImageNet, and ImageNet\cite{Deng_CVPR_2009}.

\begin{table*}[!t]
	\centering
	\caption
	{
		Saliency prediction performance of the models which are trained on SALICON 2017 training set and evaluated on SALICON 2017 validation set. Higher score is better in all the metrics. 
		Each experiment is repeated for 3 times and the mean and std of the scores are reported. \REVISION{We follow \cite{Yang_arXiv_2019} to only use Adam as the optimizer for DINet.}
	}
	\begin{tabular}{L{30ex} C{18ex} C{18ex} C{18ex} C{18ex}}
		\toprule
		& NSS & sAUC & AUC & CC \\
		\cmidrule(lr){2-2} \cmidrule(lr){3-3} \cmidrule(lr){4-4} \cmidrule(lr){5-5}
		ResNet-50 RMSP                                 & 1.7933$\pm$0.0083            & 0.8311$\pm$0.0017             & 0.8393$\pm$0.0039             & 0.8472$\pm$0.0048          \\
		ResNet-50 RMSP GEM                             & 1.7522$\pm$0.0150            & 0.8267$\pm$0.0017             & 0.8341$\pm$0.0016             & 0.8291$\pm$0.0033          \\
		ResNet-50 RMSP DCL-$\infty$-1                  & \textbf{1.8226}$\pm$0.0014   & \textbf{0.8376}$\pm$0.0017    & \textbf{0.8445}$\pm$0.0016    & \textbf{0.8569}$\pm$0.0032 \\ \midrule
		ResNet-50 Adam                                 & 1.7978$\pm$0.0019            & 0.8328$\pm$0.0007             & 0.8405$\pm$0.0011             & 0.8495$\pm$0.0004          \\
		ResNet-50 Adam GEM                             &  1.7962$\pm$0.0034           & 0.8344$\pm$0.0021             & 0.8399$\pm$0.0009             & 0.8494$\pm$0.0034          \\
		ResNet-50 Adam DCL-$\infty$-1                  & \textbf{1.8019}$\pm$0.0024   & \textbf{0.8360}$\pm$0.0023    & \textbf{0.8430}$\pm$0.0023    & \textbf{0.8548}$\pm$0.0038 \\ \midrule
		DINet Adam \cite{Yang_arXiv_2019}                                & 1.8786$\pm$0.0063            & 0.8426$\pm$0.0008             & 0.8489$\pm$0.0008             & 0.8799$\pm$0.0010          \\
		DINet Adam GEM                             &  1.8746$\pm$0.0067           & 0.8423$\pm$0.0014             & 0.8492$\pm$0.0012             & 0.8791$\pm$0.0030          \\
		DINet Adam DCL-500-1                  & \textbf{1.8857}$\pm$0.0006   & \textbf{0.8430}$\pm$0.0002    & \textbf{0.8493}$\pm$0.0002    & \textbf{0.8804}$\pm$0.0009 \\
		\bottomrule	
	\end{tabular}
	\label{tbl:salicon}
\end{table*}

\begin{table*}[!t]
	\centering
	\caption
	{
		Saliency prediction performance of the models which are trained on OSIE and tested on MIT1003. 
		Each experiment is repeated for 3 times and the mean and std of the scores are reported.
	}
	\begin{tabular}{L{30ex} C{18ex} C{18ex} C{18ex} C{18ex}}
		\toprule
		& ~~~NSS~~ & ~~sAUC~~ & ~AUC~ & ~~CC~~ \\
		\cmidrule(lr){2-2} \cmidrule(lr){3-3} \cmidrule(lr){4-4} \cmidrule(lr){5-5}
		ResNet-50 RMSP                & 2.4047$\pm$0.0055          & 0.7612$\pm$0.0019          & 0.8455$\pm$0.0028          & 0.7595$\pm$0.0002          \\
		ResNet-50 RMSP GEM            & 2.3960$\pm$0.0057          & 0.7566$\pm$0.0045          & 0.8412$\pm$0.0055          & 0.7500$\pm$0.0037          \\
		ResNet-50 RMSP DCL-$\infty$-1 & \textbf{2.4252}$\pm$0.0053 & \textbf{0.7620}$\pm$0.0018 & \textbf{0.8469}$\pm$0.0027 & \textbf{0.7658}$\pm$0.0016 \\ \midrule
		ResNet-50 Adam                & 2.4064$\pm$0.0015          & 0.7597$\pm$0.0012          & 0.8429$\pm$0.0021          & 0.7618$\pm$0.0005 \\
		ResNet-50 Adam GEM            & 2.3685$\pm$0.0065          & 0.7594$\pm$0.0007          & 0.8427$\pm$0.0017           & 0.7524$\pm$0.0011 \\
		ResNet-50 Adam DCL-$\infty$-1 & \textbf{2.4108}$\pm$0.0063 & \textbf{0.7613}$\pm$0.0007 & \textbf{0.8442}$\pm$0.0008 & \textbf{0.7617}$\pm$0.0007 \\ 
		\midrule
		DINet Adam                & 2.4406$\pm$0.0058          & 0.7570$\pm$0.0005          & 0.8442$\pm$0.0016          & 0.7534$\pm$0.0005 \\
		DINet Adam GEM            & 2.4456$\pm$0.0037          & 0.7571$\pm$0.0005          & 0.8432$\pm$0.0003           & 0.7540$\pm$0.0006 \\
		DINet Adam DCL-120-1 & \textbf{2.4566}$\pm$0.0007 & \textbf{0.7611}$\pm$0.0011 & \textbf{0.8476}$\pm$0.0008 & \textbf{0.7597}$\pm$0.0008 \\ 
		\bottomrule	
	\end{tabular}
	\label{tbl:osie}
\end{table*}

\subsubsection{Models} 
For saliency prediction, we adopt an improved SALICON saliency model \cite{Huang_ICCV_2015} and DINet \cite{Yang_arXiv_2019} as the baselines. \REVISION{Both the baseline models takes ResNet-50 \cite{He_CVPR_2016} as the backbone architecture.}

For continual learning, we adopt the same models used in GEM, \ie Multiple Layer Perceptron (MLP) and ResNet-18, \REVISION{as well as EfficientNet-B1~\cite{Tan_ICML_2019}} as the backbone architecture for evaluation. EWC \cite{Kirkpatrick_PNAS_2017} and GEM are used for comparison.

For classification, we use the state-of-the-art model without any architecture modifications for a fair evaluation. ResNeXt \cite{Xie_CVPR_2017} (\ie ResNeXt-29), DenseNet \cite{Huang_CVPR_2017} (\ie DenseNet-100-12), \REVISION{and EfficientNet-B1 \cite{Tan_ICML_2019}} are used in the evaluation of CIFAR-10 and CIFAR-100. ResNet (\ie ResNet-101), DenseNet (\ie DenseNet-169-32), \REVISION{and EfficientNet-B1 \cite{Tan_ICML_2019}} are used in the experiments on Tiny ImageNet. ResNet (\ie ResNet-34 and ResNet-50) is used in the experiments on ImageNet.

\subsubsection{Notation} 
For convenience, we notate \textit{model name} + \textit{optimizer name} + \textit{DCL-$\beta_{w}$-$N_{r}$} for key experimental details in \tab \ref{tbl:salicon}, \ref{tbl:osie}, \mbox{\ref{tbl:batch_cifar} and \ref{tbl:tiny_imgnet}}. $\beta_{w}=\infty$ indicates it never resets the references when the initialization of references is finished.

\subsubsection{Evaluation Metrics} 
For saliency prediction, we report the performance using the commonly use metrics, namely area under curve (AUC)~\cite{Borji_TIP_2013,Judd_Report_2012}, shuffled AUC (sAUC) \cite{Borji_TIP_2013,Seo_JOV_2009}, normalized scanpath saliency (NSS)~\cite{Rothenstein_IVC_2008,Itti_ASNN_2003}, and correlation coefficient (CC)~\cite{Ouerhani_ELCVIAs_2004}. Human fixations are used to form the positive set while the points from the saliency map are sampled to form the negative set. With the two sets, an ROC curve of true positive rate v.s. false positive rate would be plotted by thresholding over the saliency map. If the points are sampled in a uniform distribution, it is AUC. If the points are sampled from the human fixation points, it is sAUC. NSS would average the response values at human eye positions in an predicted saliency map which has been normalized to be zero-mean and with unit standard deviation. CC measures the strength of a linear correlation between a ground-truth map and a predicted saliency map. 
For continual learning, we use the same metrics used in GEM \cite{Lopez_NIPS_2017}, \ie accuracy, backward transfer (BWT), and forward transfer (FWT).
For classification, we evaluate the proposed DCL method with top 1 error rate metric on the CIFAR experiments while both top 1 and top 5 error rate are reported in the experiments of Tiny ImageNet and ImageNet. 

\subsubsection{Experimental \& Training Details}
\label{subsec:exp}
In the experiments of saliency prediction, we use Adam \cite{Kingma_arXiv_2014} and RMSProp (RMSP) \cite{Hinton_CO_2012} optimizers. In the setting with Adam, we use $\eta=0.0002$, weight decay 1e-5 while $\eta=0.0005$, weight decay 1e-5 are used within the setting of RMSP. The momentum is set to 0.9 for both Adam and RMSP. $\eta$ would be adjusted along with the epochs, \ie $\eta_{k+1} \leftarrow \eta_{0}\times 0.5^{k-1}$, where $k$ is the current epoch. The batch size is 8 by default. To fairly evaluate the performances of the models, we use cross-dataset validation technique,
\ie the models are trained on the SALICON 2017 training set and evaluated on the SALICON 2017 validation set, and trained on OSIE and evaluated on MIT1003.

We follow the experimental settings in \cite{Lopez_NIPS_2017} for continual learning. Specifically, MNIST-P and MNIST-R have 20 tasks and each task has 1000 examples from 10 different classes. On iCIFAR-100, there are 20 tasks and each task has 2500 examples from 5 different classes. For each task, the first 256 training samples will be selected and stored as the memory on MNIST-P, MNIST-R, and iCIFAR-100. In this work, GEM constraints are concatenated with the DCL constraints by \eqn~(\ref{eqn:constraints_conn}). As the different concepts are learned across the episodes, \ie the tasks, we only consider that the accumulation of gradients would take place in each episode.

In the classification task, we evaluate the models with SGD optimizer \cite{Robbins_AMS_1951}. The hyperparameters are kept by default, \ie weight decay 5e-4, initial $\eta=0.1$, the number of total epochs 300. $\eta$ would be changed to 0.01 and 0.001 at epoch 150 and 225, respectively. For the Tiny ImageNet experiments, we will train the models in 30 epochs with weight decay 1e-4, initial $\eta=0.001$. $\eta$ would be changed to 1e-4 and 1e-5 at epoch 11 and 21, respectively. The momentum is 0.9 by default. The batch size is 128 in the CIFAR experiments and 64 in the Tiny ImageNet experiments. In the ImageNet experiments, we use batch size of 512 to train ResNet-50.

In addition, 
we present the performance of GEM for reference as well. Note that more samples in memory may lead to inconsistent constraints. We set memory size to 1 and reset the memory at each epoch beginning, which is analogous to the case that GEM for continual learning would reset the memory at each beginning of the episode. The implementations of this work are built upon PyTorch\footnote{\url{https://github.com/pytorch/pytorch}} and quadprog package is employed to solve quadratic programming problems.

\subsection{Performance Evaluation}

\subsubsection{Saliency Prediction} 
\label{subsubsec:sal}
\tab~\ref{tbl:salicon} reports the mean and standard deviation (std) of the scores in NSS, sAUC, AUC, and CC over 3 runs on the SALICON 2017 validation set. We can see that the proposed DCL method overall improves the saliency prediction performance \REVISION{with both ResNet-50 and DINet} over all the metrics.
Moreover, \REVISION{small values of stds w.r.t. the proposed DCL method} show that the randomness caused by the stochastic process does not contribute much to the improvement.
\tab \ref{tbl:osie} shows that the proposed DCL method trained on OSIE consistently improves the saliency prediction performance on MIT1003. 

Note that Adam and RMSP optimizer are different algorithms to compute effective step sizes based on the gradients. The consistency of the improvement with both optimizers shows that the proposed DCL method generally works with these optimizers.

\subsubsection{Continual Learning} As introduced in Section \ref{sec:method}, we apply the proposed DCL method to enhance the congruency of the learning process for continual learning. Specifically, following \eqn (\ref{eqn:constraints_conn}), we concatenate the DCL constraints with the GEM constraints \cite{Lopez_NIPS_2017}. As reported in \tab \ref{tbl:conn_rota}, the proposed DCL method improves the classification accuracy by $0.7\%$ on MNIST-R. 
Similarly, the proposed DCL method improves the classification accuracy on MNIST-P as well (see \tab \ref{tbl:conn_perm}). The marginal improvement may results from the difference between MNIST-R and MNIST-P. Permuting the pixels of the digits is harder to recognize than rotating the digits by a fixed angle, and makes the accumulated gradient less informative in terms of leading to the solution.
We observe that shorter effective window size is helpful to improve the accuracy in the continual learning task. This is because the training process of continual learning is one-off and a fast variation could be caused by the limited images with brand new labels in each episode.
The experiments on iCIFAR-100 in \tab \ref{tbl:conn_cifar} confirm this pattern. 
The proposed DCL method with \REVISION{ResNet and} $\beta_{w}=4$ improves the accuracy by $1.25\%$ on iCIFAR-100.

\begin{table}[!t]
	\centering
	\caption
	{
		Performances on MNIST-R in continual learning setting using SGD \cite{Robbins_AMS_1951} as the optimizer. The reported accuracy is in percentage. \textit{MEM} indicates that the constraints of GEM \cite{Lopez_NIPS_2017} are concatenated to use as \eqn (\ref{eqn:constraints_conn}) describes.
	}
	\vspace{-2ex}
	\begin{tabular}{L{23ex} C{9ex} C{9ex} C{9ex}}
		\toprule
    		        & Accuracy & BWT & FWT \\
		\cmidrule(lr){2-2} \cmidrule(lr){3-3} \cmidrule(lr){4-4}
		EWC & 54.61 & -0.2087 &  0.5574 \\
		GEM & 83.35 & -0.0047 &  \textbf{0.6521} \\
		MLP DCL-30-1 MEM & \textbf{84.08} & 0.0094 & 0.6423 \\
		MLP DCL-40-1 MEM & 84.02 & 0.0127 & 0.6351 \\
		MLP DCL-50-1 MEM & 82.77 & \textbf{0.0238} & 0.6111 \\
		\bottomrule	
		\label{tbl:conn_rota}
	\end{tabular}
\end{table}

\begin{table}[!t]
	\centering
	\caption
	{
		Performances on MNIST-P in continual learning setting using SGD as the optimizer.
	}
	\vspace{-2ex}
	\begin{tabular}{L{23ex} C{9ex} C{9ex} C{9ex}}
		\toprule
		            & Accuracy & BWT & FWT \\
		\cmidrule(lr){2-2} \cmidrule(lr){3-3} \cmidrule(lr){4-4}
		EWC & 59.31 & -0.1960 &  -0.0075 \\
		GEM & 82.44 & 0.0224 &  -0.0095 \\
		MLP DCL-3-1 MEM & 82.30 & 0.0248 & \textbf{-0.0038} \\
		MLP DCL-4-1 MEM & \textbf{82.58} & 0.0402 & -0.0092 \\
		MLP DCL-5-1 MEM & 82.10 & \textbf{0.0464} & -0.0095 \\
		\bottomrule	
		\label{tbl:conn_perm}
	\end{tabular}
\end{table}

\begin{table}[!t]
	\centering
	\caption
	{
		Performances on iCIFAR-100 in continual learning setting using SGD as the optimizer. \REVISION{\textit{EffNet} stands for EfficientNet \cite{Tan_ICML_2019}.}
	}
	\vspace{-2ex}
	\begin{tabular}{L{23ex} C{9ex} C{9ex} C{9ex}}
		\toprule
		            & Accuracy & BWT & FWT \\
		\cmidrule(lr){2-2} \cmidrule(lr){3-3} \cmidrule(lr){4-4}
		EWC & 48.33 & -0.1050 &  0.0216 \\
		iCARL & 51.56 & -0.0848 &  0.0000 \\ \midrule
		ResNet GEM & 66.67 & 0.0001 &  0.0108 \\
		ResNet DCL-4-1 MEM & \textbf{67.92} & 0.0063 & 0.0102 \\
		ResNet DCL-8-1 MEM & 67.27 & 0.0104 & 0.0190 \\
		ResNet DCL-12-1 MEM & 66.58 & 0.0089 & 0.0139 \\
		ResNet DCL-20-1 MEM & 66.56 & 0.0030 & 0.0102 \\
		ResNet DCL-24-1 MEM & 64.97 & 0.0082 & \textbf{0.0238} \\
		ResNet DCL-32-1 MEM & 66.10 & \textbf{0.0305} & 0.0176 \\
		ResNet DCL-50-1 MEM & 64.86 & 0.0244 & 0.0125 \\ \midrule
		EffNet GEM & 80.80 & 0.0318 & -0.0050 \\
		EffNet DCL-4-1 MEM & \textbf{81.55} & 0.0383 & -0.0048 \\
		EffNet DCL-8-1 MEM & 80.84 & 0.0367 & \textbf{0.0068} \\
		EffNet DCL-12-1 MEM & 79.45 & 0.0322 & 0.0011 \\
		EffNet DCL-20-1 MEM & 79.33 & 0.0316 & -0.0095 \\
		EffNet DCL-24-1 MEM & 79.05 & 0.0375 & -0.0006 \\
		EffNet DCL-32-1 MEM & 79.97 & 0.0452 & -0.0145 \\
		EffNet DCL-50-1 MEM & 77.87 & \textbf{0.0602} & -0.0101 \\ 
		\bottomrule	
	\end{tabular}
	\label{tbl:conn_cifar}
\end{table}

There are another two metrics for continual learning, \ie forward transfer (FWT) and backward transfer (BWT). FWT is that learning a task is helpful in learning for the future tasks. Particularly, positive FWT is correlated to $n$-shot learning. Since the proposed DCL method utilizes the directional information of the past updates, it has less influence/correlation to FWT. Hence, we will focus on BWT. BWT is the influence that learning a task has on the performance on the previous tasks. Positive BWT is correlated to congruency in the learning process, while large negative BWT is referred as catastrophic forgetting. \tab \ref{tbl:conn_rota} and \ref{tbl:conn_perm} show that the proposed DCL method is useful in improving BWT on MNIST-R and MNIST-P. The BWT of GEM is negative (-0.0047) and the proposed DCL method improves it to 0.0238 on MNIST-R. Similarly, the BWT of GEM is 0.0224 and the proposed DCL method improves it to 0.0464 on MNIST-P. Similarly, in \tab \ref{tbl:conn_cifar}, the proposed DCL method with ResNet improves BWT of GEM from 0.0001 to 0.0305, 
\REVISION{while the proposed DCL method with EfficientNet \cite{Tan_ICML_2019} improves BWT to 0.0602.}

\subsubsection{Classification}
\label{subsubsec:cls}
\tab~\ref{tbl:batch_cifar} reports the top 1 error rates on CIFAR-10 and \mbox{CIFAR-100} with ResNeXt, DenseNet, and EfficientNet. In all cases, the proposed DCL method outperforms the baseline, \ie ResNeXt-29 SGD, DenseNet-100-12 SGD, and EfficientNet-B1 SGD. 
Specifically, the proposed DCL method with ResNeXt decreases the error rate by $0.2\%$ on CIFAR-10 and by $0.28\%$ on CIFAR-100, 
\REVISION{while the proposed DCL method with EfficientNet decreases the error rate by $0.12\%$ on CIFAR-10 and by $0.16\%$ on CIFAR-100.}
Similar improvements can be found in the experiments with DenseNet and this shows that the proposed DCL method is generally able to work with various models. Moreover, it can be seen in \tab \ref{tbl:batch_cifar} that GEM has a higher error rate than the baseline in the experiments with ResNeXt, DenseNet, \REVISION{and EfficientNet}. Because of the dynamical update process in learning, the gradient of the samples in memory does not guarantee that the direction leads to the solution. The direction can be even worse, \REVISION{\eg it is} possible to go in an opposite way to the solution.

\begin{table}[!t]
	\centering
	\caption
	{
		Top 1 error rate (in $\%$) on CIFAR with various models.
	}
	\vspace{-2ex}
	\begin{tabular}{L{34ex} C{11ex} C{11ex}}
		\toprule
		 & CIFAR-10 & CIFAR-100  \\
		\cmidrule(lr){2-2} \cmidrule(lr){3-3}
		ResNeXt-29 SGD & 3.53 & 17.30 \\
		ResNeXt-29 SGD GEM & 7.70 & 32.70 \\
		ResNeXt-29 SGD DCL-$\infty$-1  & \textbf{3.33} & \textbf{17.02} \\ \midrule
		DenseNet-100-12 SGD & 4.54 & 22.88 \\
		DenseNet-100-12 SGD GEM & 6.92 & 33.72 \\
		DenseNet-100-12 SGD DCL-90-1 & \textbf{4.32} & \textbf{22.16} \\ \midrule
		EfficientNet-B1 SGD \cite{Tan_ICML_2019} & 1.91 & 11.81  \\
		EfficientNet-B1 SGD GEM & 3.06 & 19.48 \\
		EfficientNet-B1 SGD DCL-5-1  & \textbf{1.79} & \textbf{11.65}  \\
		\bottomrule	
	\end{tabular}
	\label{tbl:batch_cifar}
\end{table}

\begin{table}[!t]
	\centering
	\caption{Top 1 and top 5 error rate (in $\%$) on the validation set of Tiny ImageNet with various models.}
	\vspace{-2ex}
	\begin{tabular}{L{34ex} C{11ex} C{11ex}}
		\toprule
		\hspace{33ex} & Top 1 error & Top 5 error \\
		\cmidrule(lr){2-2} \cmidrule(lr){3-3}
		ResNet-101 SGD & 17.34 & 4.82  \\
		ResNet-101 SGD GEM & 21.78 & 7.21 \\
		ResNet-101 SGD DCL-60-1  & \textbf{16.89} & \textbf{4.50}  \\ \midrule
		DenseNet-169-32 SGD & 20.24 & 6.11  \\
		DenseNet-169-32 SGD GEM & 26.81 & 9.43 \\
		DenseNet-169-32 SGD DCL-50-1  & \textbf{19.55} & \textbf{6.09}  \\ \midrule
		EfficientNet-B1 SGD & 15.73 & 3.90  \\
		EfficientNet-B1 SGD GEM & 28.74 & 11.31 \\
		EfficientNet-B1 SGD DCL-8-1  & \textbf{15.61} & \textbf{3.75}  \\
		\bottomrule
		\label{tbl:tiny_imgnet}
	\end{tabular}
\end{table}

\begin{figure*}[!t]
	\centering
	\begin{minipage}{1\textwidth}
		\centerline{\includegraphics[width=1.0\linewidth]{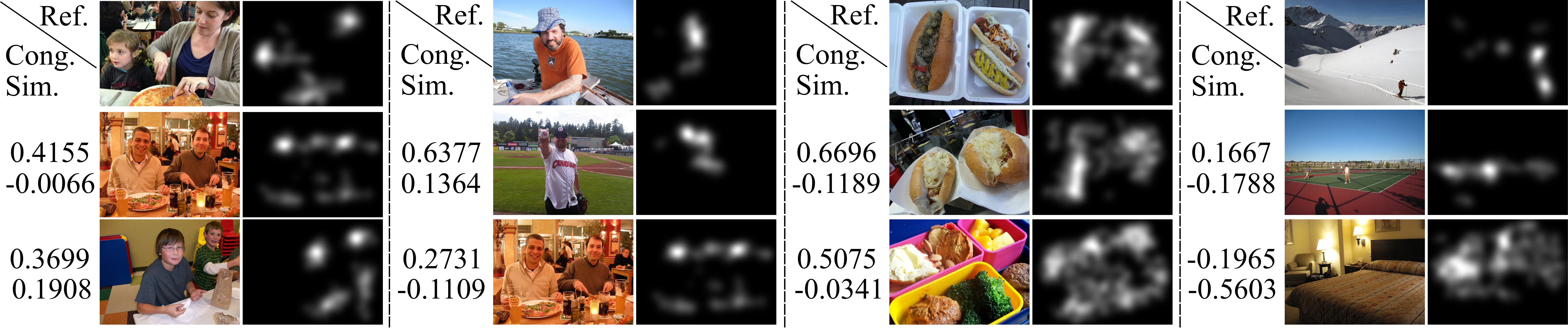}}
	\end{minipage}
	\caption{The congruencies (\textit{Cong.}) generated by the given references (\textit{Ref.}) and samples with the baseline ResNet-50 RMSP in \tab \ref{tbl:osie}. The cosine similarities (\textit{Sim.}) between referred images and sample images are provided for comparison purposes. Source images and the corresponding ground-truths, \ie fixation maps, are displayed along with the congruencies. The first and second block are the results of subset that contains persons in various scenes. The third block is examples of food subset. The rightmost block shows subset with mixed image categories, \ie contain objects of various categories in various scenes.}
	\label{fig:sal_cos}
\end{figure*}

A consistent improvement w.r.t. the proposed DCL method can be found in the experiments on Tiny ImageNet (see \tab \ref{tbl:tiny_imgnet}). 
The proposed DCL method decreases top 1 error rate by $0.45\%$ with ResNet, by $0.69\%$ with DenseNet, \REVISION{and by $0.12\%$ with EfficientNet}. Also, the performance degradation caused by GEM \cite{Lopez_NIPS_2017} can be observed that top 1 error rate generated by GEM with ResNeXt is increased by almost $4.44\%$, comparing to the baseline ResNet.

\tab~\ref{tbl:imgnet} reports the mean and std of 1-crop validation error of ResNet-50 on ImageNet. Comparing to Tiny ImageNet and CIFAR, ImageNet has more categories and more high resolution images. Given such difficulties, the proposed DCL method reduces the mean of top 1 errors by 0.24$\%$ over three runs. In summary, the improvement gained by the proposed DCL method is benefited from the better solution searched by optimizing DCL quadratic programming problem (\ref{eqn:obj}).

\section{Analysis}
\label{sec:aly}

In this section, we first validate the defined congruency by comparing through qualitative examples. Then, an ablation study w.r.t. $\beta$ and $N_{r}$ is presented. Moreover, we provide a congruency analysis in the training processes for the three tasks. In the end, the comparison between training from scratch and fine-tuning, as well as the computational cost are provided.

\subsection{Validity of Congruency Metric}
\label{subsec:validity}

In this subsection, we conduct a sanity check on the validity of the defined congruency. To do this, we consider a simple case where we directly take the gradients (\ie $g_{S_{1}}$ and $g_{S_{2}}$) of two samples (\ie $S_{1}\text{ and }S_{2}$) to compute the corresponding congruency, \ie {\small $\nu=\frac{g_{S_{1}}^{\top}g_{S_{2}}}{\|g_{S_{1}}\| \|g_{S_{2}}\|}$}. 
For comparison purposes, the cosine similarity, $Sim$, between raw image $S_{1}$ and $S_{2}$ is also computed by {\small $Sim=\frac{S_{1}^{\top}S_{2}}{\|S_{1}\| \|S_{2}\|}$}. Note that congruency is semantics-aware, whereas cosine similarity between the two raw images is semantics-blind. This is because the gradients are computed by images and its semantic ground truth, \eg the class label in the classification task or human fixation in the saliency prediction task.

\begin{table}[!t]
	\centering
	\caption{Top 1 and top 5 1-crop validation error (in $\%$) on ImageNet with SGD optimizer. $\beta_{w}=5$ and $N_{r}=1$ are used for ResNet-50 DCL. Within the same experimental settings, ResNet-50 GEM does not converge in this experiment. The mean and std of errors are computed over three runs.}
	\vspace{-2ex}
	\begin{tabular}{L{34ex} C{11ex} C{11ex}}
		\toprule
		\hspace{33ex} & Top 1 error & Top 5 error \\
		\cmidrule(lr){2-2} \cmidrule(lr){3-3}
		ResNet-50 \cite{He_CVPR_2016}  & 24.70 & 7.80  \\ \midrule
		ResNet-50 (reproduced)  & 24.33$\pm$0.08 & 7.30$\pm$0.07  \\
		ResNet-50 DCL & 24.09$\pm$0.03 & 7.23$\pm$0.02  \\
		\bottomrule
		\label{tbl:imgnet}
	\end{tabular}
\end{table}

\begin{figure*}[!t]
		\centerline{\includegraphics[width=0.98\linewidth]{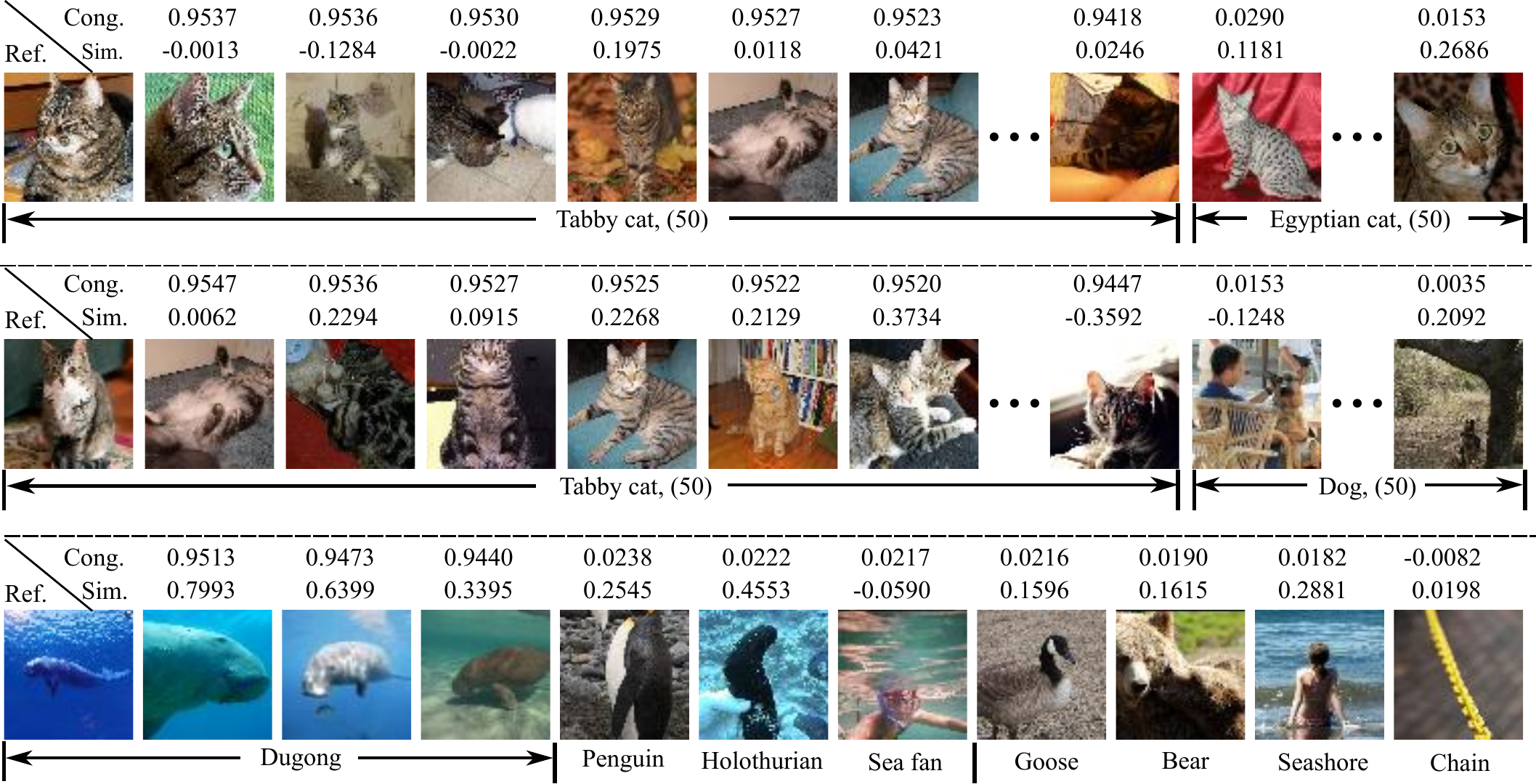}}
	\caption{The congruencies (\textit{Cong.}) generated by the given references (\textit{Ref.}) and samples with the baseline ResNet-101 SGD in \tab~\ref{tbl:tiny_imgnet}. The images with its labels are displayed along with the congruencies. The cosine similarities (\textit{Sim.}) between referred images and sample images are provided for comparison purposes. The first block is the results of the intra-similar-class subset consisting of images of tabby cat and Egyptian cat. The middle block is the results of the inter-class subset consisting of images of tabby cat and German shepherd dog. The value in bracket indicates number of images. The bottom block is the results of images of various labels.}
	\label{fig:cls_cos}
	\vspace{-1.5ex}
\end{figure*}

For the analysis in the saliency prediction task,
we sample 3 subsets,
where 20 training samples w.r.t. person, 20 training samples w.r.t. food, and 20 training samples w.r.t. various scenes and categories were sampled from SALICON. 
For the analysis in the classification task,
3 subsets were sampled from Tiny ImageNet,
which comprised of 100 images of tabby cat and Egyptian cat to form a intra-similar-class subset, 100 images of tabby cat and German shepherd dog to form a inter-class subset, and 50 images from various classes to form a mixed subset. In this way, we can analyze the correlation between the samples in terms of congruency.
With these subsets, we use the baselines, \ie ResNet-50 for saliency prediction and ResNet-101 for classification, to yield the samples gradients without updating the model.

Figure~\ref{fig:sal_cos} demonstrates the congruencies w.r.t. the references and various samples (image + fixation map). In contrast to the deterministic nature in the classification task, saliency is context-related and semantics-based. It implies that the same objects within two different scenarios may have different saliency labels. Hence, we select the examples of same/similar objects for this experiment. In Figure~\ref{fig:sal_cos}, the first and second block on left are based on the person subset within various scenarios. The first block consists of the images of person and dining table. Taking the first row sample as reference, the sample in the second row has higher congruency (0.4155) when compared to bottom row sample (0.3699). Although all the fixation maps of all the samples are different, pizza in the second image is more similar to the reference image whereas food in the bottom sample is inconspicuous. In the second block, both the portrait of the fisher (reference) and the portrait of the baseball player (second sample) are similar in terms of the layout, comparing to the persons in dining room (third sample). Their fixation maps are similar as well. 

The congruency of the reference and second sample (0.6377) are higher than the one of the reference and third sample (0.2731). 
In the third block, the image of the reference is three hot dogs and its fixation maps is similar to the fixation maps of the second sample. 
The two hog dog samples have similar visual appearance and layout of fixations to yield a higher congruency (0.6696). In contrast, third sample is different from the reference in terms of visual appearance and layout of fixations, which yields a lower congruency (0.5075).
The rightmost block shows an interesting fact that two outdoor samples yield a positive congruency $0.1667$, whereas the outdoor reference and the indoor sample yield a negative congruency $-0.1965$. One possible reason is that the fixation pattern are different between the reference and the bottom indoor sample. In addition, the visual appearance like illumination may be the another factor causing such the discrepancy.

For classification, Figure \ref{fig:cls_cos} shows the congruencies w.r.t. the references and given samples in each subset. 
In all cases, we first observe that images with same genuine class as references yield high congruency, \ie larger than 0.94 for all cases.
These show that the gradients of the same labels are similar in the direction of the updates. Another observation is that the congruency of pairs with different labels are significantly smaller than the matched label counterpart. In Figure \ref{fig:cls_cos}, the congruencies of the reference (Tabby cat) and Egyptian cat images are below $0.03$, while the congruencies of the reference and German shepherd dog images are below $0.016$ in the middle block. 
These demonstrate that the gradients of inter-class samples are nearly perpendicular to each other. The reference of class `Dugong' has positive congruencies w.r.t. all the images that fall in the category of animal, except for the image of chain, which falls into a non-animal category. Last but not least, 
given the images with different labels,
similar visual appearance would lead to relatively higher congruency. For example, the congruencies between tabby cat and Egyptian cat are overall higher than the ones between tabby cat and German shepherd dog. In summary, the labels are an important factor to influence the direction of the gradient in the classification task. Secondly, the visual appearance is another factor for congruency.

In contrast with congruency, cosine similarity between two raw images make less sense in the context of a specific task. For example, two similar dining scenes in the first column in Figure \ref{fig:sal_cos} yield a negative cosine similarity $-0.0066$ in the saliency prediction task. Similarly, the first two cat images in the first row in Figure  \ref{fig:cls_cos}, which are cast to the same category, yield a negative cosine similarity $-0.0013$. The negative cosine similarity between two images with the same or similar ground truth are counterintuitive. It results from the fact that cosine similarity between two images only focuses on the difference between two sets of pixels and ignores the semantics associated to the pixels.

\subsection{Ablation Study}
\label{subsec:ablation}
In this subsection, we study the effects of effective window size $\beta_{w}$ and reference number $N_{r}$ on saliency prediction task (with SALICON) and classification task (with Tiny ImageNet).

In the saliency prediction experiment, \fig \ref{fig:beta_sal} shows the curve of sAUC vs. $\beta_{w}$ based on DCL-$\beta_{w}$-1, while \fig \ref{fig:refnum_sal} shows the curve of sAUC vs. $N_{r}$ based on DCL-$\infty$-$N_{r}$. Note that for the reference number study, the training process on SALICON consists of 12500 iterations so $\beta_{w}\ge 12500$ is equivalent to $\beta_{w} = \infty$, which means that it never resets the references in the whole learning process. It can be observed that different $\beta_{w}$ and $N_{r}$ yield relatively similar performance in sAUC. This aligns with the nature of saliency prediction, where it maps features to the salient label and the non-salient label. The features w.r.t. the salient label are highly related to each other so $\beta_{w}$ and $N_{r}$ would pervasively help the learning process make use of congruency.

\begin{figure}[!t]
	\centering
	\subfloat[sAUC vs. $\beta_{w}$]       {\includegraphics[width=0.48\linewidth]{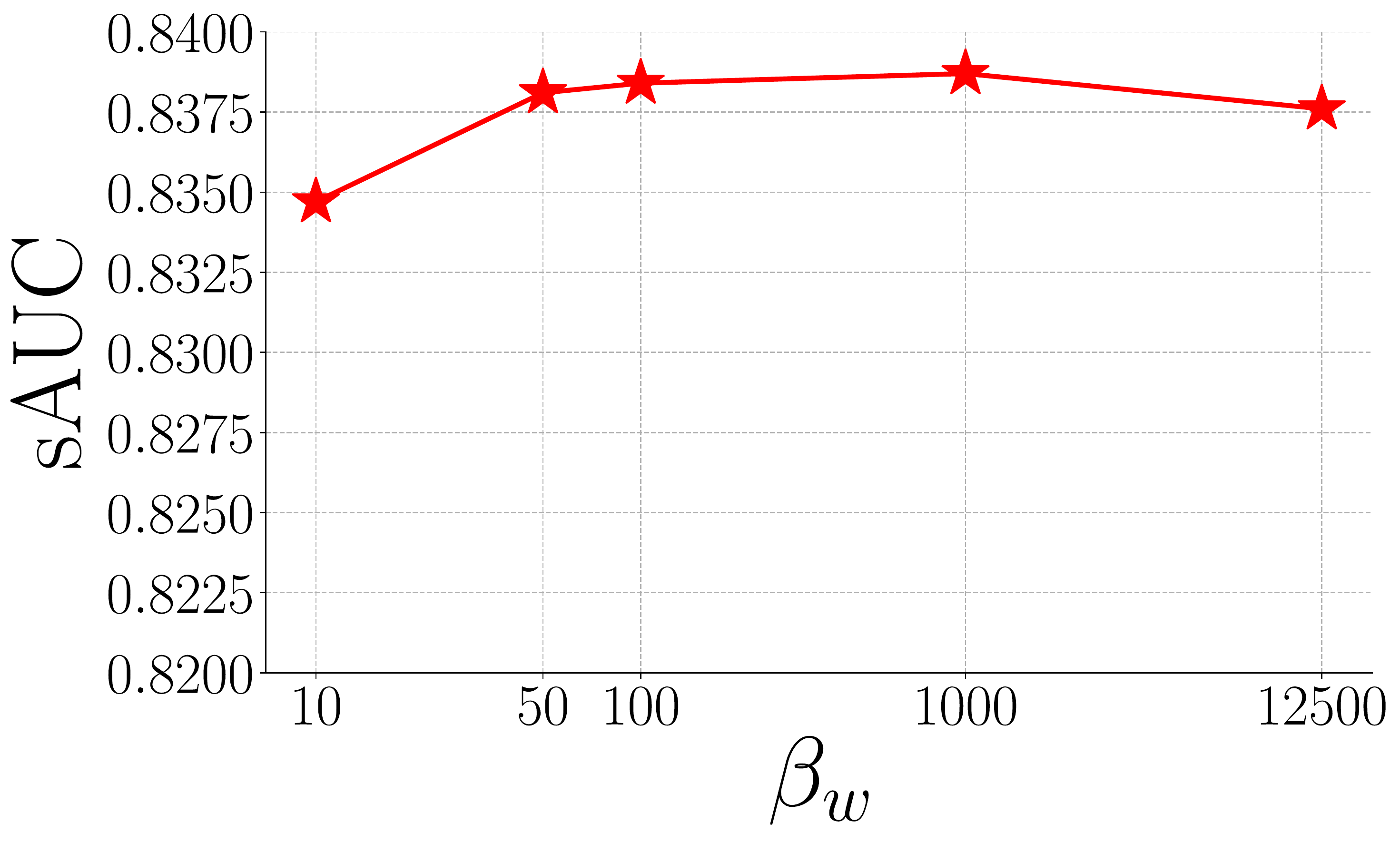} \label{fig:beta_sal}}   \hfill
	\subfloat[sAUC vs. $N_{r}$]           {\includegraphics[width=0.48\linewidth]{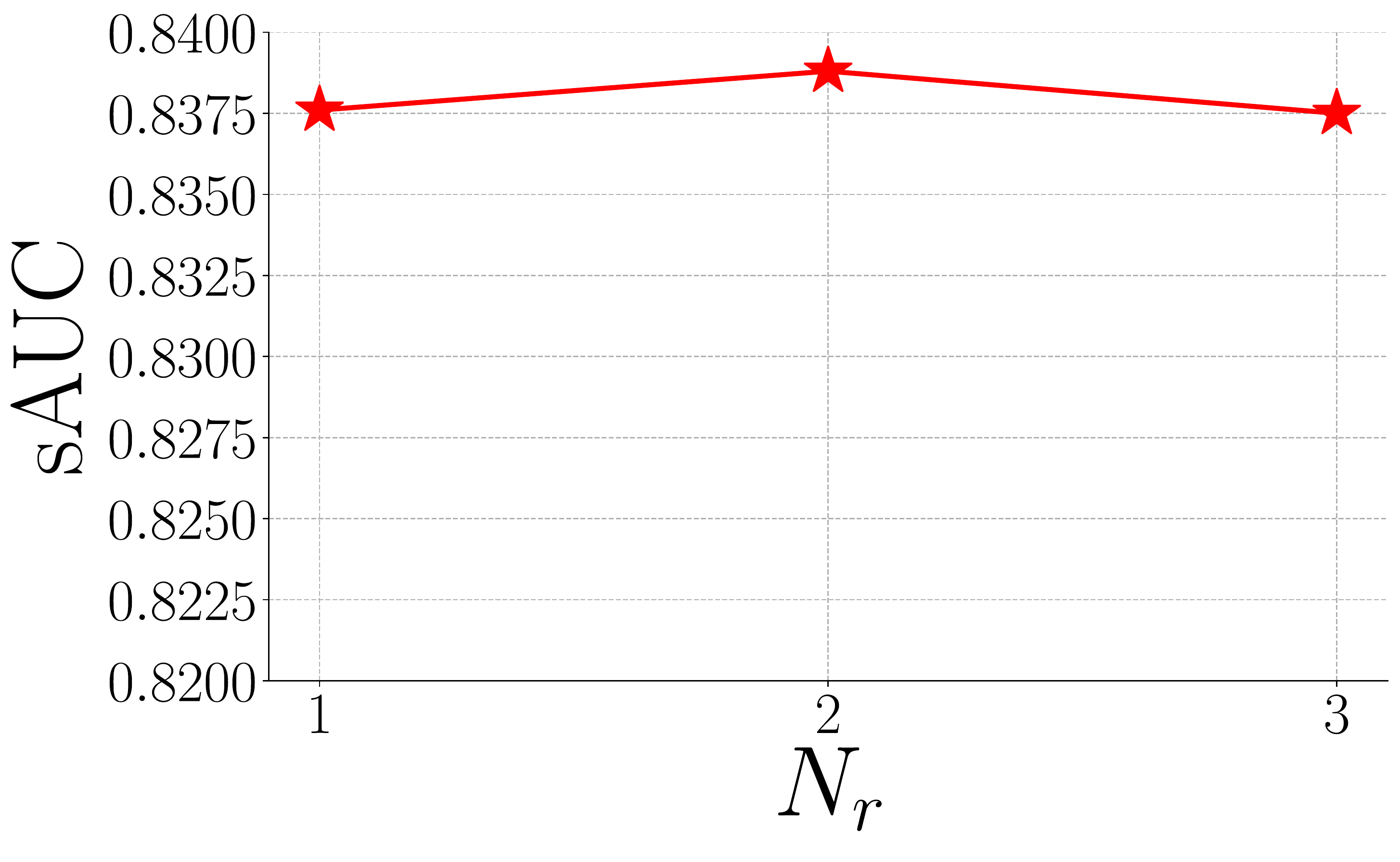} \label{fig:refnum_sal}} \\
	\subfloat[Top 1 error vs. $\beta_{w}$]{\includegraphics[width=0.48\linewidth]{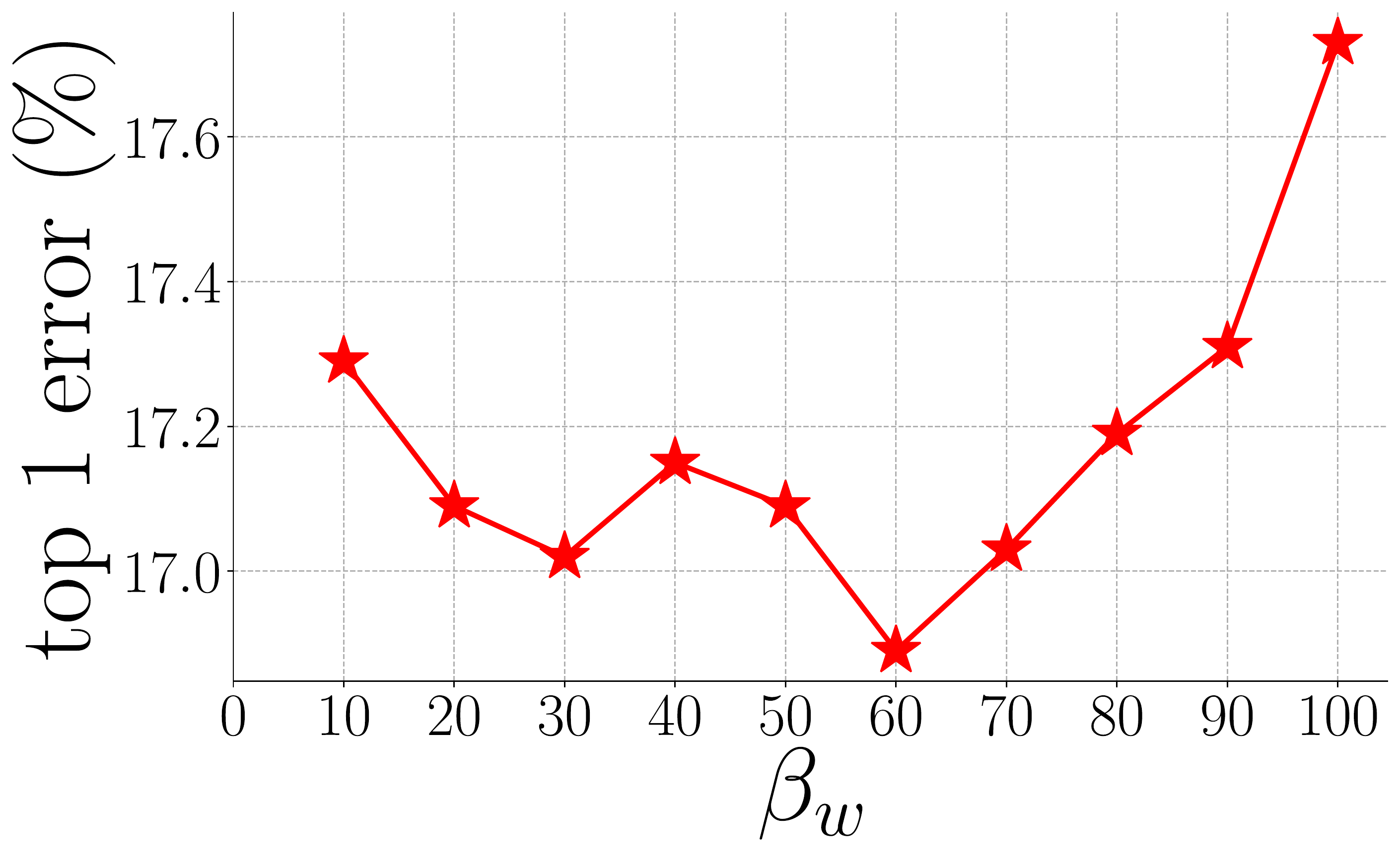}         \label{fig:beta}}       \hfill
	\subfloat[Top 1 error vs. $N_{r}$]    {\includegraphics[width=0.48\linewidth]{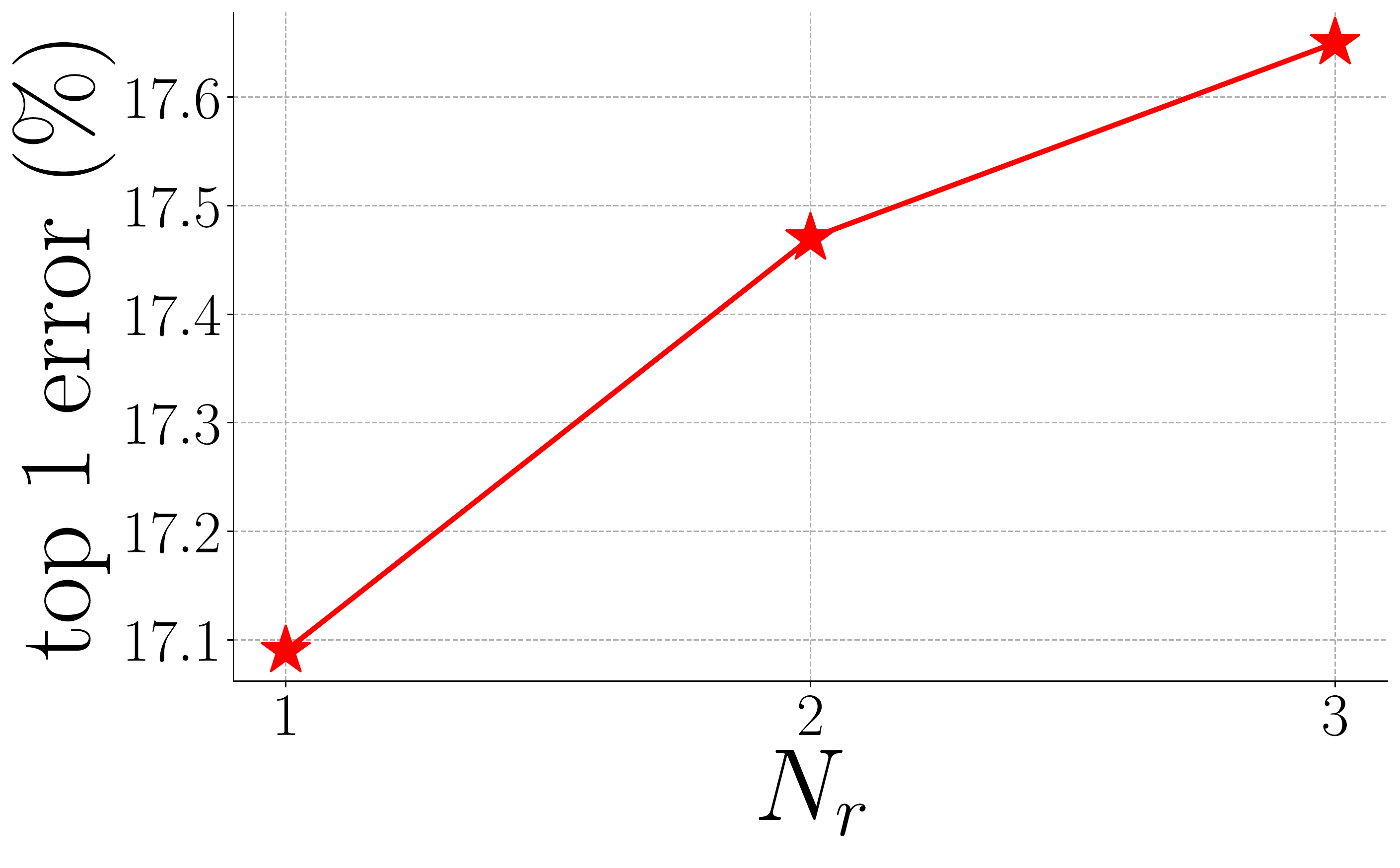}         \label{fig:refnum}} 
	\caption{Ablation study w.r.t. effective window size {\small $\beta_{w}$} and references number {\small $N_{r}$}. (a) and (b) are the experimental results on the SALICON validation set, while (c) and (d) are with the Tiny ImageNet validation set. {\small $\beta_{w}=\infty$} in (b) and {\small $\beta_{w}=50$} in (d).
	}
	\label{fig:ablation}
\end{figure}

\begin{figure*}[!t]
	\captionsetup{width=0.19\textwidth}
	\centering
	\subfloat[Stochastical permutations on OSIE.]{    \includegraphics[width=.23\textwidth]{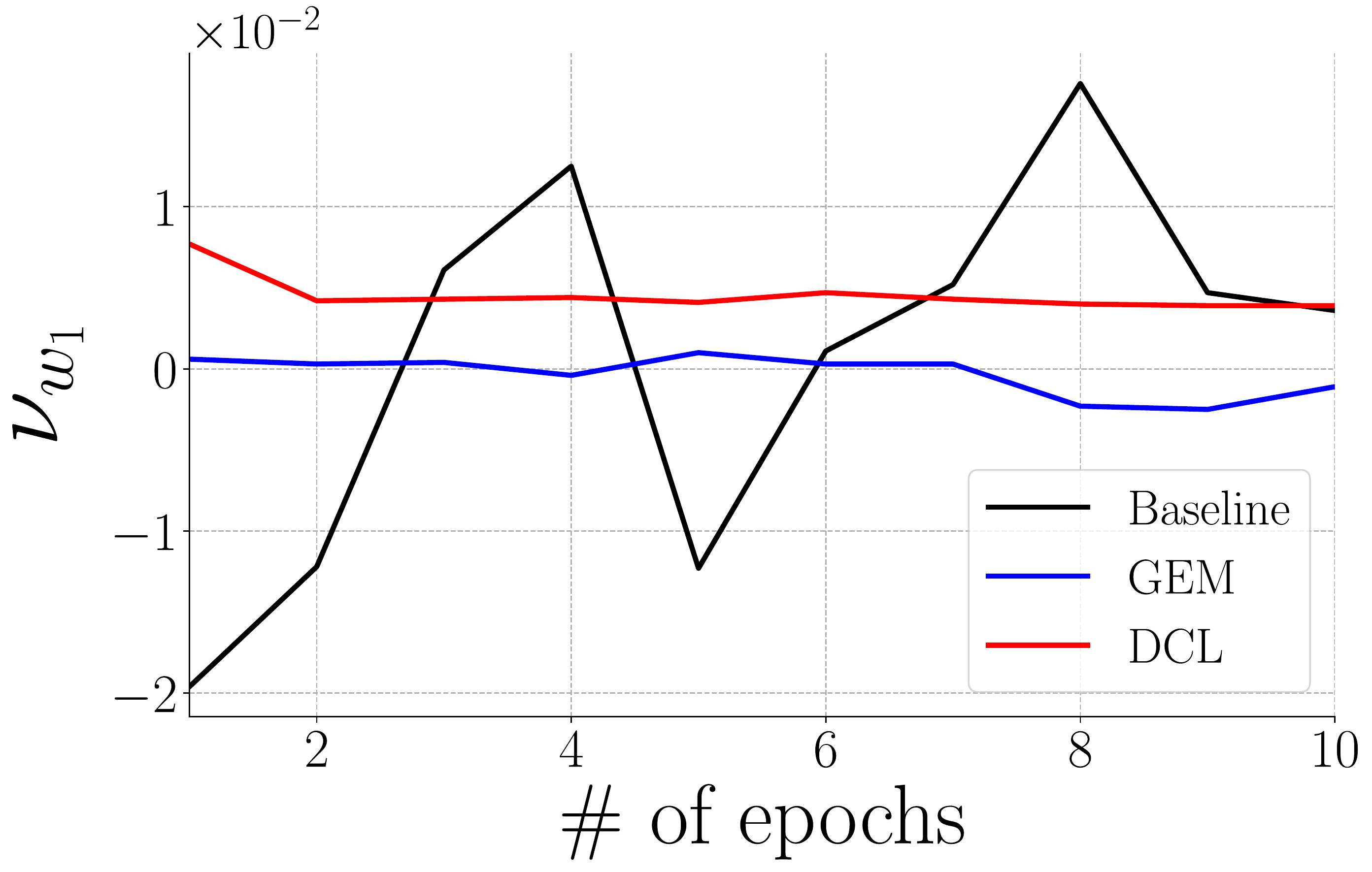}                \label{fig:cong_osie_stoch}      } \hfill
	\subfloat[Stochastical permutations on SALICON.]{ \includegraphics[width=.23\textwidth]{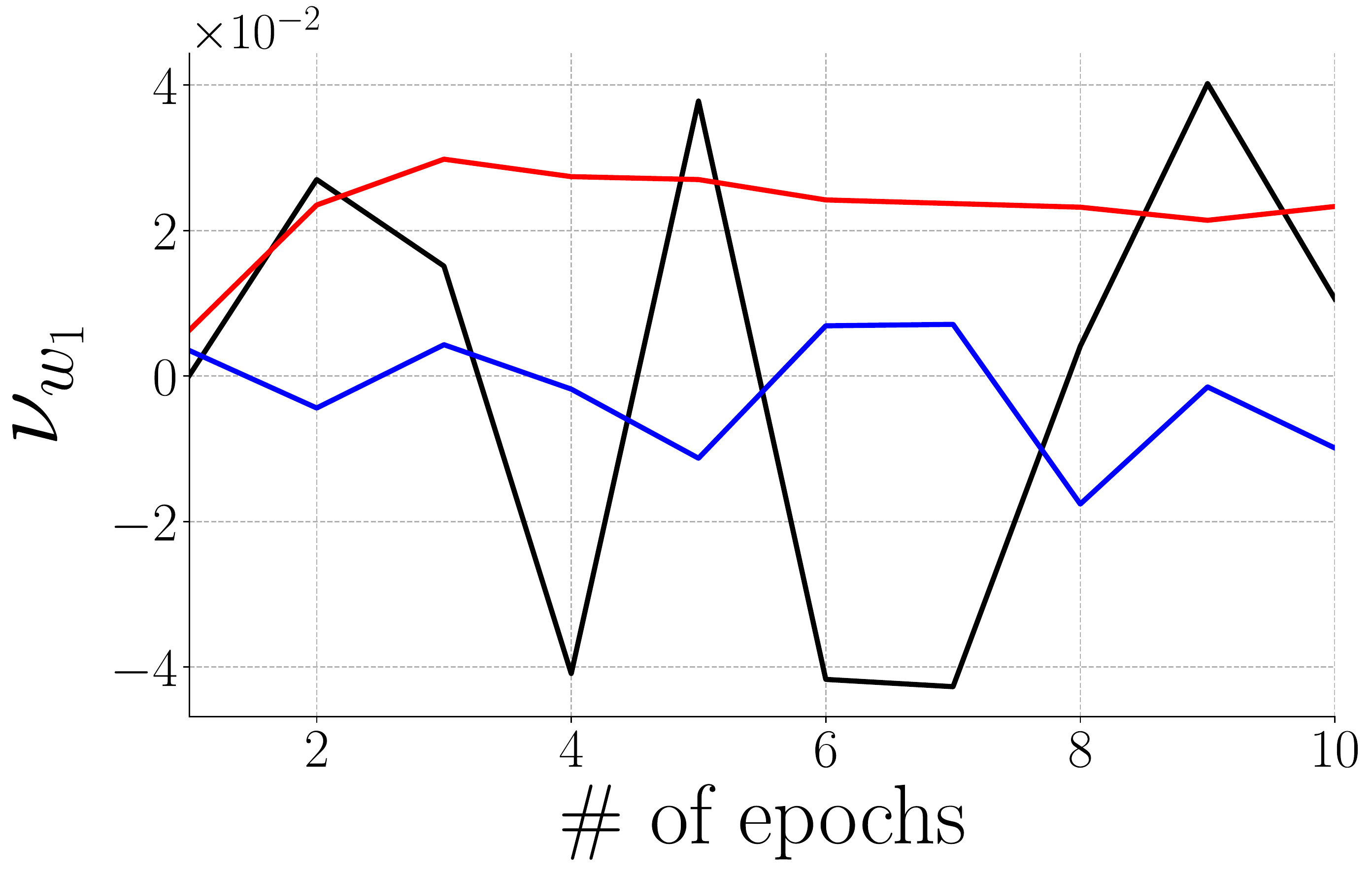}             \label{fig:cong_salicon_stoch}	  } \hfill
	\subfloat[Fixed permutations on OSIE.]{           \includegraphics[width=.23\textwidth]{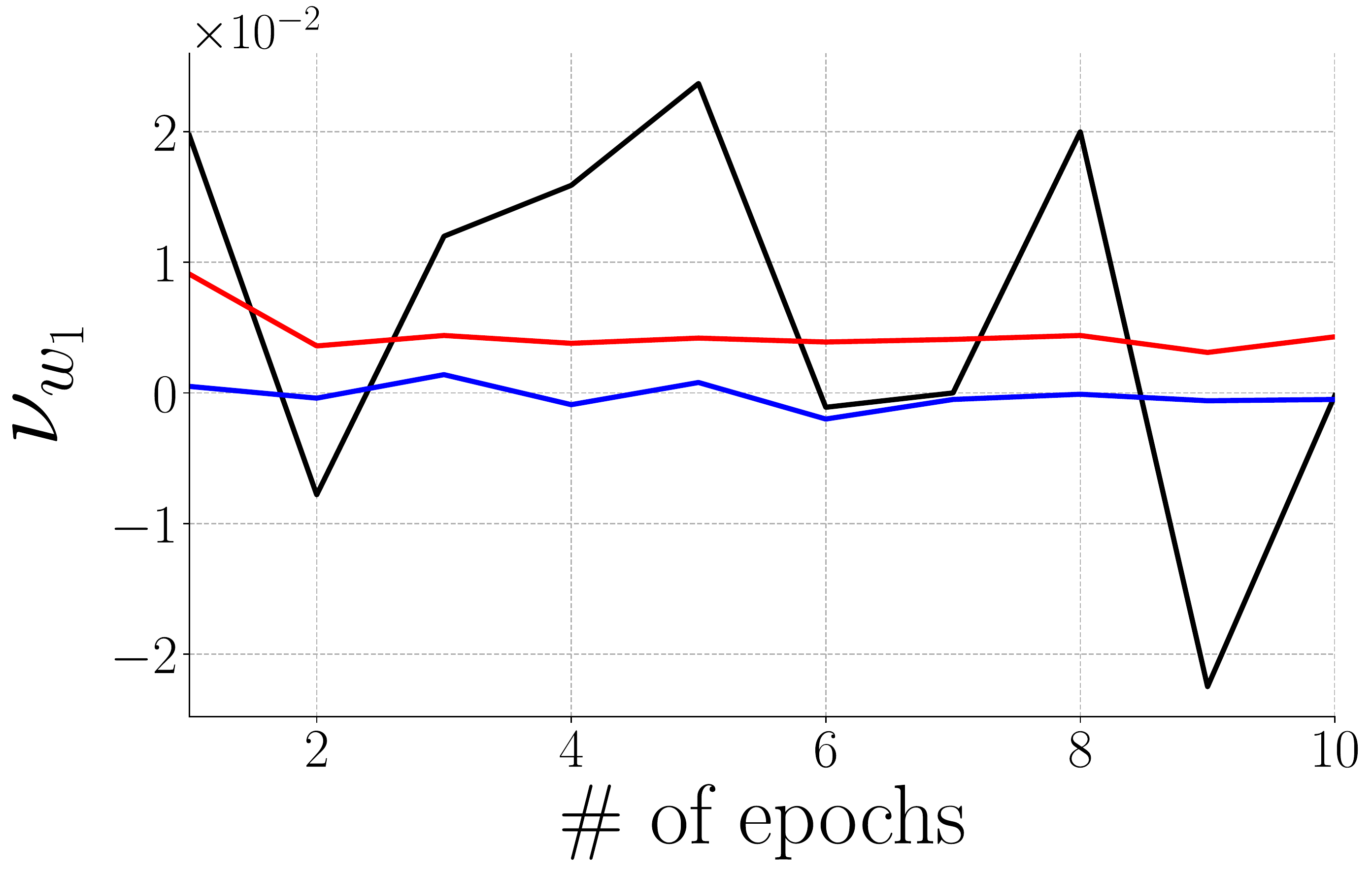}    \label{fig:cong_osie_nonstoch}   } \hfill
	\subfloat[Fixed permutations on SALICON.]{        \includegraphics[width=.23\textwidth]{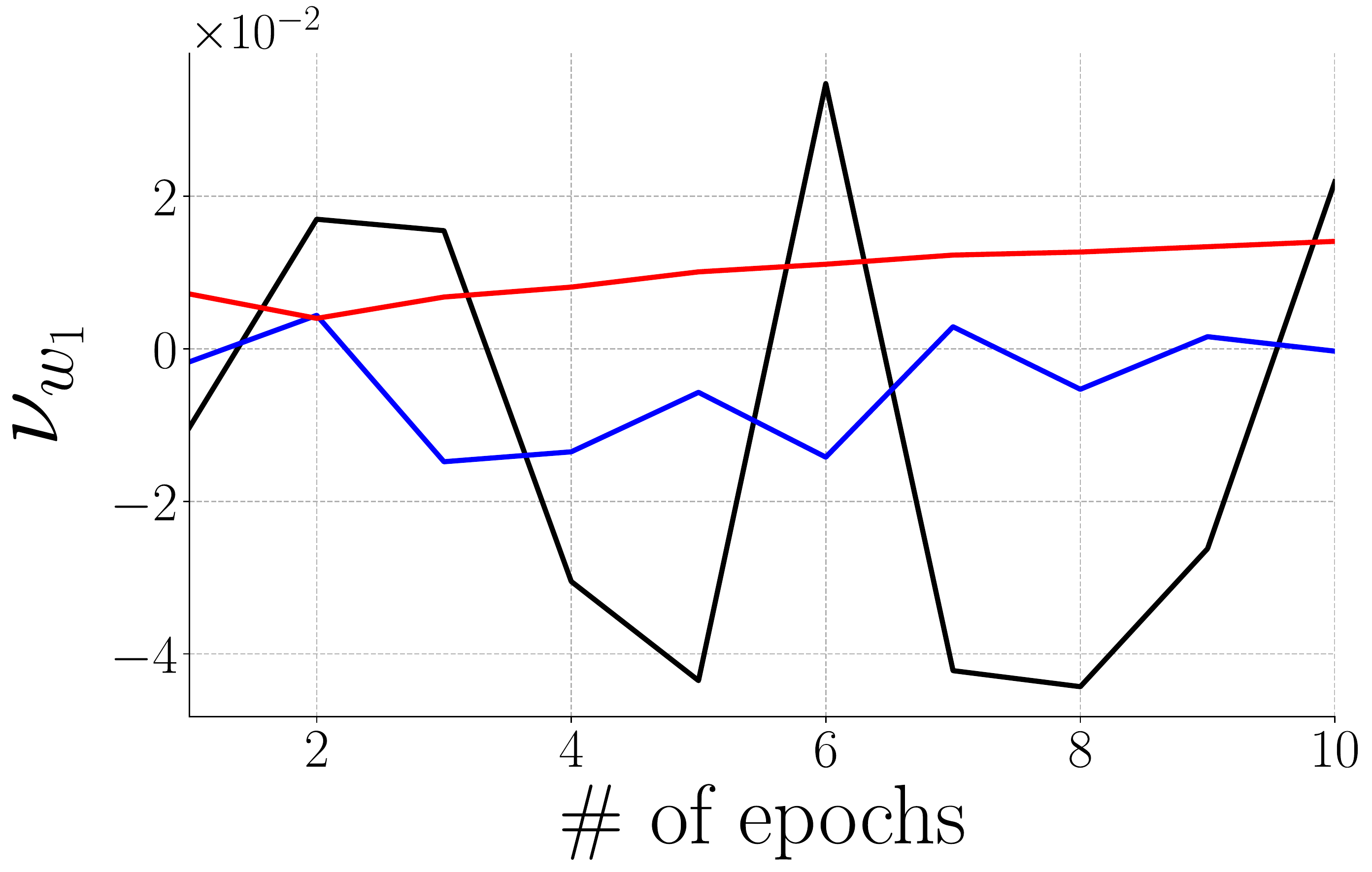} \label{fig:cong_salicon_nonstoch}}
	\captionsetup{width=1.0\textwidth}
	\caption{Congruencies along the epochs in saliency prediction learning, as defined in \eqn (\ref{eqn:cong}). The samples sequences for training models are determined by independent stochastic processes in \fig \ref{fig:cong_osie_stoch} and \ref{fig:cong_salicon_stoch}, while the permuted samples sequences are pre-determined and fixed for all models in \fig \ref{fig:cong_osie_nonstoch} and \ref{fig:cong_salicon_nonstoch}. The baseline, GEM, and DCL are ResNeXt-29 SGD, ResNeXt-29 SGD GEM, and ResNeXt-29 SGD DCL-$\infty$-1 (see \tab \ref{tbl:batch_cifar}), respectively.}
	\label{fig:cong_sal}
\end{figure*}

In the classification experiment, \fig \ref{fig:beta} shows the curve of top 1 error vs. $\beta_{w}$ based on DCL-$\beta_{w}$-1. We can see that only a small $\beta_{w}$ range, \ie between 20 and 70, yields relatively lower errors than the other $\beta_{w}$ values. On the other hand, \fig \ref{fig:refnum} shows that only using one reference is helpful in the learning process for classification. Different from the pattern shown in \fig \ref{fig:beta_sal} and \ref{fig:refnum_sal}, where the curves are relatively flat, the pattern in \fig \ref{fig:beta} and \ref{fig:refnum} implies that the gradients in the learning process for classification are dramatically changed in angle to satisfy the 200-way prediction. Hence, the learning process for classification does not prefer large {\small $\beta_{w}$} and {\small $N_{r}$}.

In summary, the nature of the task should be taken into account to determine the values of {\small $\beta_{w}$} and {\small $N_{r}$}. Both parameters can lead to significantly different performances if the task-specific semantics in the data are highly varying. Specifically, as {\small $N_{r}$} increases, the feasible region for searching a local minimum possibly becomes narrow as shown in \fig~\ref{fig:illu_grad}. If the local minimum is not in the narrowed feasible region, large {\small $N_{r}$} could lead to a slower convergence or even a divergence.

\subsection{Congruency Analysis}
\label{subsec:cong_analysis}

In this section, we focus on analyzing the patterns of congruency on saliency prediction, continual learning, and classification. 
For saliency prediction and classification, to study how the gradients of GEM and the proposed DCL method vary in the training process, we compute the congruency of each epoch in the training process by \eqn~(\ref{eqn:cong}). Specifically, it turns to be $\nu_{w_{es}\rightarrow w_{ee}|w_{1}}$, 
where $w_{es}$ and $w_{ee}$ is the weights at the first and last iteration of each epoch, respectively.
Here, $w_{0}$ is randomly initialized and $w_{1}$ represents the starting point of the training. 
For convenience, we simplify the notation of the average congruency $\nu_{w_{es}\rightarrow w_{ee}|w_{1}}$ for each epoch as $\nu_{w_{1}}$. Correspondingly, we define the average magnitude $d_{w_{1}}$ of the accumulated gradients over the iterations in an epoch, \ie
\begin{align}
\resizebox{0.89\linewidth}{!}{$
d_{w_{1}} =  \frac{1}{sub(w_{ee})-sub(w_{es})+1} \sum\limits_{i=sub(w_{es})}^{sub(w_{ee})}\|w_{i}-w_{1}\|_{2}
$}
\label{eqn:mag}
\end{align}
where $d_{w_{1}}$ indicates the measurement of magnitudes of the accumulated gradients takes $w_{1}$ as the reference.
\REVISION{
Note that \eqn (\ref{eqn:mag}) does work not only with an absolute reference (\eg $w_{1}$), but can work with a relative reference (\eg $w_{i-1}$) as well. Specifically, we can substitute $w_{i-1}$ for $w_{1}$ in \eqn (\ref{eqn:mag}) to compute $d_{w_{i-1}}$.
\eqn (\ref{eqn:mag}) can allow us to peek into the convergence process in the high dimensional weight space, where it is difficult to visualize the convergence. By taking an absolute reference (\eg $w_1$) as the reference, it is able to provide an overview about how the learning process converges to the local minimum from the fixed reference, while a relative reference (\eg $w_{i-1}$) is helpful to reveal the iterative pattern.
}

For the experiments of continual learning, since GEM uses the samples in memory to regulate the optimization direction, we follow this setting to check the effect of the proposed DCL method on the cosine similarities between the corrected gradient and the gradients generated by the samples in memory for analysis. More concretely, the average cosine similarity is defined as {\small $\frac{1}{\|N_{it}\|} \frac{1}{\|\mathcal{M}\|} \sum_{i=1}^{N_{it}}\sum_{s\in \mathcal{M}}^{} \cos(g_{s}, g_{GEMi})$}, where $N_{it}$ is the number of iterations in an epoch and $g_{GEMi}$ is the gradient of GEM at $i$-th iteration.

\subsubsection{Saliency Prediction} 
We analyze the models from \tab~\ref{tbl:osie}, \ie ResNet-50 Adam (baseline), ResNet-50 Adam GEM (GEM), and ResNet-50 Adam DCL-$\infty$-1 (DCL). As the training samples sequence is affected by the stochastic process and it may be a factor influencing the proposed DCL method, we present two settings, \ie within the independent stochastic process and within the same stochastic process amid the training of the three models, to gauge the influence of the stochastic process on the proposed DCL method. Specifically, \fig \ref{fig:cong_osie_stoch} and \ref{fig:cong_salicon_stoch} are the curves with the independent stochastic process on OSIE and SALICON, respectively, whereas the same permuted samples sequences are used in the trainings of the three models in \fig \ref{fig:cong_osie_nonstoch} and \ref{fig:cong_salicon_nonstoch}. We can see that they are similar in pattern and it implies that the permutation of the training samples has less influence on the proposed DCL method. Moreover, the proposed DCL method consistently gives rise to a more congruent learning process than the baseline and GEM.

\begin{figure}[!t]
	\centering
	\subfloat[MNIST-R]{ \label{fig:cong_mnistr}	\includegraphics[width=0.48\columnwidth]{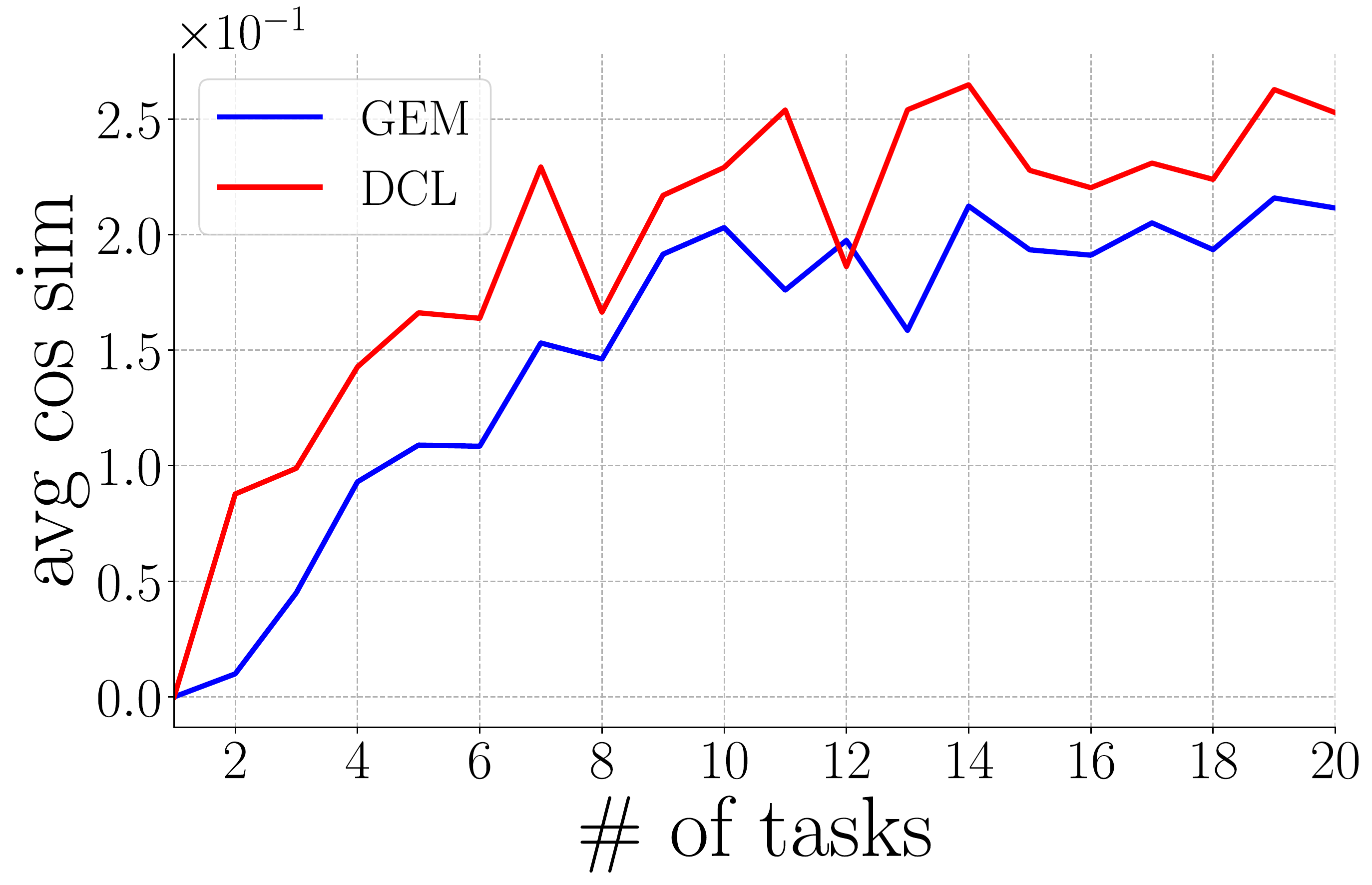}}
	\subfloat[MNIST-P]{ \label{fig:cong_mnistp}	\includegraphics[width=0.48\columnwidth]{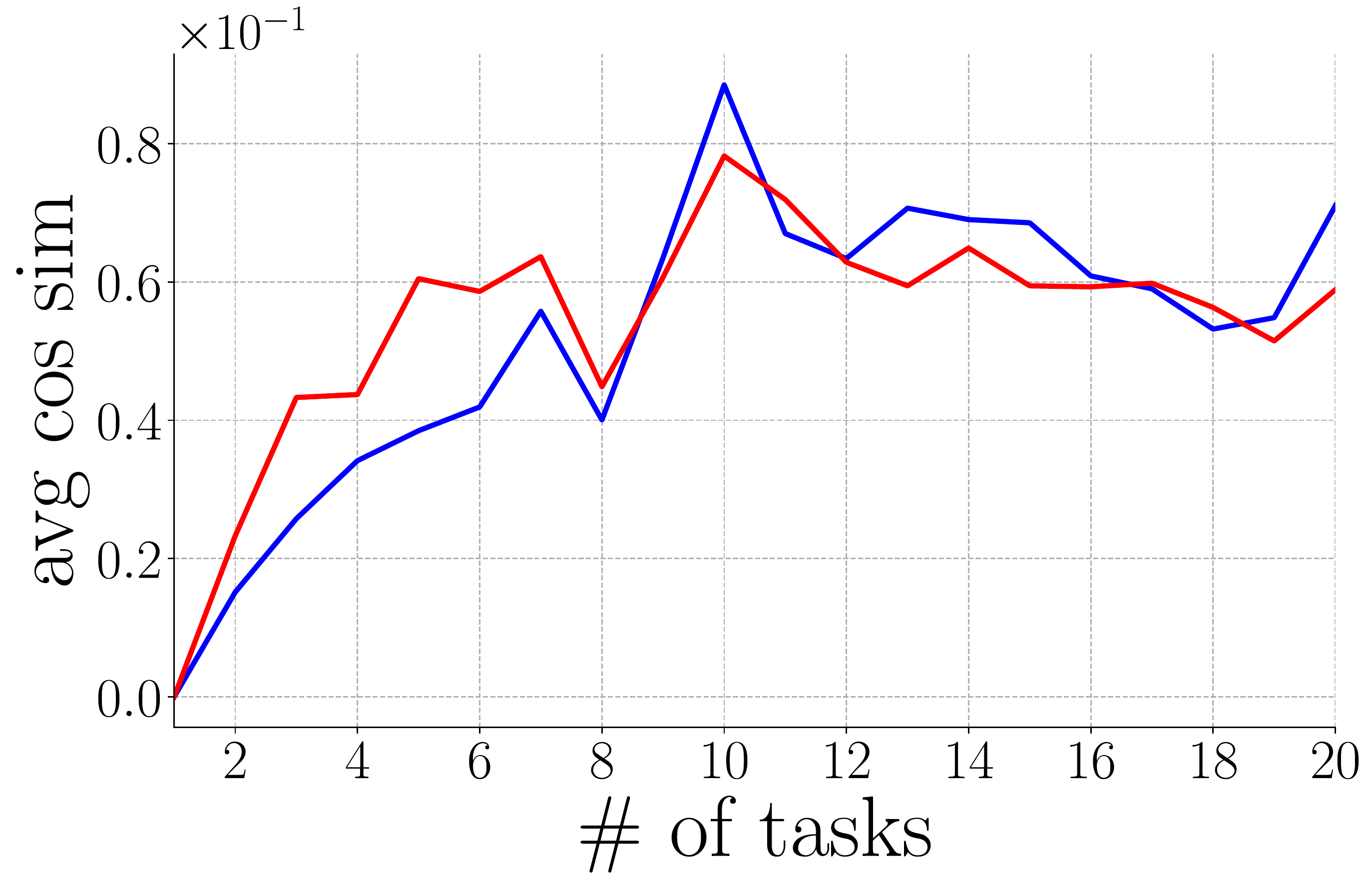}} \\
	\subfloat[iCIFAR-100]{ \label{fig:cong_icifar100}	\includegraphics[width=0.48\columnwidth]{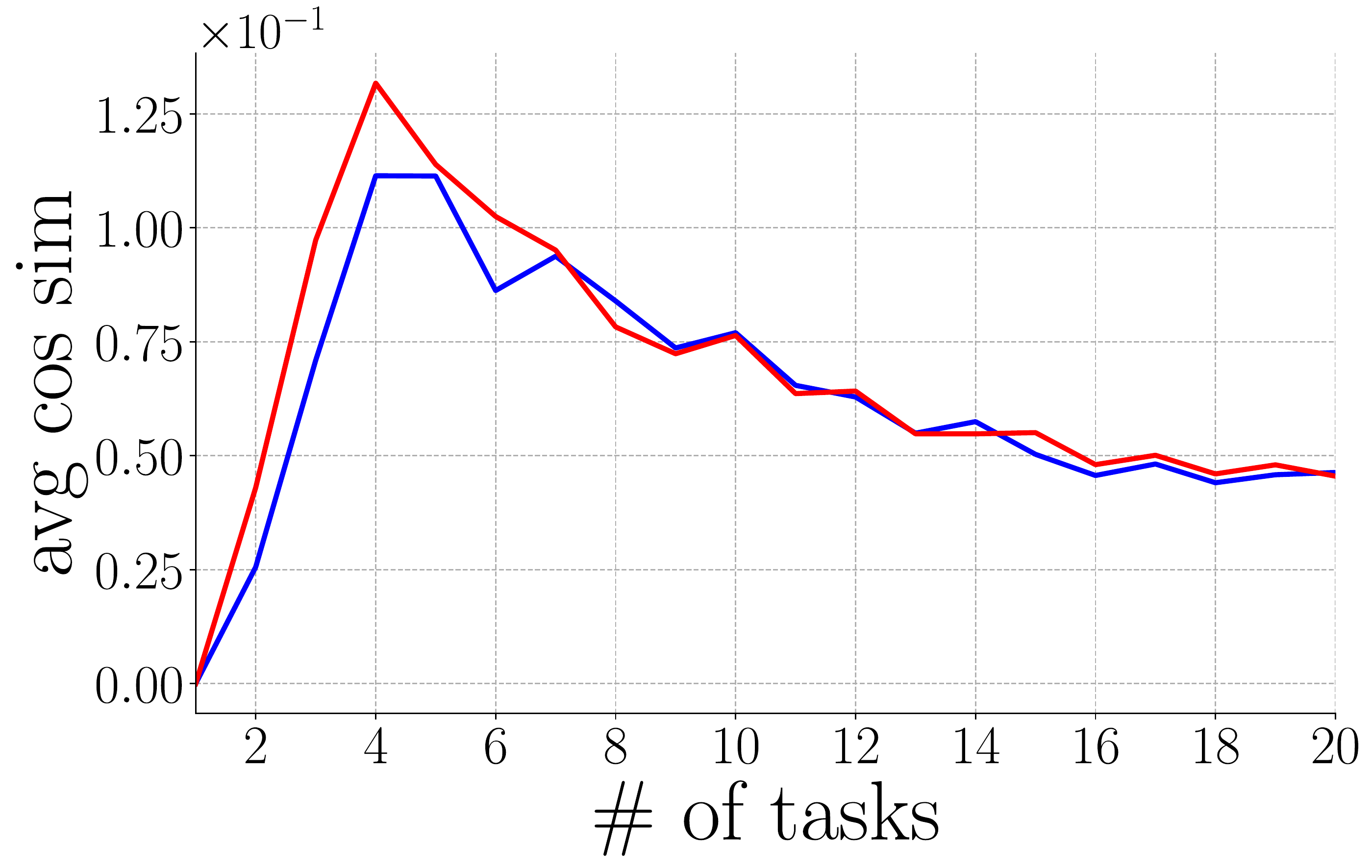}}
	\vspace{-2ex}
	\caption{The average congruencies over epochs in training on the three datasets for continual learning.}
	\label{fig:cong_continual}
\end{figure}

\subsubsection{Continual Learning}
\fig \ref{fig:cong_mnistr},\ref{fig:cong_mnistp}, and \ref{fig:cong_icifar100} shows the congruency along the tasks which are the episodes to learn the new classes. It can be seen that the proposed DCL method significantly enhances the cosine similarities between the gradients for updates and the gradients generated by the samples in memory on MNIST-R. There are improvements made by the proposed DCL method on early tasks on MNIST-P. Moreover, an overall consistent improvement of the proposed DCL method can be observed on iCIFAR-100.
Overall, the corrected updates for the model are computed by proposed DCL method to be more congruent with its previous updates. This consistently results in the improvement of BWT in \tab \ref{tbl:conn_rota}, \ref{tbl:conn_perm}, and \ref{tbl:conn_cifar}.


\begin{figure}[!t]
	\captionsetup{width=0.32\columnwidth}
	\centering
	\subfloat[Avg. cong. on CIFAR-10.]{ \label{fig:cong_cifar10}	\includegraphics[width=0.32\columnwidth]{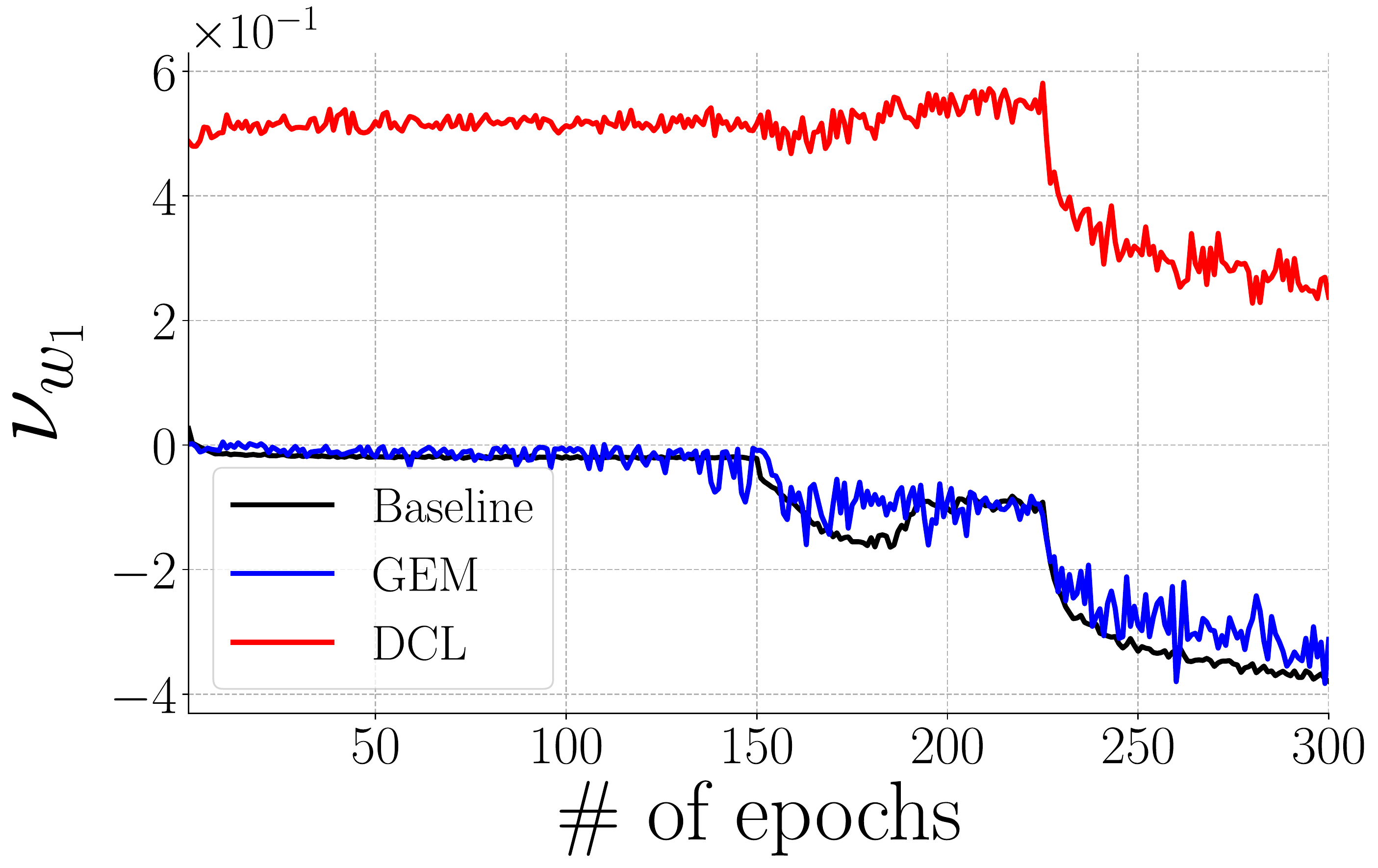}	}
	\subfloat[$d_{w_{1}}$ on CIFAR-10.]{ \label{fig:mag_cifar10} 	\includegraphics[width=0.32\columnwidth]{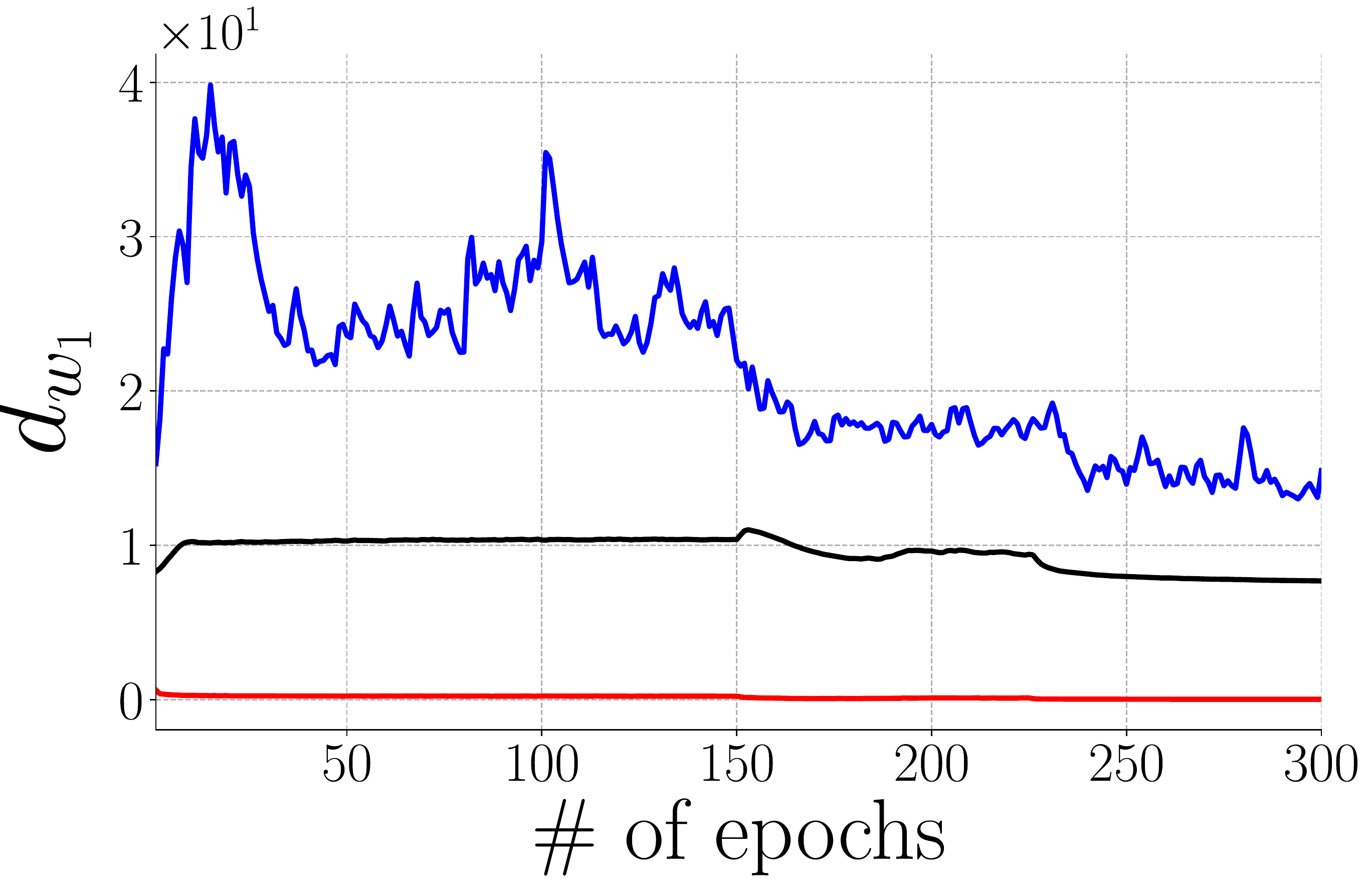}	} 
	\subfloat[$d_{w_{i-1}}$ on CIFAR-10.]{ \label{fig:mag_cifar10_rel} 	\includegraphics[width=0.32\columnwidth]{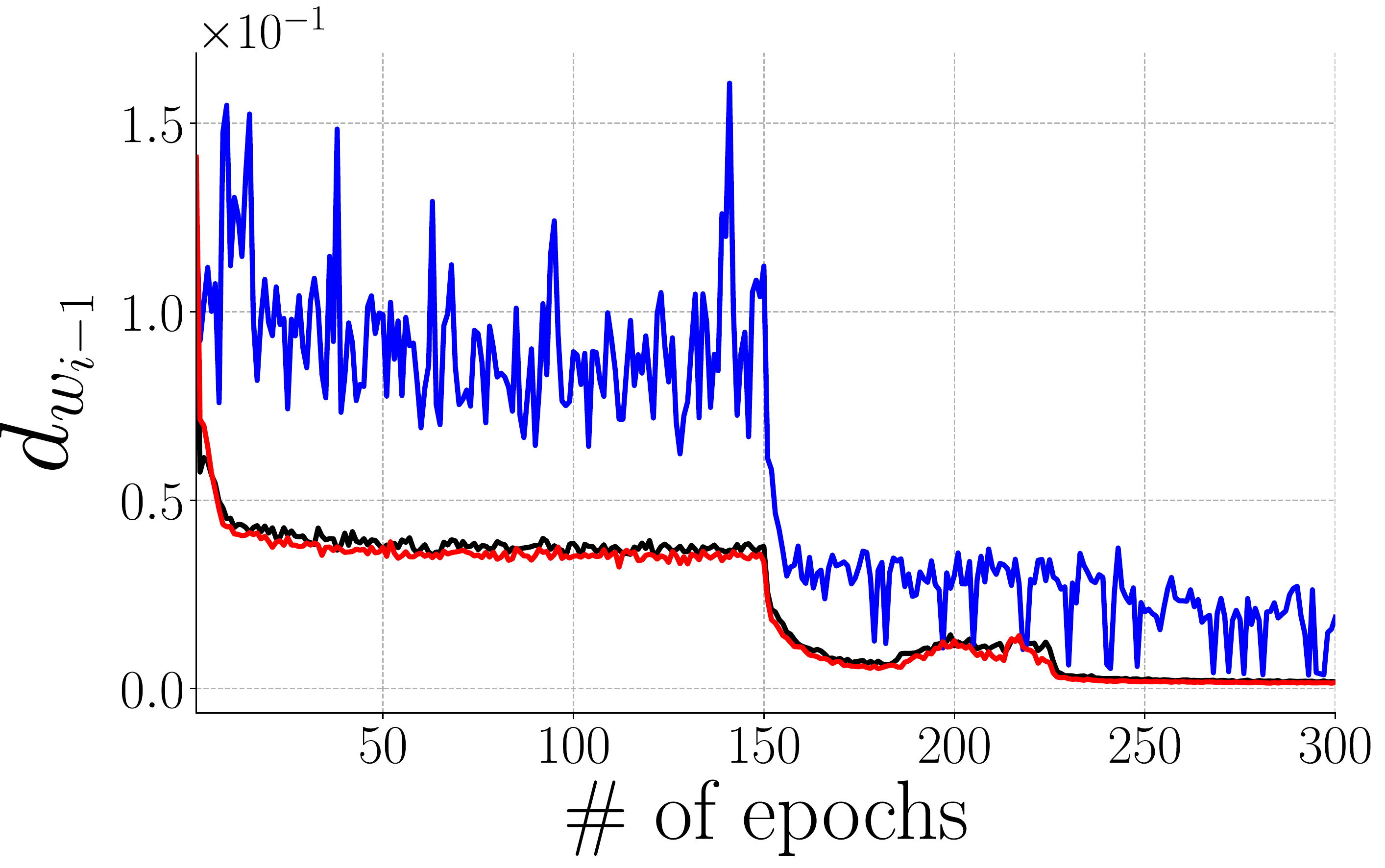}	} 
	\\
	\subfloat[Avg. cong. on CIFAR-100.]{ \label{fig:cong_cifar100}	\includegraphics[width=0.32\columnwidth]{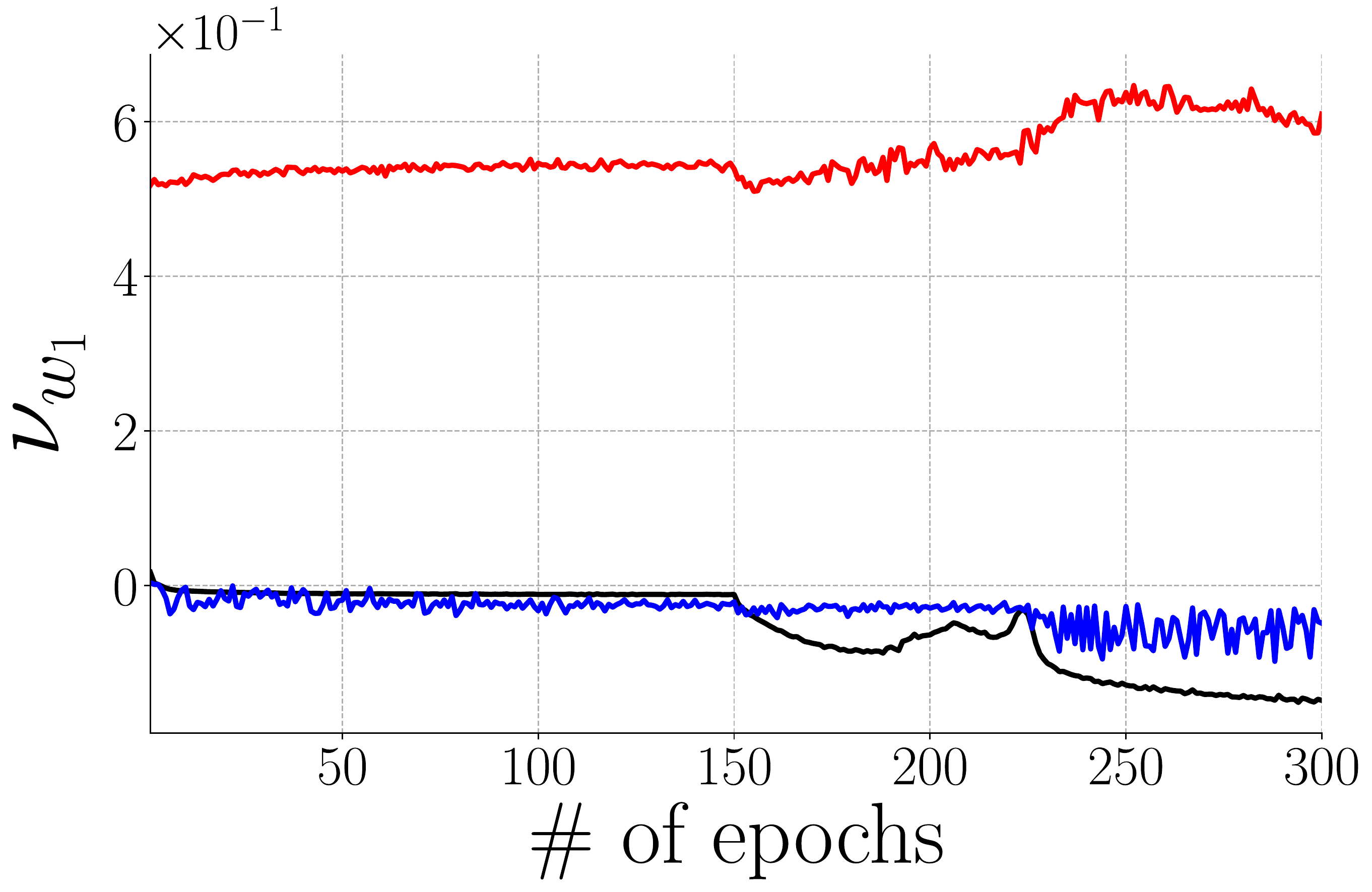}	} 
	\subfloat[$d_{w_{1}}$ on CIFAR-100.]{ \label{fig:mag_cifar100}	\includegraphics[width=0.32\columnwidth]{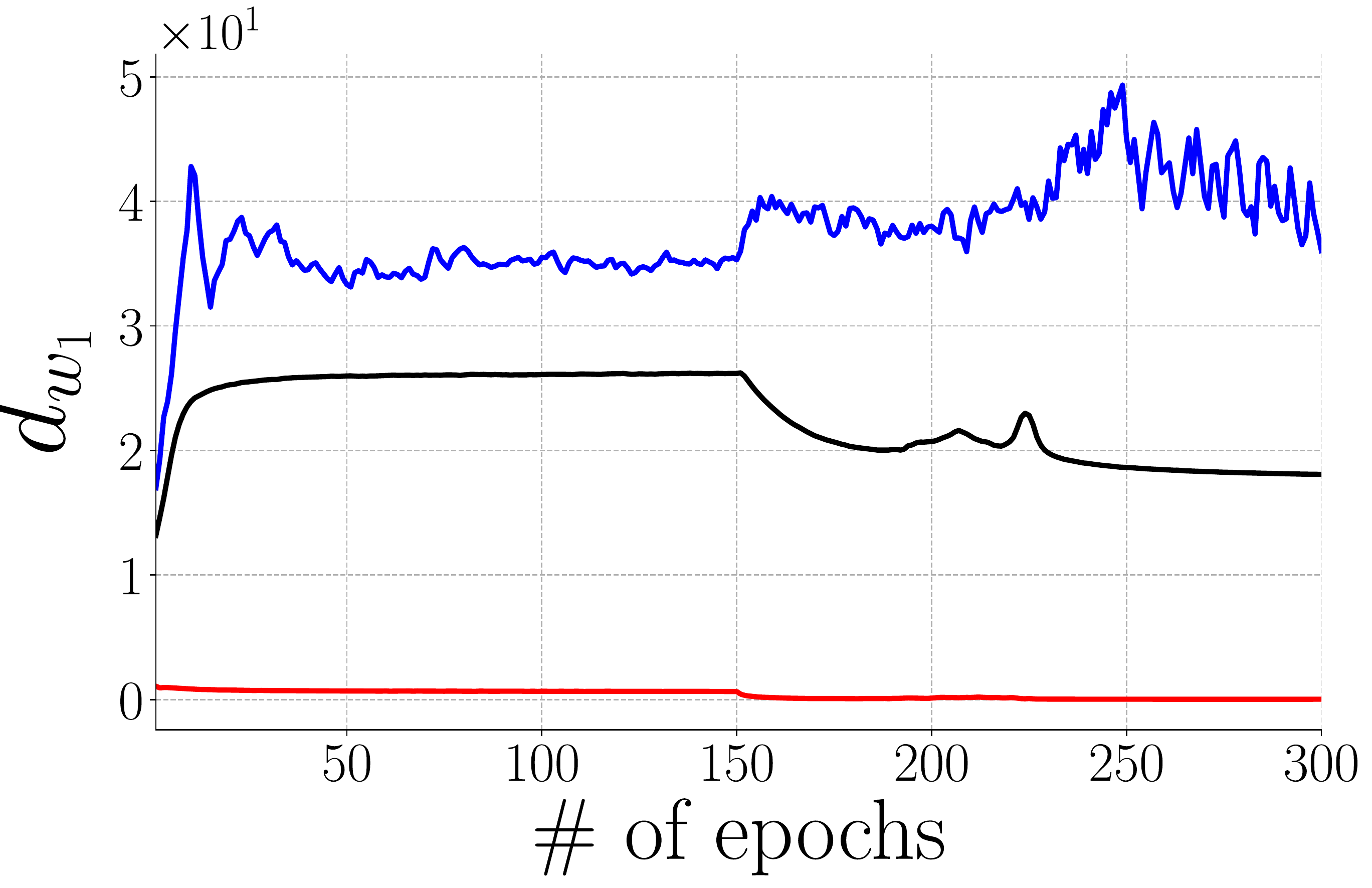}	} 
	\subfloat[$d_{w_{i-1}}$ on CIFAR-100.]{ \label{fig:mag_cifar100_rel} 	\includegraphics[width=0.32\columnwidth]{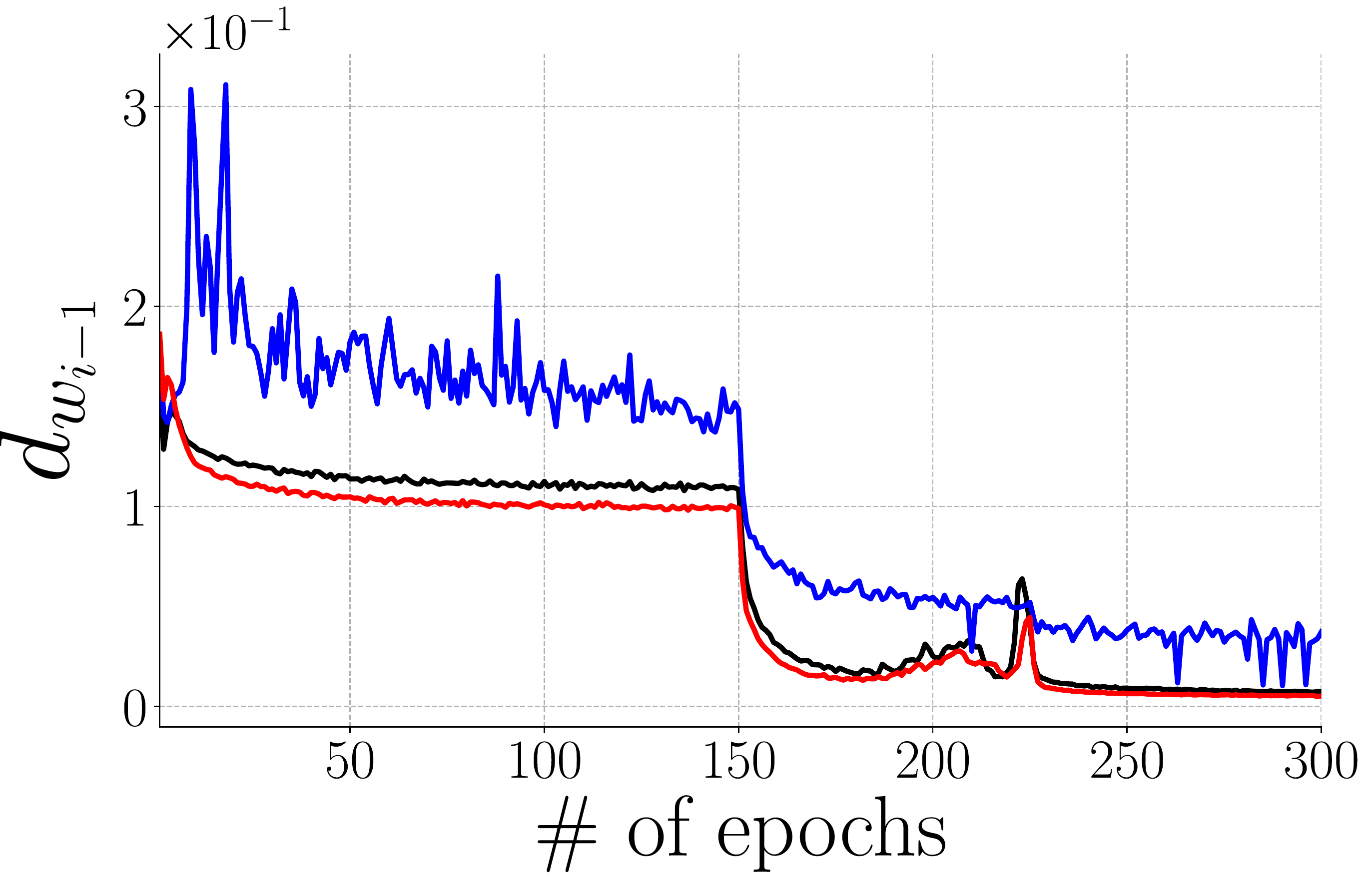}	} 
	\\
	\subfloat[Avg. cong. on Tiny ImageNet.]{ \label{fig:cong_timgnet}	\includegraphics[width=0.32\columnwidth]{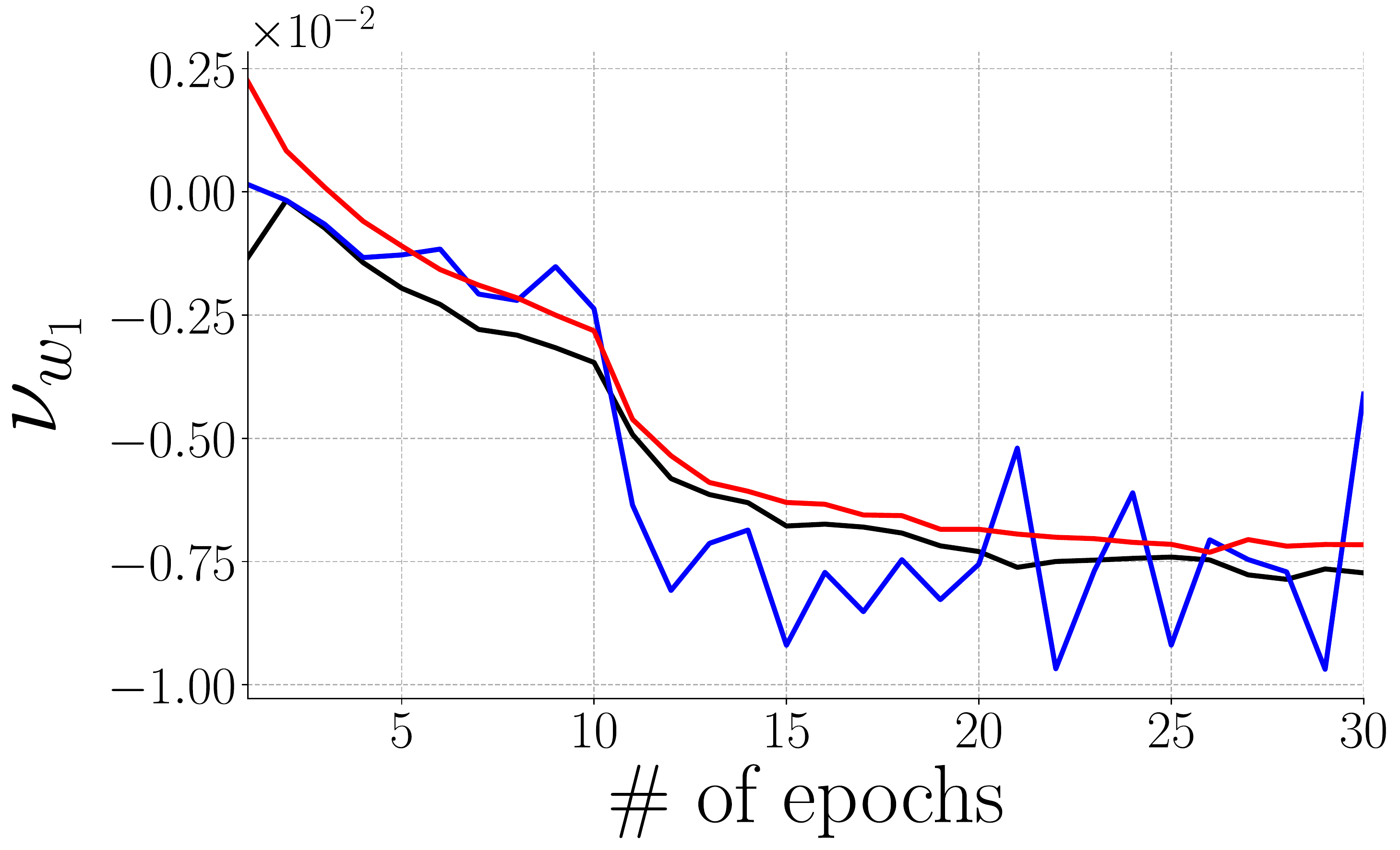}	}
	\subfloat[$d_{w_{1}}$ on Tiny ImageNet.]{ \label{fig:mag_timg}	\includegraphics[width=0.32\columnwidth]{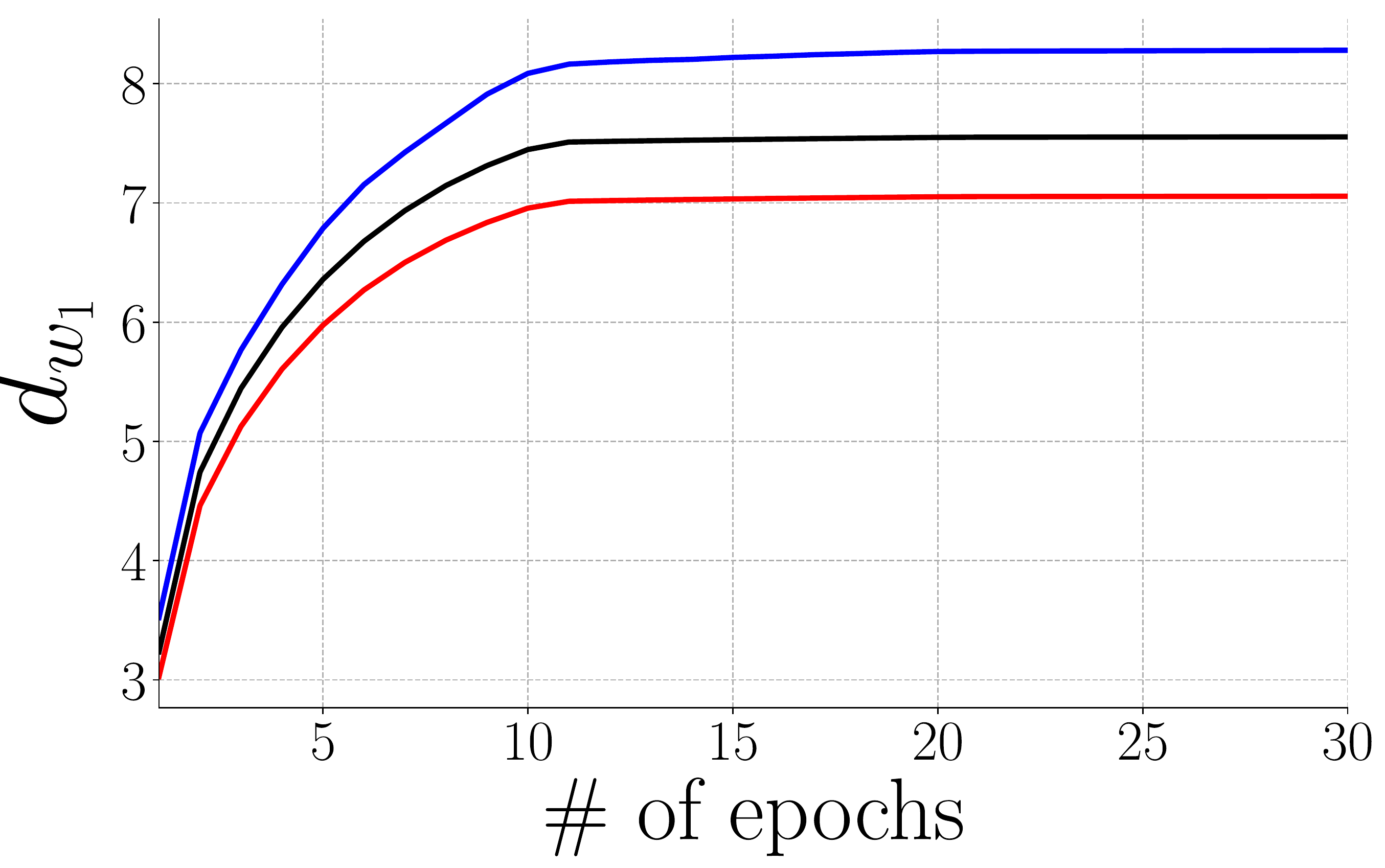}	}
	\subfloat[$d_{w_{i-1}}$ on Tiny ImageNet.]{ \label{fig:mag_timgnet_rel} 	\includegraphics[width=0.32\columnwidth]{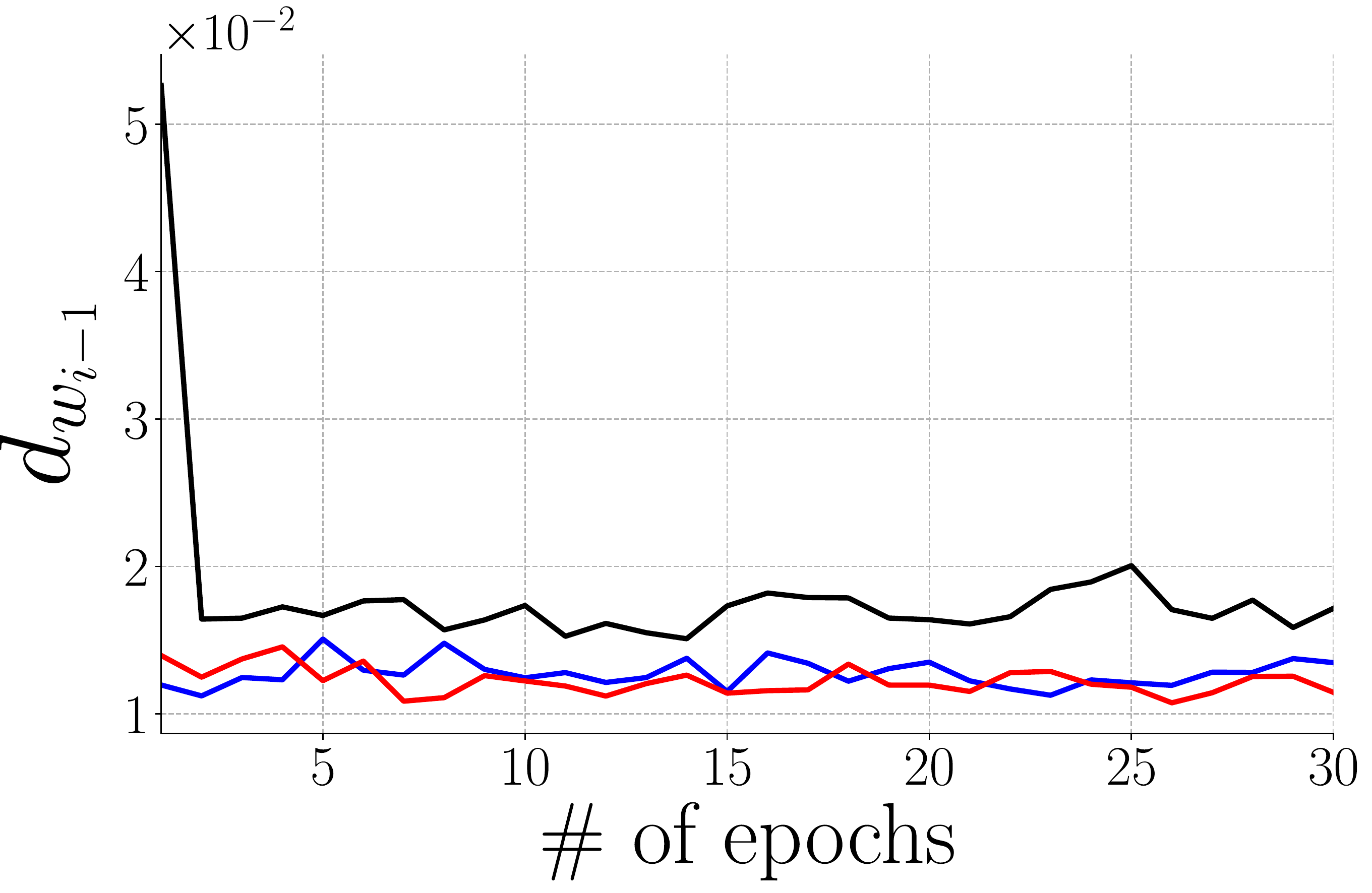}	} 
	\captionsetup{width=1.0\columnwidth}
	\caption{Analyses of the congruencies and magnitudes along the epochs in classification task, as defined in \eqn (\ref{eqn:cong}) and (\ref{eqn:mag}). 
	}
	\label{fig:cong_cls}
\end{figure}

\subsubsection{Classification}
\label{subsec:cls}
We analyze the models from \tab~\ref{tbl:batch_cifar}, \ie ResNeXt-29 SGD (baseline), ResNeXt-29 SGD GEM (GEM), and ResNeXt-29 SGD DCL-$\infty$-1 (DCL), in term of the resulting congruency of each epoch in the learning process on CIFAR. Similarly, ResNet-101 SGD, ResNet-101 SGD GEM, and ResNet-101 SGD DCL-50-1 in \tab \ref{tbl:tiny_imgnet} are used for analysis on Tiny ImageNet. The curves of the average congruencies are shown in \fig \ref{fig:cong_cifar10}, \ref{fig:cong_cifar100}, and \ref{fig:cong_timgnet}, while \fig \ref{fig:mag_cifar10}, \ref{fig:mag_cifar100}, and \ref{fig:mag_timg} show the average magnitudes.

As shown in \fig \ref{fig:cong_cifar10}, \ref{fig:cong_cifar100}, and \ref{fig:cong_timgnet}, the congruency of the proposed DCL method is significantly higher than the baseline and GEM along all epochs on CIFAR-10 and CIFAR-100. 
Higher congruency indicates the convergence path would be flatter and smoother.
For example, if all the congruencies of each epoch are 0, the convergence path would be a straight line. 

\begin{figure}[!t]
	\centering
	\subfloat[Mean of errors over epochs]{\includegraphics[width=0.483\linewidth]{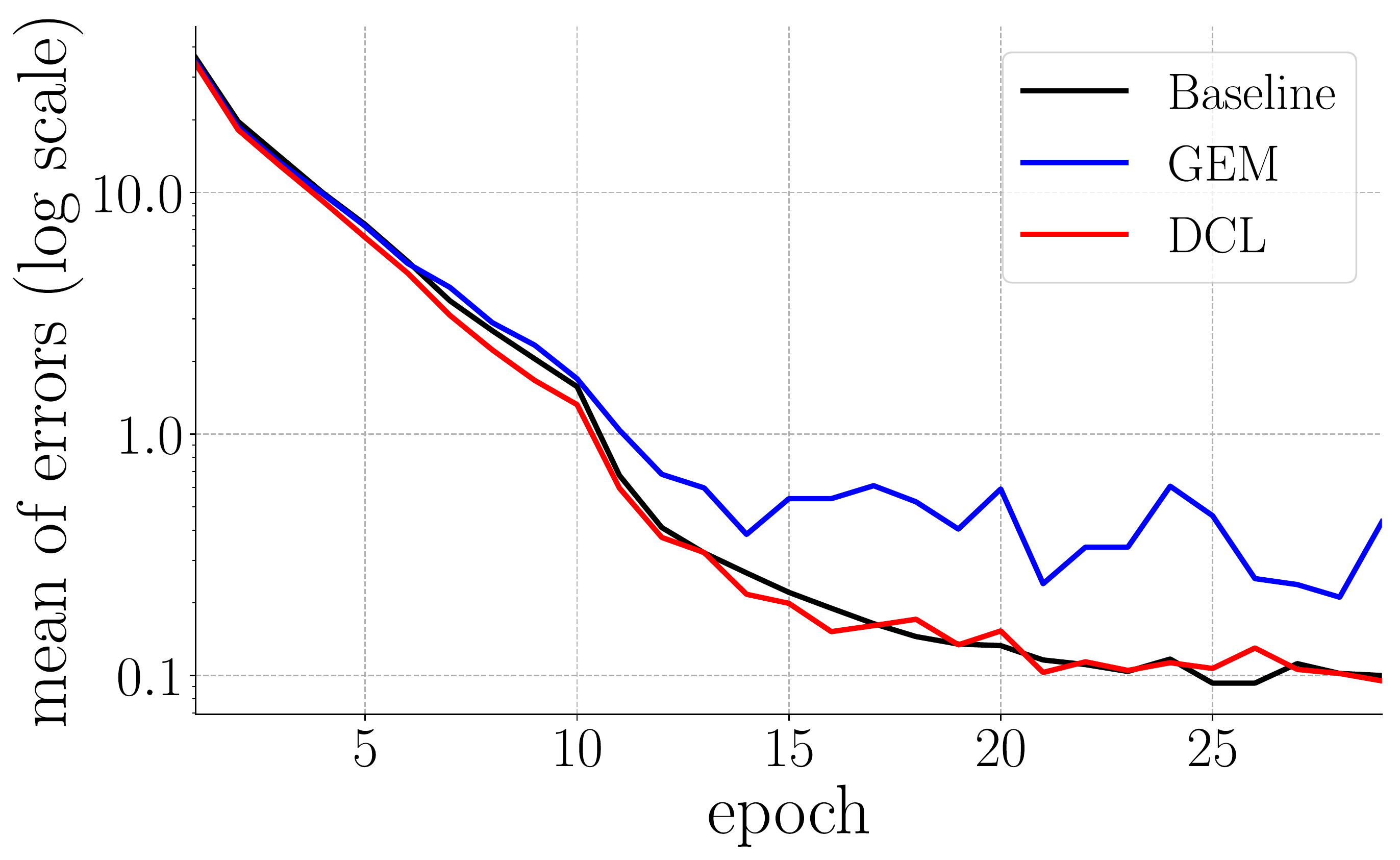} \label{fig:err_step_mean}} \hfill
	\subfloat[Std of errors over epochs] {\includegraphics[width=0.483\linewidth]{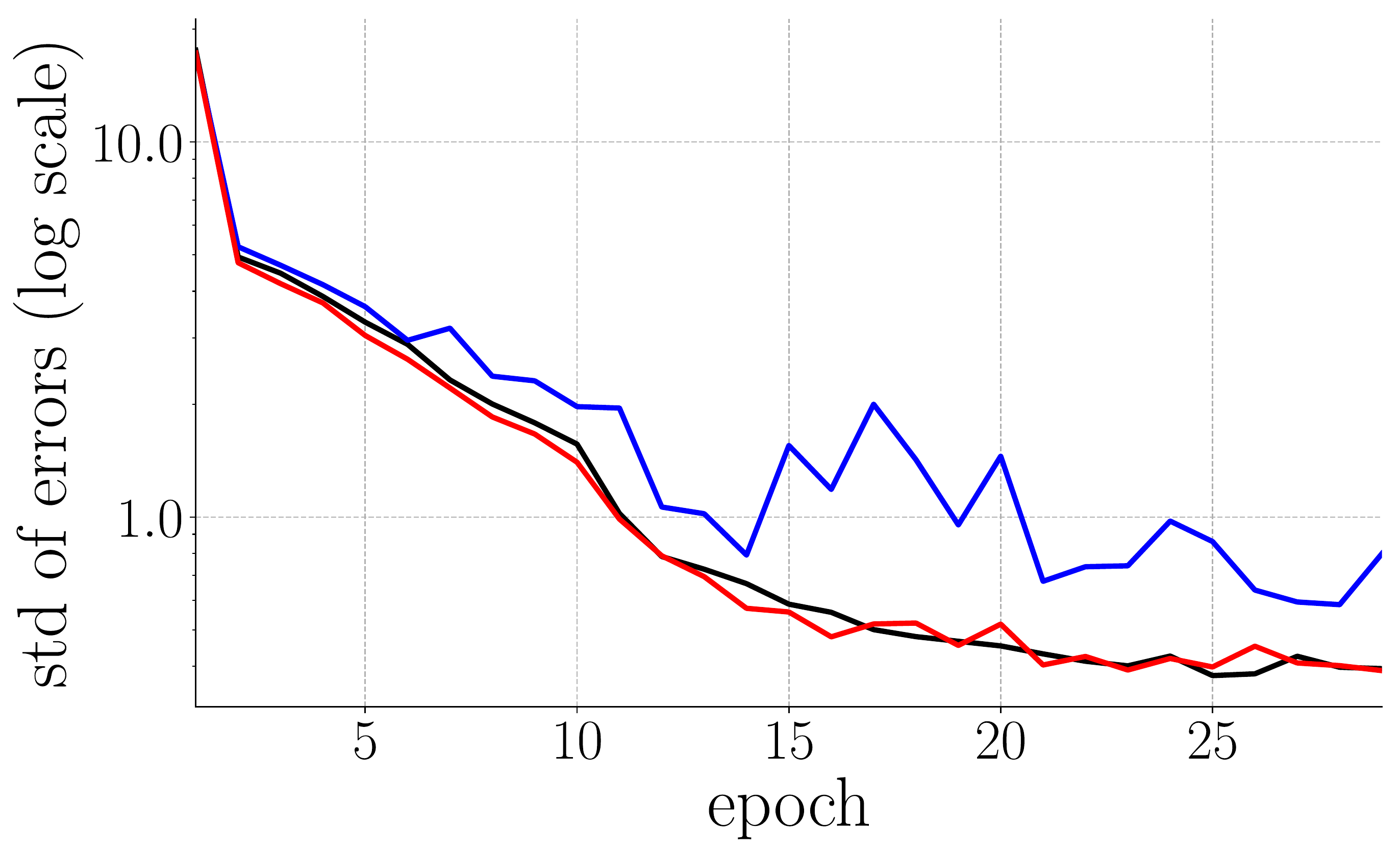} \label{fig:err_step_std}} \\
	\subfloat[Training error vs. iter. at epoch 1] {\includegraphics[width=0.465\linewidth]{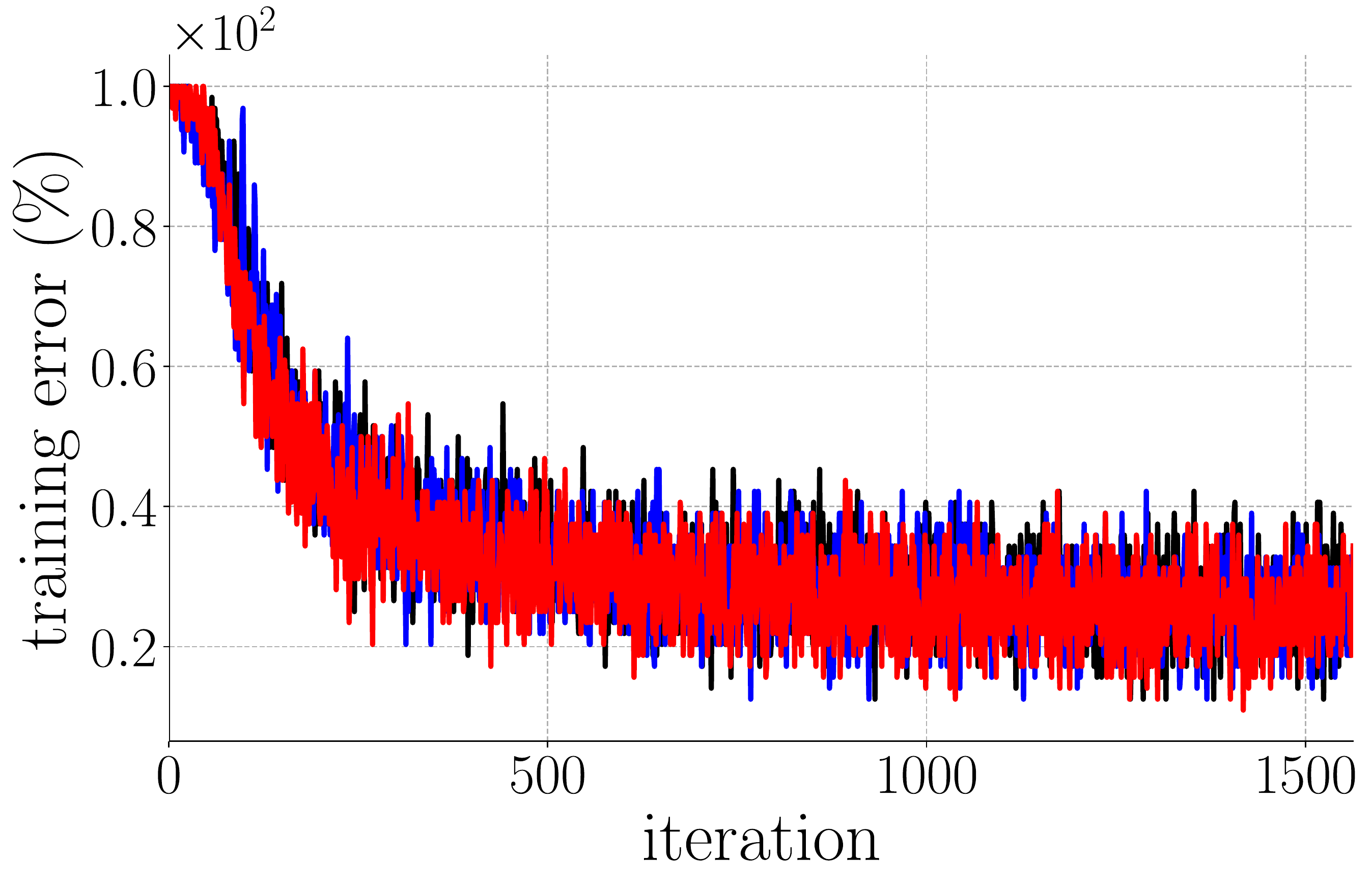} \label{fig:err_step_1}} \hfill
	\subfloat[Training error vs. iter. at epoch 5] {\includegraphics[width=0.465\linewidth]{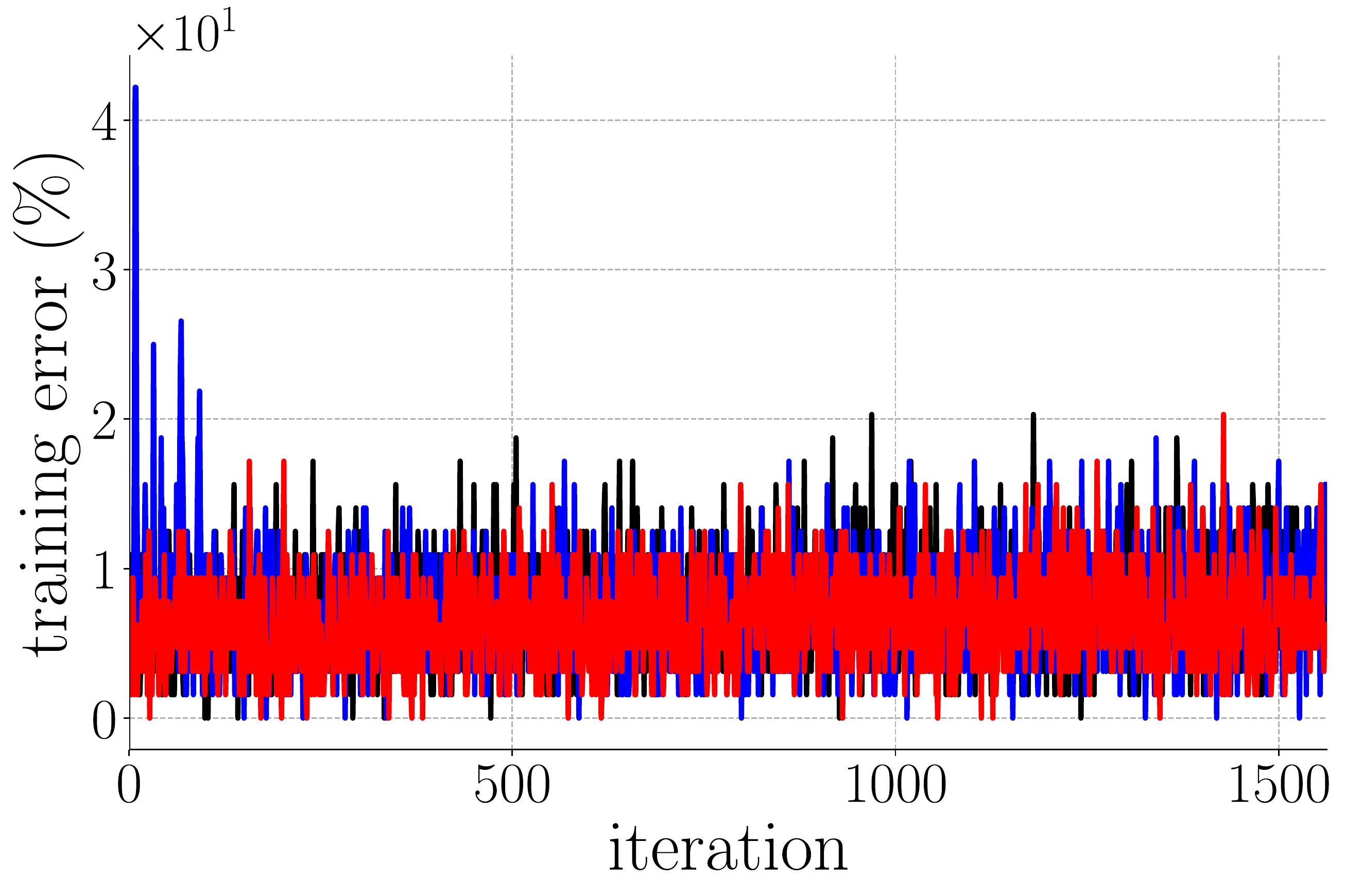} \label{fig:err_step_5}} \\
	\subfloat[Training error vs. iter. at epoch 10]{\includegraphics[width=0.48\linewidth]{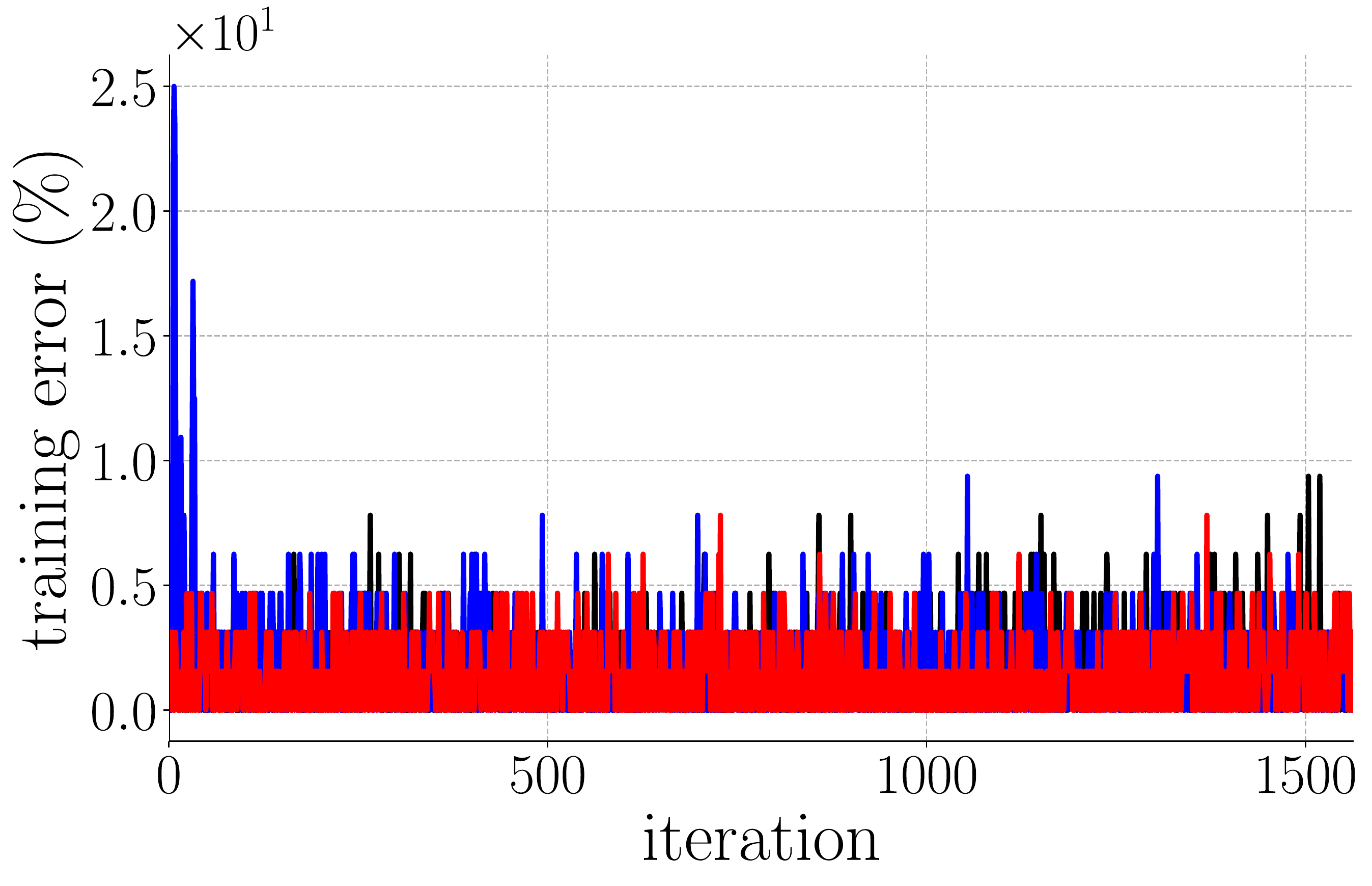} \label{fig:err_step_10}} \hfill
	\subfloat[Training error vs. iter. at epoch 15]{\includegraphics[width=0.48\linewidth]{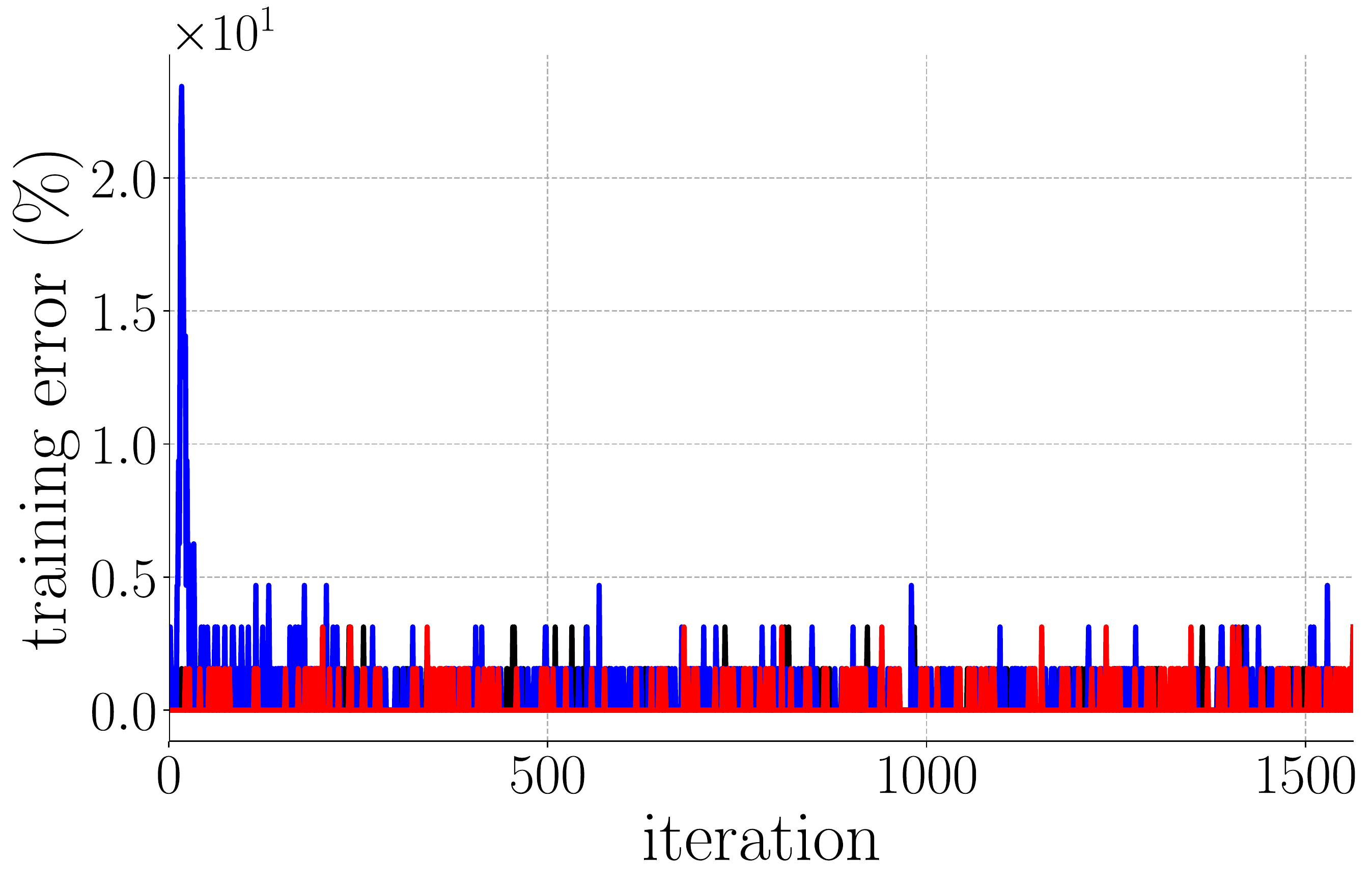} \label{fig:err_step_15}} \\
	\caption{Training error vs. iteration on Tiny ImageNet with ResNet-101. (a) and (b) plot the mean and standard deviation of training errors at each epoch, respectively. Specifically, we show four representative curves of training error vs. iteration at epoch 1, 5, 10, and 15 in (c) - (f), respectively.
	}
	\label{fig:err_vs_step}
\end{figure}


On the other hand, the average magnitudes of the proposed DCL method are relatively flat and smooth in \fig \ref{fig:mag_cifar10}, \ref{fig:mag_cifar100}, and \ref{fig:mag_timg}, comparing to the baseline and GEM. Connecting the magnitudes with the congruencies in \fig \ref{fig:cong_cifar10}, \ref{fig:cong_cifar100}, and \ref{fig:cong_timgnet}, we can infer two points. First, the proposed DCL method finds a nearer local minimum to its initialized weights on CIFAR-10 and CIFAR-100. Because the magnitudes of the proposed DCL method is the smallest among the three methods. Second, the convergence path of the proposed DCL method is the least oscillatory because its congruencies are overall higher than the other two methods and its magnitudes are the lowest among the three methods.

We take a further look at the training error vs. iteration curves to better understand the convergences in \fig~\ref{fig:err_vs_step}.
To give an overview along all epochs, we compute the mean and standard deviation of the training errors at each epoch and plot them at a logarithm scale in \fig \ref{fig:err_step_mean} and \ref{fig:err_step_std}, respectively. 
The results show that the proposed DCL method yields lower training errors from epoch 1 to epoch 16. From epoch 15 onwards, the proposed DCL method is little different from the baseline in terms of the mean because they are both around 0.1. Therefore, we plot the representative curve at epoch 1, 5, 10, and 15 in \fig~\ref{fig:err_step_1}-\ref{fig:err_step_15}.

\subsection{Empirical Convergence}
\label{subsec:convg}

\REVISION{\fig \ref{fig:loss} shows the validation losses w.r.t. the three tasks, \ie saliency prediction (a), continual learning (b), and classification (c). In general, the proposed DCL method achieves lower loss than the baseline and GEM, which is aligned with the fact that the proposed DCL method outperforms the baseline and GEM. Note that classification losses of GEM are above 1.0 so they are not shown in \fig \ref{fig:loss_cls}.}

\subsection{Training from Scratch vs. Fine-tuning}
\label{subsec:finetune}
We analyze the proposed approach with two types of training scheme on the validation set of Tiny ImageNet. The first training scheme train the models from scratch using the training set of the target dataset, whereas the second training scheme fine-tunes the pre-trained ImageNet models on Tiny ImageNet. For ease of comparison, the experimental results of training the models from scratch on Tiny ImageNet as the fine-tuning results are shown in \tab~\ref{tbl:train_scratch}. 
Similar to the results of fine-tuning, the proposed DCL method achieves lower top 1 error (\ie 67.56$\%$) and top 5 error (\ie 40.74$\%$) than the baseline and GEM.


\begin{figure}[!t]
	\centering
	\subfloat[Saliency prediction on SALICON.]{ \label{fig:loss_sal}	\includegraphics[width=0.24\textwidth]{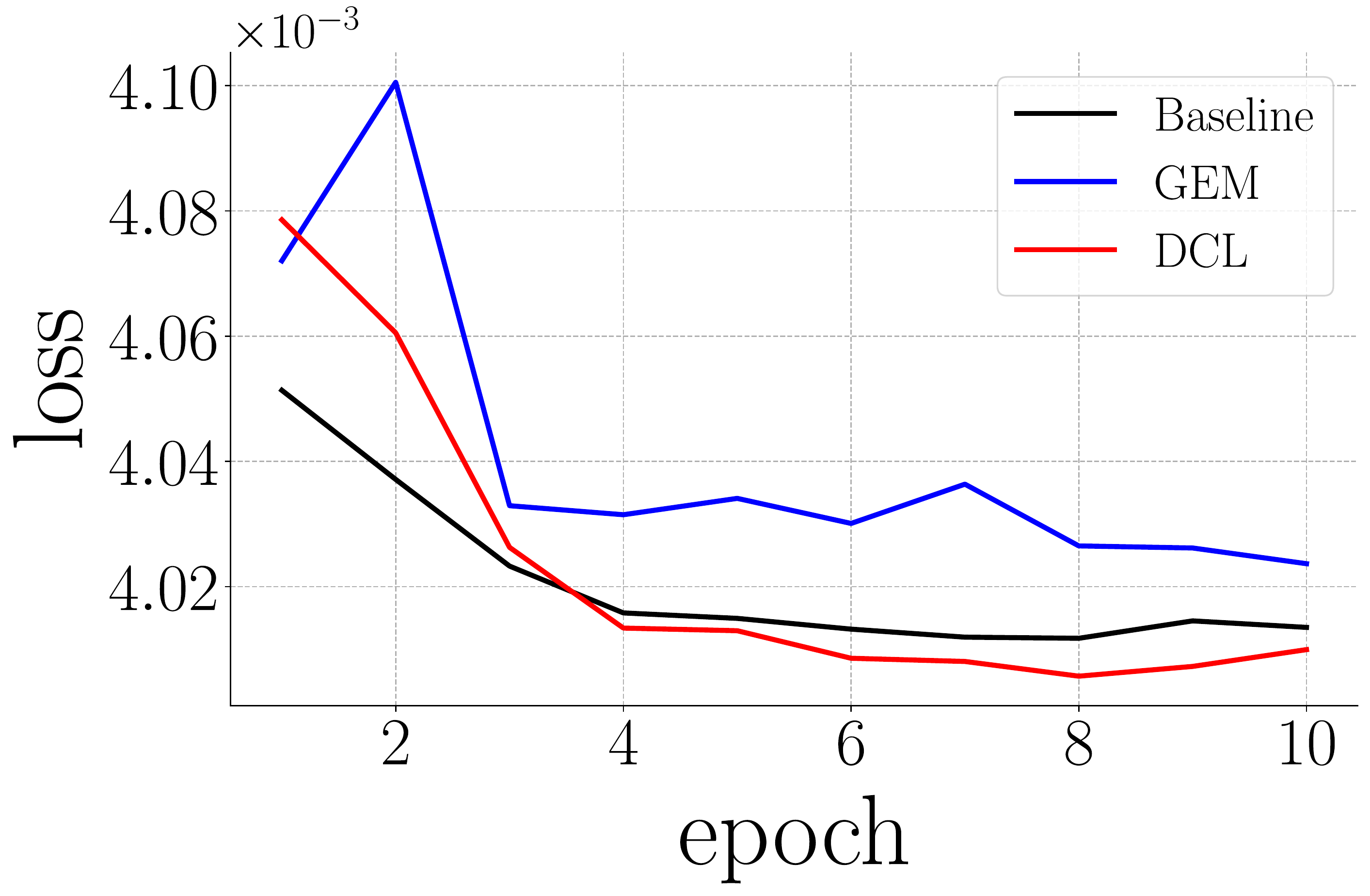}	} 
	\subfloat[Continual learning on iCIFAR-100.]{ \label{fig:loss_ctn}	\includegraphics[width=0.24\textwidth]{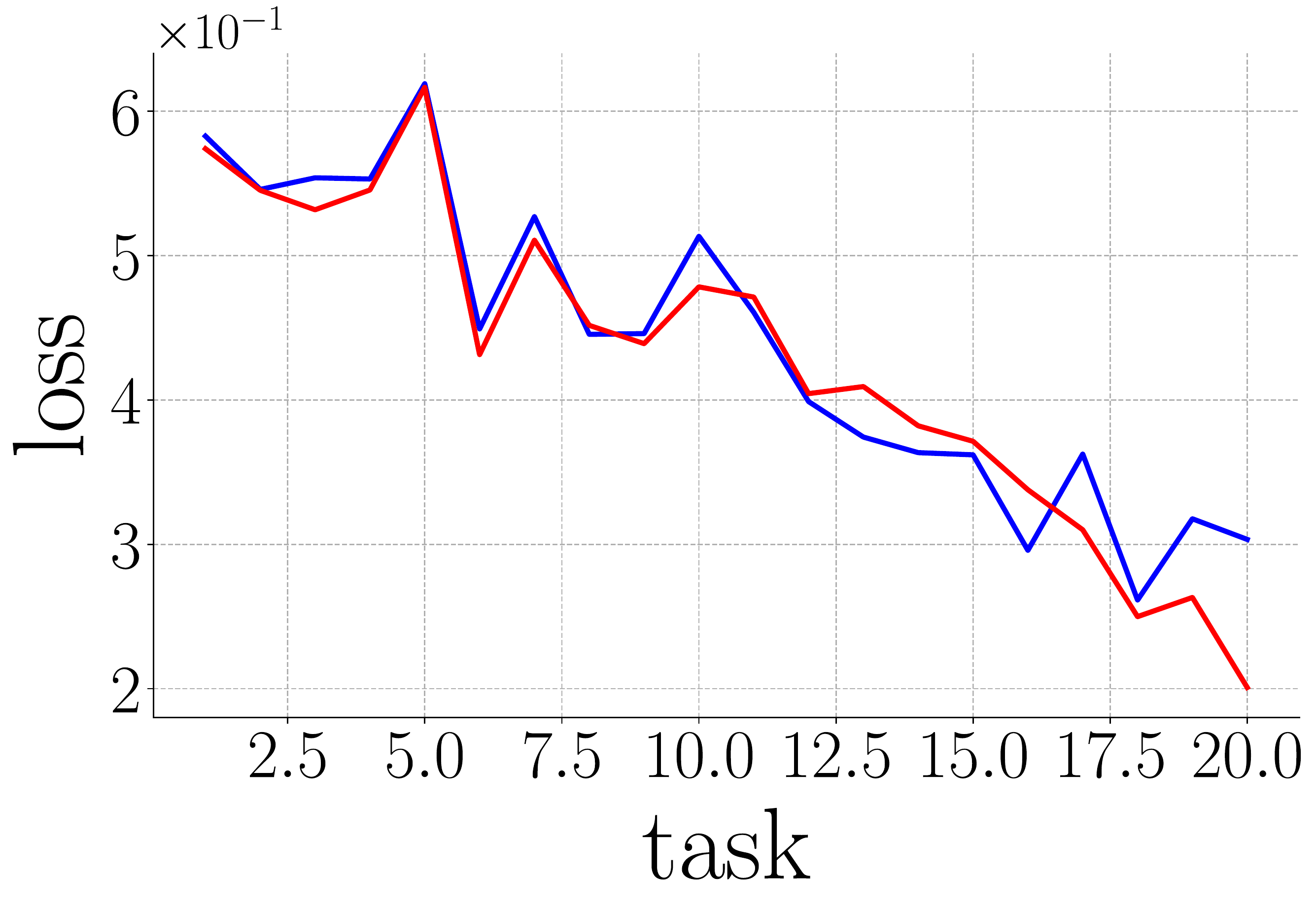}	} \\
	\subfloat[Classification on Tiny ImageNet.]{ \label{fig:loss_cls} 	\includegraphics[width=0.24\textwidth]{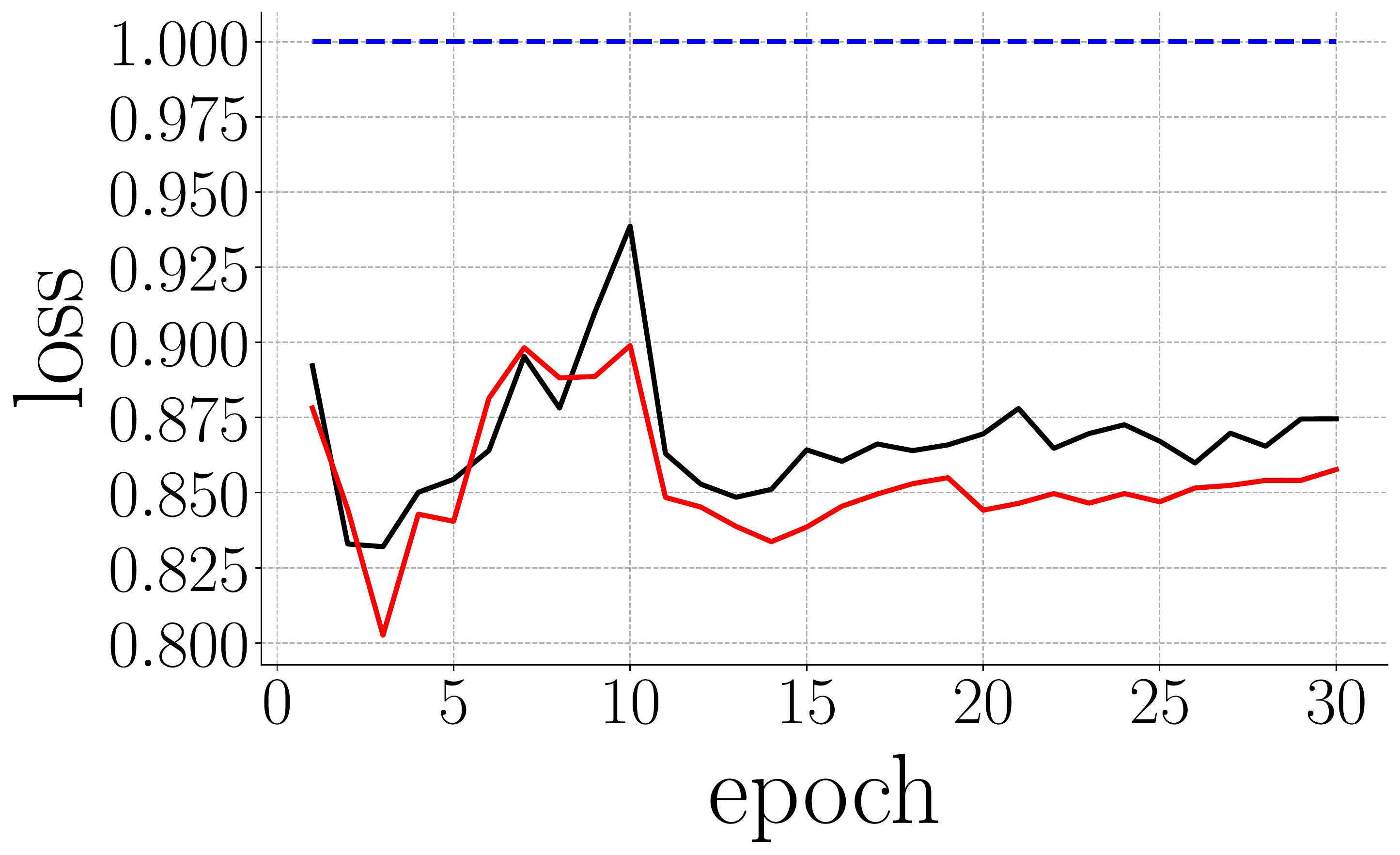}	} 
	\caption{Validation loss vs. epoch/task. In (c), the dashed blue curve indicates that the classification losses of GEM on Tiny ImageNet are all above 1.0 so they are not shown in the figure for clarity.}
	\label{fig:loss}
\end{figure}

\begin{table}[!t]
	\centering
	\caption{Top 1 and top 5 error rate (in \%) on the validation set of Tiny ImageNet. We compare of (1) training the models from scratch (TFS) on Tiny ImageNet, and (2) fine-tuning (FT) the pre-trained ImageNet models on Tiny ImageNet. ResNet-101 SGD DCL is with $\beta_{w}=60$ and $N_{r}=1$. The validation errors of FT are from \tab \ref{tbl:tiny_imgnet}.}
	\vspace{-1ex}
	\begin{tabular}{L{25ex} C{7ex} C{7ex} C{7ex} C{7ex}}
		\toprule
		& \multicolumn{2}{c}{Top 1 error} & \multicolumn{2}{c}{Top 5 error} \\
		\cmidrule(lr){2-3} \cmidrule(lr){4-5}
		& TFS & FT & TFS & FT \\
		\cmidrule(lr){2-2} \cmidrule(lr){3-3} \cmidrule(lr){4-4} \cmidrule(lr){5-5}
		ResNet-101 SGD & 68.21 & 17.34 & 42.72 & 4.82  \\
		ResNet-101 SGD GEM & 77.20 & 21.78 & 53.18 & 7.21 \\
		ResNet-101 SGD DCL  & \textbf{67.56} & \textbf{16.89} & \textbf{40.74} & \textbf{4.50} \\
		\bottomrule
		\label{tbl:train_scratch}
	\end{tabular}
\end{table}


\subsection{Computational Cost}
\label{subsec:cost}
We report computational cost on Tiny ImageNet, SALICON and ImageNet in \tab \ref{tbl:cost_timgnet} and \ref{tbl:cost_r50}, respectively. Specifically, the number of parameters of the models and the corresponding processing time per image are presented. The processing time per image is computed by $(\text{batch time} - \text{data time})/\text{batch size}$, where batch time is the time to complete the process of a batch of images, and data time is the time to load a batch of images. Note that the processes of gradient descent with or without the proposed DCL method are the same in the testing phase. 
We train the models on 3 NVIDIA 1080 Ti graphics cards for the experiments on Tiny ImageNet and SALICON, and on 8 NVIDIA V100 graphics cards for the experiment on ImageNet.

ResNet-101 DCL on Tiny ImageNet is with $\beta_{w}=60$ and $N_{r}=1$. 
ResNet-50 DCL is with $\beta_{w}=\infty$ and $N_{r}=1$ on SALICON, and $\beta_{w}=1$ and $N_{r}=1$ on ImageNet. 
ResNet-50 GEM is with $N_{r}=1$ on all the datasets.
The difference of the numbers of parameters between the baseline and the proposed DCL method (or GEM) lies in the final layer, \ie $1\times 1$ convolutional layer for saliency prediction and the fully connected layer for classification. The proposed DCL method has more parameters to store the weights of the final layer for the references.

In the experiment on Tiny ImageNet, the proposed DCL method with ResNet-101 takes 2 more milliseconds than the baseline to solve the constrained quadratic problem (\ref{eqn:obj}).
Similarly, with ResNet-50, it takes 1 and 2 more milliseconds than the baseline on SALICON and ImageNet, respectively. 
This shows that quadratic problems with high dimensional input can be efficiently solved by the tool quadprog. Hence, the proposed DCL method is practically accessible. On the other hand, GEM \cite{Lopez_NIPS_2017} is less efficient than the other two methods across the three datasets. This is because GEM has to compute the gradients according to the memory, \ie the input features of the final layer, at each iteration. Instead, the proposed DCL method uses a subtraction operation (\ie \eqn \ref{eqn:constraints}) to compute the accumulated gradient. Thus, it is faster than GEM.

\REVISION{
Moreover, we discuss the effects of $\beta_w$ and $N_r$ on the computational cost here. The computational cost w.r.t. $\beta_w$ and $N_r$ with ResNet on Tiny ImageNet is reported in \tab \ref{tbl:cost_N_beta}.
As $\beta_w$ indicates the effective window, it is implemented by a subtraction operation according to \eqn (\ref{eqn:constraints}) and updating the reference point is a copying operation in RAM which is fast. Therefore, $\beta_w$ would not affect computational cost. 
On the other hand, the time difference between various $N_r$ is small because we only apply the proposed DCL method to the downstream layer, \ie the final layer, where the parameters are much fewer than the ones used by the whole network. For example, there are only 2304 parameters in the final convolutional layer for saliency prediction. Any quadratic programming solver like quadprop can efficiently handle the corresponding dual problem (\ref{eqn:dual}) in a small scale.
}

\begin{table}[!t]
	\centering
	\caption{Computational cost of training models on Tiny ImageNet. The processing time (proc time) per image is calculated by $(\text{batch time} - \text{data time})/\text{batch size}$.} 
    \begin{tabular}{L{18ex} C{9ex} C{9ex}}
		\toprule
		               & \# params & proc time \\
		\cmidrule(lr){2-2} \cmidrule(lr){3-3}
		ResNet-101     & 42.50M    & 47 ms \\
		ResNet-101 GEM & 42.91M    & 78 ms \\
		ResNet-101 DCL & 42.91M    & 49 ms \\
		\bottomrule
		\label{tbl:cost_timgnet}
	\end{tabular}
\end{table}

\subsection{Discussion of Generalization}
\label{subsec:general}
\REVISION{
Incongruency is ubiquitous in the learning process. It results from the diversity of the input data, \eg real-world images, and rich task-specific semantics. The proposed DCL method can effectively alleviate the incongruency problem in saliency prediction, continual learning, and classification. Specifically, saliency prediction can be seen as a typical regression problem while continual learning and classification can be seen as a typical learning problem that aims to predict a discrete label. In this sense, the input-output mapping and the learning settings of the three tasks are fundamental to other vision tasks.
}

\REVISION{
From the point of view of task-dependent incongruency, here we consider general vision tasks to be cast into three groups according to the form of input and output.
The first group consists of visual tasks that take images as input for classification or regression, \eg object detection \cite{Ren_NIPS_2015} and visual sentiment analysis \cite{You_AAAI_2017}. In object detection, visual appearance of a region of interest could be diverse in terms of its label and location, while an arbitrary sentiment class can have a number of visual representations in visual sentiment analysis. 
Since tasks in this group has similar incongruency as that in image classification, \ie the diversity of raw image features w.r.t. a certain label, the proposed DCL method is expected to boost this type of vision tasks.
The second group consists of visual tasks that have complex outputs of regression or classification, \eg visual relationship detection~\cite{Li_ICCV_2017,Lu_ECCV_2016} and human object interaction~\cite{Xu_CVPR_2019,Li_CVPR_2019} whose output can involve multiple possible relationships among two or more objects that belong to various visual concepts. 
The incongruency of tasks in this group lies in the diversity of raw image features w.r.t. a higher dimensional variable, \eg a relationship which involves multiple objects and corresponding predicates.
Last but not least, the third group consists of visual tasks that take a series of images, \eg action recognition \cite{Wang_ICCV_2013}. Usually, it takes a clip of videos as input and incorporates temporal information. The incongruency of tasks in this group lies in the diversity of temporal raw image features w.r.t. a certain label, and the feature space with clips is often more complicated than that in static images. 
Therefore, the incongruency of tasks in the second and third groups could be more remarkable than that of tasks in the first group. 
Note that the proposed DCL method is gradient-based and not restricted to specific forms of input or output. Therefore, it could naturally generalize or be used as a starting point to alleviate incongruency for tasks with different forms of input and output in the three groups.
}

\section{Conclusion}
\label{sec:conclusion}

In this work, we define congruency as the agreement between new information and the learned knowledge in a learning process. We propose a Direction Concentration Learning (DCL) method to take into account the congruency in a learning process to search for a local minimum. We study the congruency in the three tasks, \ie saliency prediction, continual learning, and classification. The proposed DCL method generally improves the performances of the three tasks. More importantly, our analysis shows that the proposed DCL method improves catastrophic forgetting.

\begin{table}[!t]
	\centering
	\caption{Computational cost of training models on SALICON and ImageNet.}
	\vspace{-0.7ex}
    \begin{tabular}{l C{9ex} C{9ex} C{9ex} C{9ex}}
		\toprule
		&  \multicolumn{2}{c}{SALICON} &  \multicolumn{2}{c}{ImageNet} \\
		\cmidrule(lr){2-3} \cmidrule(lr){4-5}
		& \# params & proc time & \# params & proc time \\
		\cmidrule(lr){2-2} \cmidrule(lr){3-3} \cmidrule(lr){4-4} \cmidrule(lr){5-5}
		ResNet-50       & 23.51M & ~~64 ms & 23.50M & ~~6 ms \\
		ResNet-50 GEM   & 23.51M & 102 ms & 25.55M & 10 ms \\
		ResNet-50 DCL   & 23.51M & ~~65 ms & 25.55M & ~~8 ms \\
		\bottomrule
		\label{tbl:cost_r50}
	\end{tabular}
\end{table}

\begin{table}[!t]
	\centering
	\caption{The effect of $\beta_w$ and $N_r$ on computational cost (\ie proc time) with ResNet on Tiny ImageNet. Note that $\beta_w$ would not affect computational cost because $\beta_w$ indicates the effective window and resetting the references is implemented as a subtraction operation according to \eqn (\ref{eqn:constraints}). }
	\begin{tabular}{l c}
		\toprule
		$\beta_w$ ($N_r=1$) & proc time \\
		\cmidrule(lr){1-1} \cmidrule(lr){2-2}
		10      & 49 ms \\
		20 & 49 ms \\
		30 & 49 ms \\
		40 & 49 ms \\
		50 & 49 ms \\
		\bottomrule
		\label{tbl:cost_beta}
	\end{tabular}
	\hspace{4ex}
	\begin{tabular}{l c}
		\toprule
		$N_r$ ($\beta_w=50$) & proc time \\
		\cmidrule(lr){1-1} \cmidrule(lr){2-2}
		1      & 49 ms \\
		5 & 51 ms \\
		10 & 53 ms \\
		15 & 54 ms \\
		20 & 54 ms \\
		\bottomrule
		\label{tbl:cost_N}
	\end{tabular}
\label{tbl:cost_N_beta}
\end{table}

\ifCLASSOPTIONcompsoc
  \section*{Acknowledgments}
\else
  \section*{Acknowledgment}
\fi
This research was funded in part by the NSF under Grants 1908711, 1849107, in part by the University of Minnesota Department of Computer Science and Engineering Start-up Fund (QZ), and in part by the National Research Foundation, Prime Minister's Office, Singapore under its Strategic Capability Research Centres Funding Initiative.

\ifCLASSOPTIONcaptionsoff
  \newpage
\fi


\begin{IEEEbiography}[{\includegraphics[width=1in,height=1.25in,clip,keepaspectratio]{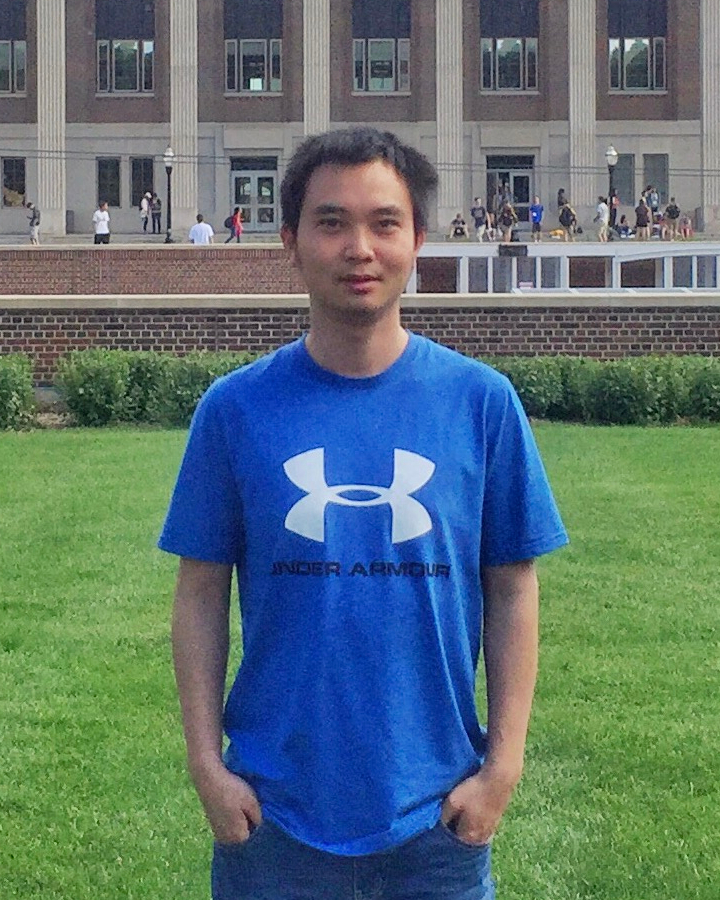}}]{Yan Luo}
is currently pursuing the Ph.D. degree with the Department of Computer Science and Engineering, University of Minnesota at Twin Cities, Minneapolis, MN, USA. He received the B.Sc. degree in computer science from Xi'an University of Science and Technology. In 2013, he joined the Sensor-enhanced Social Media (SeSaMe) Centre, Interactive and Digital Media Institute, National University of Singapore, as a Research Assistant. In 2015, he joined the Visual Information Processing Laboratory at the National University of Singapore as a Ph.D. Student. He worked in the industry for several years on distributed system. His research interests include computer vision, computational visual cognition, and deep learning.
\end{IEEEbiography}


\begin{IEEEbiography}[{\includegraphics[width=1in,height=1.25in,clip,keepaspectratio]{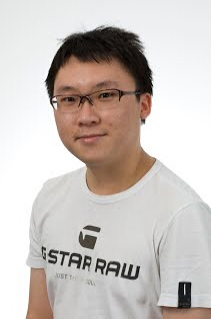}}]{Yongkang Wong}
is a senior research fellow at the School of Computing, National University of Singapore. He is also the Assistant Director of the NUS Centre for Research in Privacy Technologies (N-CRiPT). He obtained his BEng from the University of Adelaide and PhD from the University of Queensland. He has worked as a graduate researcher at NICTA's Queensland laboratory, Brisbane, OLD, Australia, from 2008 to 2012. His current research interests are in the areas of Image/Video Processing, Machine Learning, Action Recognition, and Human Centric Analysis. He is a member of the IEEE since 2009.
\end{IEEEbiography}

\begin{IEEEbiography}[{\includegraphics[width=1in,height=1.25in,clip,keepaspectratio]{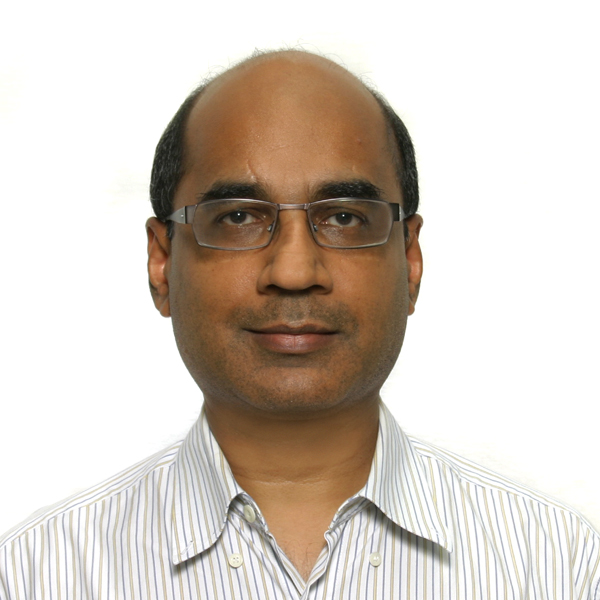}}]{Mohan~Kankanhalli}
is the Provost's Chair Professor at the Department of Computer Science of the National University of Singapore. He is the director with the N-CRiPT and also the Dean, School of Computing at NUS. Mohan obtained his BTech from IIT Kharagpur and MS \& PhD from the Rensselaer Polytechnic Institute. His current research interests are in Multimedia Computing, Multimedia Security and Privacy, Image/Video Processing and Social Media Analysis. He is on the editorial boards of several journals. Mohan is a Fellow of IEEE.
\end{IEEEbiography}

\begin{IEEEbiography}[{\includegraphics[width=1in,height=1.25in,clip,keepaspectratio]{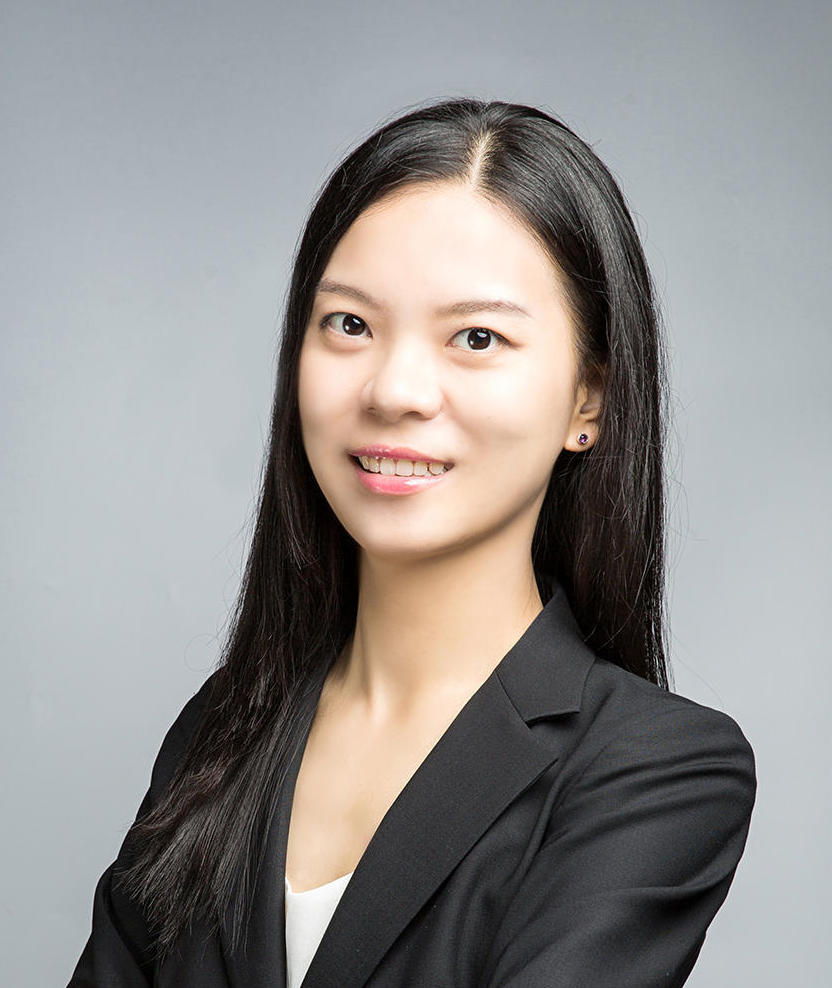}}]{Qi Zhao}
is an assistant professor in the Department of Computer Science and Engineering at the University of Minnesota, Twin Cities. Her main research interests include computer vision, machine learning, cognitive neuroscience, and mental disorders. She received her Ph.D. in computer engineering from the University of California, Santa Cruz in 2009. She was a postdoctoral researcher in the Computation \& Neural Systems, and Division of Biology at the California Institute of Technology from 2009 to 2011. Prior to joining the University of Minnesota, Qi was an assistant professor in the Department of Electrical and Computer Engineering and the Department of Ophthalmology at the National University of Singapore. She has published more than 50 journal and conference papers in top computer vision, machine learning, and cognitive neuroscience venues, and edited a book with Springer, titled Computational and Cognitive Neuroscience of Vision, that provides a systematic and comprehensive overview of vision from various perspectives, ranging from neuroscience to cognition, and from computational principles to engineering developments. She is a member of the IEEE since 2004.
\end{IEEEbiography}


\end{document}